\pdfoutput=1
\documentclass[3p]{elsarticle}
\usepackage[T1]{fontenc}
\usepackage[utf8]{inputenc}
\usepackage{xcolor}
\usepackage{pdfcolmk}
\usepackage{units}
\usepackage{mathtools}
\usepackage{url}
\usepackage{bm}
\usepackage{amsmath}
\usepackage{amsthm}
\usepackage{amssymb}
\usepackage{graphicx}
\PassOptionsToPackage{normalem}{ulem}
\usepackage{ulem}
\usepackage[unicode=true,
 bookmarks=false,
 breaklinks=false,pdfborder={0 0 1},backref=false,colorlinks=false]
 {hyperref}

\makeatletter

\newcommand{\lyxdot}{.}

\providecolor{lyxadded}{rgb}{0,0,1}
\providecolor{lyxdeleted}{rgb}{1,0,0}

\DeclareRobustCommand{\lyxdeleted}[3]{{\texorpdfstring{\color{lyxdeleted}\lyxsout{#3}}{}}}
\DeclareRobustCommand{\lyxsout}[1]{\ifx\\#1\else\sout{#1}\fi}

\theoremstyle{plain}
\newtheorem{thm}{\protect\theoremname}
\theoremstyle{plain}
\newtheorem{lem}[thm]{\protect\lemmaname}
\theoremstyle{plain}
\newtheorem{cor}[thm]{\protect\corollaryname}
\theoremstyle{definition}
\newtheorem{defn}[thm]{\protect\definitionname}
\theoremstyle{remark}
\newtheorem{rem}[thm]{\protect\remarkname}
\theoremstyle{plain}
\newtheorem{fact}[thm]{\protect\factname}
\theoremstyle{plain}
\newtheorem{prop}[thm]{\protect\propositionname}

\usepackage{bm}
\usepackage{color}

\usepackage{import}
\usepackage{tikz}

\journal{Journal of the Franklin Institute}

\@ifundefined{showcaptionsetup}{}{%
 \PassOptionsToPackage{caption=false}{subfig}}
\usepackage{subfig}
\makeatother

\providecommand{\corollaryname}{Corollary}
\providecommand{\definitionname}{Definition}
\providecommand{\factname}{Fact}
\providecommand{\lemmaname}{Lemma}
\providecommand{\propositionname}{Proposition}
\providecommand{\remarkname}{Remark}
\providecommand{\theoremname}{Theorem}

\begin{document}
\begin{frontmatter}{}

\title{Pose consensus based on dual quaternion algebra with application
to decentralized formation control of mobile manipulators}

\author[UFAL]{Heitor J. Savino}

\ead{heitor.savino@ic.ufal.br}

\author[DELT]{Luciano C. A. Pimenta}

\ead{lucpim@cpdee.ufmg.br}

\author[MIT]{Julie A. Shah}

\ead{julie\_a\_shah@csail.mit.edu}

\author[DEE]{Bruno V. Adorno\corref{cor1}}

\ead{adorno@ufmg.br}

\cortext[cor1]{Corresponding author}

\address[UFAL]{Institute of Computing, Federal University of Alagoas, Av Lourival
Melo Mota, S/N, Tabuleiro do Martins, Maceio, AL, 57072-970, Brazil}

\address[DELT]{Department of Electronic Engineering, Federal University of Minas
Gerais, , Av. Antonio Carlos, 6627, Pampulha, Belo Horizonte, MG,
31270-901, Brazil}

\address[MIT]{Department of Aeronautics and Astronautics, Massachusetts Institute
of Technology, 77 Massachusetts Av, Cambridge, MA, 02139, USA}

\address[DEE]{Department of Electrical Engineering, Federal University of Minas
Gerais, Av. Antonio Carlos, 6627, Pampulha, Belo Horizonte, MG, 31270-901,
Brazil}


\begin{abstract}
This paper presents a solution based on dual quaternion algebra to
the general problem of pose (i.e., position and orientation) consensus
for systems composed of multiple rigid-bodies. The dual quaternion
algebra is used to model the agents' poses and also in the distributed
control laws, making the proposed technique easily applicable to time-varying
formation control of general robotic systems. The proposed pose consensus
protocol has guaranteed convergence when the interaction among the
agents is represented by directed graphs with directed spanning trees,
which is a more general result when compared to the literature on
formation control. In order to illustrate the proposed pose consensus
protocol and its extension to the problem of formation control, we
present a numerical simulation with a large number of free-flying
agents and also an application of cooperative manipulation by using
real mobile manipulators.
\end{abstract}
\end{frontmatter}{}





\global\long\def\dq#1{\underline{\bm{#1}}}%

\global\long\def\quat#1{\boldsymbol{#1}}%

\global\long\def\mymatrix#1{\boldsymbol{#1}}%

\global\long\def\myvec#1{\boldsymbol{#1}}%

\global\long\def\mapvec#1{\boldsymbol{#1}}%

\global\long\def\crossproduct#1#2{\frac{#1#2-#2#1}{2}}%

\global\long\def\dualvector#1{\underline{\boldsymbol{#1}}}%

\global\long\def\dual{\varepsilon}%

\global\long\def\dotproduct#1{\langle#1\rangle}%

\global\long\def\norm#1{\left\Vert #1\right\Vert }%

\global\long\def\mydual#1{\underline{#1}}%

\global\long\def\hami#1{\overset{#1}{\mymatrix H}}%

\global\long\def\hamidq#1#2{\overset{#1}{\mymatrix H}_{8}\left(#2\right)}%

\global\long\def\hamilton#1#2{\overset{#1}{\mymatrix H}\left(#2\right)}%

\global\long\def\hamiquat#1#2{\overset{#1}{\mymatrix H}_{4}\left(#2\right)}%

\global\long\def\tplus{\dq{\mathcal{T}}}%

\global\long\def\gett#1{\dq{\mathcal{T}}\left(#1\right)}%

\global\long\def\dgett#1{\dq{\mathcal{T}}'\left(#1\right)}%

\global\long\def\getp#1{\operatorname{\mathcal{P}}\left(#1\right)}%

\global\long\def\dgetp#1{\operatorname{\mathcal{P}}'\left(#1\right)}%

\global\long\def\getd#1{\operatorname{\mathcal{D}}\left(#1\right)}%

\global\long\def\swap#1{\text{swap}\{#1\}}%

\global\long\def\imi{\hat{\imath}}%

\global\long\def\imj{\hat{\jmath}}%

\global\long\def\imk{\hat{k}}%

\global\long\def\real#1{\operatorname{\mathrm{Re}}\left(#1\right)}%

\global\long\def\imag#1{\operatorname{\mathrm{Im}}\left(#1\right)}%

\global\long\def\imvec{\boldsymbol{\imath}}%

\global\long\def\vector{\operatorname{vec}}%

\global\long\def\mathpzc#1{\fontmathpzc{#1}}%

\global\long\def\cost#1#2{\underset{\text{#2}}{\operatorname{\text{cost}}}\left(\ensuremath{#1}\right)}%

\global\long\def\diag#1{\operatorname{diag}\left(#1\right)}%

\global\long\def\frame#1{\mathcal{F}_{#1}}%

\global\long\def\ad#1#2{\text{Ad}\left(#1\right)#2}%

\global\long\def\adsharp#1#2{\text{Ad}_{\sharp}\left(#1\right)#2}%

\global\long\def\error#1{\tilde{#1}}%

\global\long\def\derror#1{\dot{\tilde{#1}}}%

\global\long\def\dderror#1{\ddot{\tilde{#1}}}%

\global\long\def\spin{\text{Spin}(3)}%

\global\long\def\spinr{\text{Spin}(3){\ltimes}\mathbb{R}^{3}}%

\section{Introduction}

Recent technological advances have enabled the use of distributed
multi-agent systems in the solution of different real-world problems.
In fact, replacing a single complex agent by multiple yet simpler
ones yields many benefits such as flexibility, fault tolerance, cost
reduction, etc., which justifies the development of decentralized
controllers for this class of systems. There exist many results regarding
the use of decentralized controllers in autonomous systems such as
formation control of autonomous vehicles \citep{ren2008distributed,oh2015survey},
networked robotics \citep{ARTEAGAPEREZ20186810,hatanaka2015}, etc.
Many other results are summarized in \citep{Cao2013}.

Some decentralized strategies are based on the solution of a consensus
problem, whose main objective is to enable agents in a multi-agent
system to reach an agreement about some variable of interest by means
of local distributed control laws, called consensus protocols. These
protocols rely on the assumption that each agent has access to the
information provided by only a subset of agents, called neighbors.
This subset is defined according to an interaction network that is
usually modeled by a graph. The problem of achieving consensus based
only on neighbors interactions was initially proposed in \citep{Vicsek1995}
and algebraically formulated in the works of \citep{fax2002,jadbabaie2003}.

On the application side, an interesting use of consensus-based algorithms
is in the solution of decentralized formation control problems in
multi-agent systems \citep{oh2015survey} and robotics \citep{schwager2011}.
In fact, several tasks may benefit from solutions of formation control,
such as load transportation with cooperative robots to move flexible
payloads \citep{bai2010cooperative}. Different formation control
scenarios have been investigated such as the ones incorporating, but
not limited to, time-varying formations \citep{brinon2014cooperative,wang2016distributed,Zhao2018,Li2018},
formations with multiple leaders \citep{Li2013,dong2017time,Dong2019},
switching network topologies \citep{Wang2012f,dong2016time,Hua2019},
time-delays \citep{dong2015formation,Yu2018}, etc. Stochastic switching
topologies with time-varying delays have also been considered in \citep{savinoTIE}.

Devising new solutions for different aspects of consensus and formation
control problems is still an active research topic. Some recent studies
have considered multi-agent systems composed of rigid-body agents,
usually with the objective of achieving a common orientation or, more
generally, a common pose (position and orientation). Hatanaka et al.
\citep{hatanaka2015}, for example, use homogeneous representations
to describe the complete pose and make use of passivity theory to
show consensus in the case of strongly connected networks. Mayhew
et al. \citep{mayhew2012} show consensus in the orientation for undirected
networks by applying a hybrid controller and a representation based
on quaternions. Sarlette et al. \citep{Sarlette2009} show relaxed
conditions for directed and varying networks. Aldana et al. \citep{Aldana2014a}
decoupled agents' positions and orientations expressing poses as two
independent entities, position vectors and orientation quaternions,
and addressed leader-follower and leaderless pose-consensus problems
in undirected networks. The same authors \citep{ALDANA20141517} extend
the previous results to consensus problems in the operational space
of robotic manipulators without velocity measurements. Wang et al.
\citep{wang2012dual} consider dual quaternions to represent the pose
and propose a control law based on the logarithm of dual quaternions
to show consensus in networks with rooted-tree topologies. Wang and
Yu \citep{WANG20173594} also consider dual quaternions for leader-followers
in undirected topologies. The logarithm of a quaternion was defined
by \citep{kim1996compact}, which served as base for the logarithmic
controller proposed by \citep{wang2012dual}.

In networks composed of multiple robotic manipulators, described as
rigid-body agents, the agents can be modeled with dual quaternions
\citep{Adorno2011} and consensus theory can be used to analyze or
design distributed control laws. Some advantages of using quaternions
and dual quaternions in formation control are shown by Mas and Kitts
\citep{Mas2017} in the framework of Cluster Space Control, by defining
each relative position of the agents by means of relative transformations
given by dual quaternions. An application on formation of unmanned
aerial vehicles is shown in \citep{MasICUAS}.

A growing interest in dual quaternions for rigid-body pose consensus
and formation control arises from the many benefits of using dual
quaternion algebra. As pointed by \citep{adornoCMI}, it is straightforward
to use dual quaternions in the representation of rigid motions, twists,
wrenches, and several geometric primitives---e.g., Plücker lines
and planes. In addition, dual quaternions are more compact than homogeneous
transformation matrices (HTM)---the former has only eight parameters
whereas the latter has sixteen---and dual quaternion multiplications
have lower computational cost than HTM multiplications \citep{Adorno2011}.
Furthermore, unit dual quaternions do not have representational singularities
(although this feature is also present in HTM) and, given a unit dual
quaternion, it is easy to extract relevant geometric parameters as,
for example, translation, axis of rotation, and angle of rotation.
Moreover, dual quaternions are easily mapped into a vector structure,
which can be particularly convenient when controlling a robot as they
can be used directly in the control law. Finally, complex systems
(e.g., mobile manipulators and humanoids) can be easily modeled with
dual quaternions using a whole-body approach \citep{Adorno2011,Fonseca2016}.
Thanks to the aforementioned advantages, dual quaternions are used
throughout the paper as the main mathematical tool for representing
poses and rigid motions.

\subsection{Statement of Contributions and Paper Organization}

The contributions of this paper are the following:
\begin{enumerate}
\item First, we derive a logarithmic differentiable mapping of dual quaternions,
extending the result in \citep{kim1996compact}. This allows a straightforward
theoretical connection between the myriad of results of rigid-body
modeling based on dual quaternion algebra and the results of linear
consensus theory applied to Euclidean spaces. The advantage of such
connection is that previous results in linear consensus theory for
time-delays and switching topologies in Euclidean spaces, such as
the ones presented in \citep{savinoTIE}, may be easily applied to
the problem of formation control of rigid bodies, which is non-linear
and whose underlying topological space is a non-Euclidean manifold;
\item Next, by defining the agent's output as the logarithmic mapping of
the unit dual quaternion corresponding to the agent's pose, we propose
a pose-consensus protocol with guaranteed convergence for scenarios
where the interaction graphs are given by directed graphs with directed
spanning trees, which is a more general case when compared to previous
results, for instance, the ones in \citep{wang2012dual,Aldana2014a,ALDANA20141517,hatanaka2015}.
It is important to note that guaranteeing consensus in the pose is
not a trivial task as unit dual quaternions lie in a non-Euclidean
topological space (more specifically, unit dual quaternions belong
to the Lie group $\text{Spin(3)}\ltimes\mathbb{R}^{3}$, whose underlying
manifold is $\mathbb{S}^{3}\times\mathbb{R}^{3}$ \citep{Kussaba2017});
\item Different from other works such as \citep{Aldana2014a,ALDANA20141517},
we propose a consensus-based strategy for decentralized formation
control of rigid-bodies in which both position and orientation are
treated in a unified manner, which allows to consider any arbitrary
communication network containing a directed spanning tree. An extension
to consider time-varying formations is also devised. This result is
more general than the previous ones found in the literature that also
focus on the formation control of systems composed of rigid bodies
\citep{wang2012dual,Mas2017,MasICUAS,WANG20173594} as our approach:
(i) is decentralized in the sense that only neighbor information is
needed by each agent, in contrast to the necessity of obtaining global
information such as the state variables of a shape or of a leader
as in \citep{Mas2017,MasICUAS}; and (ii) is also able to deal with
general directed graph topologies, in contrast to the requirement
of imposing some specific graph topologies such as undirected graphs
\citep{WANG20173594} and rooted trees \citep{wang2012dual};
\item On the application side, whole-body control and consensus protocols
are used to propose a strategy that allows decentralized formation
control of the end-effectors of mobile manipulators whose kinematic
models are given directly in the algebra of dual quaternions;
\item Finally, the proposed strategy is verified by means of numerical simulations
and also in a real-world cooperative manipulation task.
\end{enumerate}
The paper is organized as follows. Section~\ref{sec:preliminaries}
presents a brief mathematical background whereas Section~\ref{differential_logarithmic_mapping}
presents the differential logarithmic mapping of unit dual quaternions,
which is of central importance in the development of the pose-consensus
protocols proposed in Section~\ref{sec:consensus}. In Section~\ref{sec:formation}
we solve the problem of formation control of multiple rigid-bodies
by using dual quaternion algebra. Section~\ref{sec:Examples} shows
a numerical simulation with a large number of agents to illustrate
the results and scalability of the proposed method, and also shows
the formation control applied to real robots in a cooperative manipulation
task. Finally, Section~\ref{sec:Conclusion} concludes the paper
and provides indications of future works.

\section{Mathematical Preliminaries\label{sec:preliminaries}}

This section briefly presents the main mathematical tools and notations
used throughout the paper. For more information on the algebraic formulation
of the consensus problem and dual quaternion algebra, please refer
to \citep{jadbabaie2003} and \citep{Adorno2017}, respectively.

\subsection{Algebraic Graph Theory\label{subsec:Algebraic-Graph-Theory}}

The information flow of the multi-agent system is represented by a
simple directed graph. Let a simple weighted directed graph be defined
by the ordered triplet $\mathcal{G}\left(\mathcal{V},\mathcal{E},\mymatrix A\right)$,
where: $\mathcal{V}$ is a set of $n\in\mathbb{N}$ vertices (nodes)
arbitrarily labeled as $v_{1},v_{2},\ldots,v_{n}$; the set $\mathcal{E}$
contains the directed edges $e_{ij}=(v_{i},v_{j})$ that connect the
vertices, where the first element $v_{i}\in\mathcal{V}$ is said to
be the parent node (tail) and the latter, $v_{j}\in\mathcal{V}$,
to be the child node (head); and $\mymatrix A=[a_{ij}]$ is the adjacency
matrix of order $n\times n$ related to the edges that assigns a real
non-negative weight value for each $e_{ji}$: 
\begin{gather}
a_{ij}\begin{cases}
=0, & \text{if }i=j\text{ or }\nexists e_{ji},\\
>0, & \text{iff }\exists e_{ji}.
\end{cases}\label{ruleaij}
\end{gather}

The degree matrix $\mymatrix{\Delta}=[\Delta_{ij}]$, which is related
to $\mymatrix A$, is a diagonal matrix with elements $\Delta_{ii}=\sum_{j=1}^{n}a_{ij}$.
The Laplacian matrix associated to the graph $\mathcal{G}$ is given
by $\mymatrix L=\mymatrix{\Delta}-\mymatrix A$, and the following
property holds: 
\begin{align}
\mymatrix L\myvec 1_{n} & =\myvec 0_{n},\label{eq:laplacian}
\end{align}
where $\myvec 1_{n}$ and $\myvec 0_{n}$ are $n$-dimensional column-vectors
of ones and zeros, respectively.

A directed tree is a directed graph with only one node without parent
nodes (or without directed edges pointing towards it) called root,
and all other nodes having exactly one parent. Also, there is a path,
i.e. a sequence of edges, connecting the root to any other node in
the tree. A directed spanning tree is a directed tree that can be
formed from the removal of some of the edges of a directed graph,
such that all nodes are included and there is a unique directed path
from the root node to any other node in the graph.

\subsection{Quaternions and dual quaternions}

Quaternions can be regarded as an extension of complex numbers, and
the quaternion set is defined as 
\begin{align}
\mathbb{H} & \triangleq\left\{ h_{1}+\imi h_{2}+\imj h_{3}+\imk h_{4}\,:\,h_{1},h_{2},h_{3},h_{4}\in\mathbb{R}\right\} ,\label{eq:quaternion_set}
\end{align}
in which the imaginary units $\imi$, $\imj$, and $\imk$ have the
following properties: 
\begin{align}
\hat{\imath}^{2} & =\hat{\jmath}^{2}=\hat{k}^{2}=\hat{\imath}\hat{\jmath}\hat{k}=-1.\label{eq:imaginary_units_properties}
\end{align}
Addition and multiplication are defined for quaternions analogously
to complex numbers (i.e., in the usual way), and one just needs to
respect the properties in \eqref{eq:imaginary_units_properties} for
the imaginary units. Given $\quat h\in\mathbb{H}$, such that $\quat h=h_{1}+\imi h_{2}+\imj h_{3}+\imk h_{4}$,
we define $\real{\quat h}\triangleq h_{1}$ and $\imag{\quat h}\triangleq\imi h_{2}+\imj h_{3}+\imk h_{4}$.
The conjugate of $\quat h$ is defined as $\quat h^{*}\triangleq\real{\quat h}-\imag{\quat h}$
and its norm is given by $\norm{\quat h}\triangleq\sqrt{\quat h^{*}\quat h}=\sqrt{\quat h\quat h^{*}}.$

The set 
\begin{align}
\mathbb{H}_{p} & \triangleq\left\{ \quat h\in\mathbb{H}\,:\,\real{\quat h}=0\right\} \label{eq:set_pure_quaternions}
\end{align}
is usually called the set of \emph{pure} quaternions and has a bijective
relation with $\mathbb{R}^{3}$. Hence, the quaternion $\left(x\hat{\imath}+y\hat{\jmath}+z\hat{k}\right)\in\mathbb{H}_{p}$
represents the point $\left(x,y,z\right)\in\mathbb{R}^{3}$ \citep{Selig2005}.
The set of quaternions with unit norm is defined as 
\begin{align}
\mathbb{S}^{3} & \triangleq\left\{ \quat h\in\mathbb{H}\,:\,\norm{\quat h}=1\right\} ,\label{eq:quaternions_unit_norm}
\end{align}
and elements of $\mathbb{S}^{3}$ equipped with the multiplication
operation form the group of rotations $\text{Spin}(3)$, which double
covers $\mathrm{SO}\left(3\right)$. A unit quaternion $\quat r\in\mathbb{S}^{3}$
represents a rotation from an inertial frame $\frame{}$ to frame
$\frame i$ and can always be written as 
\begin{align}
\quat r_{i} & =\cos\left(\frac{\phi_{i}}{2}\right)+\sin\left(\frac{\phi_{i}}{2}\right)\quat n_{i},\label{eq:rotation}
\end{align}
where $\phi_{i}\in\mathbb{R}$ is a rotation angle around the rotation
axis $\quat n_{i}\in\mathbb{S}^{3}\cap\mathbb{H}_{p}$ \citep{Adorno2017}.
Notice that $\quat n_{i}$ is pure (hence it is equivalent to a vector
in $\mathbb{R}^{3}$) and has unit norm.

The set of dual quaternions extends the set of quaternions and is
defined as 
\begin{align}
\mathcal{H} & \triangleq\left\{ \quat h+\epsilon\quat h'\,:\,\quat h,\quat h'\in\mathbb{H},\,\dual^{2}=0,\,\dual\neq0\right\} ,\label{eq:dual_quaternion_set}
\end{align}
where $\dual$ is usually called dual (or Clifford) unit \citep{Selig2005}.
Similarly to quaternions, addition and multiplication are defined
in the usual way, and one just needs to respect the properties of
the imaginary and dual units.

Given $\dq h\in\mathcal{H}$ such that $\dq h=h_{1}+\imi h_{2}+\imj h_{3}+\imk h_{4}+\dual\left(h_{1}'+\imi h_{2}'+\imj h_{3}'+\imk h_{4}'\right)$,
we define the operators 
\begin{align*}
\begin{split}\real{\dq h} & \triangleq h_{1}+\dual h_{1}',\\
\imag{\dq h} & \triangleq\imi h_{2}+\imj h_{3}+\imk h_{4}+\dual\left(\imi h_{2}'+\imj h_{3}'+\imk h_{4}'\right).
\end{split}
\end{align*}
Analogously to quaternions, the conjugate of $\dq h\in\mathcal{H}$
is defined as $\dq h^{*}\triangleq\real{\dq h}-\imag{\dq h}$, and
its norm is given by $\norm{\dq h}\triangleq\sqrt{\dq h\dq h^{*}}=\sqrt{\dq h^{*}\dq h}$.

The set 
\[
\mathcal{H}_{p}\triangleq\left\{ \dq h\in\mathcal{H}\,:\,\real{\dq h}=0\right\} 
\]
is called set of pure dual quaternions and is isomorphic to $\mathbb{R}^{6}$.
Some physical objects---for instance, twists (i.e., linear and angular
velocities) and wrenches (i.e., forces and moments)---can be represented
as elements of $\mathcal{H}_{p}$ \citep{Adorno2017}.

Elements of the set 
\begin{align*}
\dq{\mathcal{S}} & \triangleq\left\{ \dq h\in\mathcal{H}\,:\,\norm{\dq h}=1\right\} 
\end{align*}
are called unit dual quaternions. The set $\dq{\mathcal{S}}$ equipped
with the multiplication operation form the group $\spinr$, which
double covers $\mathrm{SE}\left(3\right)$. A unit dual quaternion
$\dq x\in\dq{\mathcal{S}}$ represents a rigid motion from an inertial
frame $\frame{}$ to frame $\frame i$ and is represented by 
\begin{align}
\dq x_{i} & =\quat r_{i}+\dual\frac{1}{2}\quat p_{i}\quat r_{i},\label{eq:rigidmotion}
\end{align}
where $\quat r_{i}\in\mathbb{S}^{3}$ and $\quat p_{i}\in\mathbb{H}_{p}$
represent the rotation and translation, respectively \citep{Selig2005}.

Since $\text{Spin}(3)$ and $\text{Spin}(3)\ltimes\mathbb{R}^{3}$
are non-commutative groups---analogously to $\mathrm{SO}\left(3\right)$
and $\mathrm{SE}\left(3\right)$---, quaternions and dual quaternions
are non-commutative under multiplication. However, we can use the
Hamilton operators, which are matrices defined in \citep{McCarthy1990,Adorno2017}
for both quaternions and dual quaternions, that can be used to commute
these terms in algebraic expressions such that, for $\quat h_{1},\quat h_{2}\in\mathbb{H}$
and $\dq h_{1},\dq h_{2}\in\mathcal{H}$, 
\begin{align}
\vector_{4}(\quat h_{1}\quat h_{2}) & =\hami +_{4}(\quat h_{1})\vector_{4}\quat h_{2}=\hami -_{4}(\quat h_{2})\vector_{4}\quat h_{1},\\
\vector_{8}(\dq h_{1}\dq h_{2}) & =\hami +_{8}(\dq h_{1})\vector_{8}\dq h_{2}=\hami -_{8}(\dq h_{2})\vector_{8}\dq h_{1},\label{eq:hami}
\end{align}
where $\vector_{4}\quat h=\begin{bmatrix}h_{1} & \cdots & h_{4}\end{bmatrix}^{T}$
and $\vector_{8}\dq h=\begin{bmatrix}h_{1} & \cdots & h_{8}\end{bmatrix}^{T}$
are mappings of quaternions into $\mathbb{R}^{4}$ and dual quaternions
into $\mathbb{R}^{8}$, respectively; i.e, $\vector_{4}:\mathbb{H}\rightarrow\mathbb{R}^{4}$
and $\vector_{8}:\mathcal{H}\rightarrow\mathbb{R}^{8}$.The Hamilton
operators are given explicitly by 
\begin{align}
\hamiquat +{\quat h} & =\begin{bmatrix}h_{1} & -h_{2} & -h_{3} & -h_{4}\\
h_{2} & h_{1} & -h_{4} & h_{3}\\
h_{3} & h_{4} & h_{1} & -h_{2}\\
h_{4} & -h_{3} & h_{2} & h_{1}
\end{bmatrix}, & \hamiquat -{\quat h} & =\begin{bmatrix}h_{1} & -h_{2} & -h_{3} & -h_{4}\\
h_{2} & h_{1} & h_{4} & -h_{3}\\
h_{3} & -h_{4} & h_{1} & h_{2}\\
h_{4} & h_{3} & -h_{2} & h_{1}
\end{bmatrix},\label{eq:hamilton-four}\\
\hamidq +{\dq h} & =\begin{bmatrix}\hamiquat +{\quat h} & \mymatrix 0_{4\times4}\\
\hamiquat +{\quat h'} & \hamiquat +{\quat h}
\end{bmatrix}, & \hamidq -{\dq h} & =\begin{bmatrix}\hamiquat -{\quat h} & \mymatrix 0_{4\times4}\\
\hamiquat -{\quat h'} & \hamiquat -{\quat h}
\end{bmatrix}.\label{eq:hamilton_eight}
\end{align}

We also define the mappings $\vector_{3}\,:\,\mathbb{H}_{p}\rightarrow\mathbb{R}^{3}$
and $\vector_{6}\,:\,\mathcal{H}_{p}\rightarrow\mathbb{R}^{6}$. Thus,
given a pure quaternion $\quat h\in\mathbb{H}_{p}$ such that $\quat h=\imag{\quat h}=h_{1}\imi+h_{2}\imj+h_{3}\imk$,
then $\vector_{3}\quat h=\begin{bmatrix}h_{1} & h_{2} & h_{3}\end{bmatrix}^{T}$.
Analogously, given a pure dual quaternion $\dq h\in\mathcal{H}_{p}$
such that $\dq h=\imag{\dq h}=h_{1}\imi+h_{2}\imj+h_{3}\imk+\dual\left(h_{4}\imi+h_{5}\imj+h_{6}\imk\right)$,
then $\vector_{6}\dq h=\begin{bmatrix}h_{1} & \cdots & h_{6}\end{bmatrix}^{T}$.

The logarithm of a unit quaternion given as in \eqref{eq:rotation}
yields \citep{kim1996compact} 
\begin{align}
\log\quat r_{i} & \triangleq\frac{\phi_{i}}{2}\quat n_{i}.\label{eq:logquat}
\end{align}

Similarly, the logarithm of a unit dual quaternion given as in \eqref{eq:rigidmotion}
is defined as \citep{Han2008}: 
\begin{align}
\log\dq x_{i} & \triangleq\frac{1}{2}(\phi_{i}\quat n_{i}+\dual\quat p_{i}),\label{eq:logdq}
\end{align}
where $\log\dq x_{i}\in\mathcal{H}_{p}$.

Let $\dq g\in\mathcal{H}_{p}$, such that $\dq g=\quat g+\dual\quat g'$,
the inverse mapping $\exp:\mathcal{H}_{p}\to\spinr$ is given by \citep{Adorno2011}
\begin{align}
\exp\dq g & =\exp\quat g+\dual\quat g'\exp\quat g,\label{eq:expdg}\\
\exp\quat g & =\begin{cases}
\cos\norm{\quat g}+\frac{\sin\norm{\quat g}}{\norm{\quat g}}\quat g & \text{if }\quat g\neq0,\\
1 & \text{otherwise.}
\end{cases}\label{eq:expg}
\end{align}
 Therefore, $\dq x=\exp(\log\dq x)$ and \eqref{eq:expdg} is an injective
mapping for $\phi_{i}\in[0,2\pi)$.

The twist $\dq{\xi}_{i}\in\mathcal{H}_{p}$ of frame $\frame i$ expressed
with respect to the inertial frame $\frame{}$ is defined as 
\begin{align}
\dq{\xi}_{i} & \triangleq\quat{\omega}_{i}+\dual(\dot{\quat p}_{i}+\quat p_{i}\times\quat{\omega}_{i}),\label{eq:twist2}
\end{align}
where $\quat{\omega}_{i}\in\mathbb{H}_{p}$ is the angular velocity
and $\dot{\quat p}_{i}\in\mathbb{H}_{p}$ is the linear velocity.
The cross-product for pure quaternions is given by 
\begin{equation}
\quat p_{i}\times\quat{\omega}_{i}=\frac{\quat p_{i}\quat{\omega}_{i}-\quat{\omega}_{i}\quat p_{i}}{2},\label{eq:crossproduct}
\end{equation}
which is equivalent to the vector cross-product in $\mathbb{R}^{3}$
thanks to the isomorphism between $\mathbb{H}_{p}$ and $\mathbb{R}^{3}$
under addition operations.

The derivative of $\dq x_{i}$ can be expressed by \citep{Adorno2017}
\begin{align}
\dot{\dq x}_{i} & =\frac{1}{2}\dq{\xi}_{i}\dq x_{i}.\label{eq:dqkinxi2}
\end{align}

\section{The differential logarithmic mapping \label{differential_logarithmic_mapping}}

In order to design the consensus protocols and the corresponding consensus-based
formation controllers, we use the differential logarithmic mapping
of dual quaternions. This differential mapping allows us to circumvent
the difficulties related to the topology of the non-Euclidean manifold
$\dq{\mathcal{S}}$. Indeed, as shown in \citep{Kussaba2017} the
set $\dq{\mathcal{S}}$ of unit dual quaternions can be regarded as
the product manifold $\mathbb{S}^{3}\times\mathbb{R}^{3}$. Therefore,
the consensus protocols usually found in the literature cannot be
directly applied to elements of $\dq{\mathcal{S}}$ because those
protocols assume an $n$-dimensional Euclidean space.

We extend the results of Kim et al. \citep{kim1996compact}, which
were proposed only for quaternions, to derive the differential logarithm
mapping for dual quaternions.\lyxdeleted{Bruno Vilhena Adorno}{Sat Jun 15 00:05:08 2019}{ }
\begin{lem}[\citep{kim1996compact}]
\label{thm:quaternion_logarithm} Consider $\quat r\in\mathbb{S}^{3}$,
with $\quat r=\cos\left(\phi/2\right)+\quat n\sin\left(\phi/2\right)$,
where $\quat n\in\mathbb{S}^{3}\cap\mathbb{H}_{p}$ and $\phi\in\left[0,2\pi\right)$,
and $\quat y=\left(y_{x}\imi+y_{y}\imj+y_{z}\imk\right)\in\mathbb{H}_{p}$
such that $\quat y=\log\quat r$. Thus 
\begin{equation}
\frac{\partial\vector_{4}\quat r}{\partial\vector_{3}\quat y}=\begin{bmatrix}-ay_{x} & -ay_{y} & -ay_{z}\\
by_{x}^{2}+a & by_{x}y_{y} & by_{x}y_{z}\\
by_{x}y_{y} & by_{y}^{2}+a & by_{y}y_{z}\\
by_{x}y_{z} & by_{y}y_{z} & by_{z}^{2}+a
\end{bmatrix},\label{eq:M}
\end{equation}
where 
\begin{align*}
a & =\frac{\sin\norm{\quat y}}{\norm{\quat y}}, & b & =\frac{\cos\norm{\quat y}}{\norm{\quat y}^{2}}-\frac{\sin\norm{\quat y}}{\norm{\quat y}^{3}}
\end{align*}
for $\quat y\neq0$; 
\[
\frac{\partial\vector_{4}\quat r}{\partial\vector_{3}\quat y}=\begin{bmatrix}\mymatrix 0_{1\times3}\\
\mymatrix I_{3}
\end{bmatrix},
\]
if $\quat y=0$.\lyxdeleted{Bruno Vilhena Adorno}{Sat Jun 15 00:05:08 2019}{ }
\end{lem}

\begin{proof}
See \citep{kim1996compact}.\lyxdeleted{Bruno Vilhena Adorno}{Sat Jun 15 00:05:08 2019}{ }
\end{proof}
The entries of $\partial\vector_{4}\quat r/\partial\vector_{3}\quat y$,
given in Theorem~\ref{thm:quaternion_logarithm}, depend on the coefficients
of $\quat y$, which is the logarithm of $\quat r\in\mathbb{S}^{3}$.
However, it is convenient to rewrite that matrix as a function of
only the coefficients of $\quat r$ in order to exploit some useful
properties later on.\lyxdeleted{Bruno Vilhena Adorno}{Sat Jun 15 00:05:08 2019}{ }
\begin{thm}[Alternative form of Lemma~\ref{thm:quaternion_logarithm}]
\label{th:rQy} Consider $\quat r=\left(r_{1}+r_{2}\imi+r_{3}\imj+r_{4}\imk\right)\in\mathbb{S}^{3}$,
with $\quat r=\cos\left(\phi/2\right)+\quat n\sin\left(\phi/2\right)$,
where $\quat n=\left(n_{x}\imi+n_{y}\imj+n_{z}\imk\right)\in\mathbb{S}^{3}\cap\mathbb{H}_{p}$
and $\phi\in\left[0,2\pi\right)$, and $\quat y=\left(y_{x}\imi+y_{y}\imj+y_{z}\imk\right)\in\mathbb{H}_{p}$
such that $\quat y\triangleq\log\quat r=\quat n\left(\phi/2\right)$.
Thus, 
\begin{gather}
\frac{\partial\vector_{4}\quat r}{\partial\vector_{3}\quat y}=\begin{bmatrix}-r_{2} & -r_{3} & -r_{4}\\
\Gamma n_{x}^{2}+\Theta & \Gamma n_{x}n_{y} & \Gamma n_{x}n_{z}\\
\Gamma n_{y}n_{x} & \Gamma n_{y}^{2}+\Theta & \Gamma n_{y}n_{z}\\
\Gamma n_{z}n_{x} & \Gamma n_{z}n_{y} & \Gamma n_{z}^{2}+\Theta
\end{bmatrix},\label{eq:L}
\end{gather}
where $\Gamma=r_{1}-\Theta$ and 
\begin{gather*}
\Theta=\begin{cases}
1 & \text{if }\phi=0,\\
\frac{\sin\left(\phi/2\right)}{\phi/2} & \text{otherwise.}
\end{cases}
\end{gather*}
\end{thm}

\begin{proof}
First, let us denote the matrix \eqref{eq:M} in Theorem~\ref{thm:quaternion_logarithm}
by $\mymatrix M=\left[m_{ij}\right]$ and the matrix \eqref{eq:L}
by $\mymatrix Q=\left[q_{ij}\right]$. For the case when $\phi=0$,
$r_{1}=1$, we have $\Gamma=0$ and then, clearly, $\mymatrix M=\mymatrix Q=\begin{bmatrix}\mymatrix 0_{3\times1} & \mymatrix I_{3}\end{bmatrix}^{T}.$

In order to show that $\mymatrix M=\mymatrix Q$ when $\phi\neq0$,
we start by verifying the terms of the first row. Using Fact~\ref{fact:identities with quaternion norm}
(see \ref{sec:AppA}) we obtain 
\begin{gather*}
m_{11}=-\frac{\sin\norm{\quat y}}{\norm{\quat y}}\frac{\phi}{2}n_{x}=-\sin\left(\frac{\phi}{2}\right)n_{x}=-r_{2}=q_{11}.
\end{gather*}
Analogously, $m_{12}=-\sin\left(\phi/2\right)n_{y}=-r_{3}=q_{12}$
and $m_{13}=-\sin\left(\phi/2\right)n_{z}=-r_{4}=q_{13}$.

Thanks to the symmetry of the the last three rows of $\mymatrix M$
and $\mymatrix Q$ only a few terms must be verified, namely $q_{21}$,
$q_{22}$, $q_{23}$, $q_{32}$, $q_{33}$, and $q_{43}$. Starting
from $m_{21}$ and using Fact~\ref{fact:identities with quaternion norm},
we obtain 
\begin{align*}
m_{21} & =by_{x}^{2}+a\\
 & =\left(\frac{\cos\norm{\quat y}}{\norm{\quat y}^{2}}-\frac{\sin\norm{\quat y}}{\norm{\quat y}^{3}}\right)y_{x}^{2}+\frac{\sin\norm{\quat y}}{\norm{\quat y}}\\
 & =\left(\frac{\cos\left(\phi/2\right)}{\left(\phi/2\right)^{2}}-\frac{\sin\left(\phi/2\right)}{\left(\phi/2\right)^{3}}\right)\left(n_{x}\frac{\phi}{2}\right)^{2}+\frac{\sin\left(\phi/2\right)}{\left(\phi/2\right)}\\
 & =\cos\left(\frac{\phi}{2}\right)n_{x}^{2}+\frac{\sin\left(\phi/2\right)}{\left(\phi/2\right)}\left(1-n_{x}^{2}\right)\\
 & =\left(r_{1}-\Theta\right)n_{x}^{2}+\Theta\\
 & =\Gamma n_{x}^{2}+\Theta=q_{21}.
\end{align*}
Analogously, $m_{32}=\Gamma n_{y}^{2}+\Theta=q_{32}$ and $m_{43}=\Gamma n_{z}^{2}+\Theta=q_{43}$.
Furthermore, 
\begin{align*}
m_{22}=by_{x}y_{y} & =\left(\frac{\cos\norm{\quat y}}{\norm{\quat y}^{2}}-\frac{\sin\norm{\quat y}}{\norm{\quat y}^{3}}\right)\left(\frac{\phi}{2}\right)^{2}n_{x}n_{y}\\
 & =\left(\cos\left(\frac{\phi}{2}\right)-\frac{\sin\left(\phi/2\right)}{\left(\phi/2\right)}\right)n_{x}n_{y}\\
 & =\Gamma n_{x}n_{y}=q_{22}.
\end{align*}
Analogously, $m_{23}=by_{x}y_{z}=\Gamma n_{x}n_{z}=q_{23}$ and $m_{33}=by_{y}y_{z}=\Gamma n_{y}n_{z}=q_{33}$,
which concludes the proof.\lyxdeleted{Bruno Vilhena Adorno}{Sat Jun 15 00:05:08 2019}{ }
\end{proof}
\begin{cor}
Consider $\quat r=\left(r_{1}+r_{2}\imi+r_{3}\imj+r_{4}\imk\right)\in\mathbb{S}^{3}$
and $\quat y\in\mathbb{H}_{p}$ such that $\quat y\triangleq\log\quat r$,
then 
\[
\lim_{\phi\rightarrow0}\frac{\partial\vector_{4}\quat r}{\partial\vector_{3}\quat y}=\begin{bmatrix}\mymatrix 0_{1\times3}\\
\mymatrix I_{3}
\end{bmatrix}.
\]
\end{cor}

\begin{proof}
Since $\quat r=\cos\left(\phi/2\right)+\quat n\sin\left(\phi/2\right)$,
then $\lim_{\phi\rightarrow0}r_{1}=1$ and $\lim_{\phi\rightarrow0}r_{l}=0$
for $l=\left\{ 2,3,4\right\} $. Defining $\Gamma$ and $\Theta$
as in Theorem~\ref{th:rQy}, $\lim_{\phi\rightarrow0}\Theta=1$,
thus 
\[
\lim_{\phi\rightarrow0}\Gamma=\lim_{\phi\rightarrow0}r_{1}-\lim_{\phi\rightarrow0}\Theta=0.
\]
Thus, 
\begin{align*}
\lim_{\phi\rightarrow0}\frac{\partial\vector_{4}\quat r}{\partial\vector_{3}\quat y} & =\lim_{\phi\rightarrow0}\begin{bmatrix}-r_{2} & -r_{3} & -r_{4}\\
\Gamma n_{x}^{2}+\Theta & \Gamma n_{x}n_{y} & \Gamma n_{x}n_{z}\\
\Gamma n_{y}n_{x} & \Gamma n_{y}^{2}+\Theta & \Gamma n_{y}n_{z}\\
\Gamma n_{z}n_{x} & \Gamma n_{z}n_{y} & \Gamma n_{z}^{2}+\Theta
\end{bmatrix}\\
 & =\begin{bmatrix}\mymatrix 0_{1\times3}\\
\mymatrix I_{3}
\end{bmatrix}.
\end{align*}
\end{proof}
Next, we extend Theorem~\ref{th:rQy} to find the mapping between
the derivative of a unit dual quaternion and the derivative of its
logarithm.\lyxdeleted{Bruno Vilhena Adorno}{Sat Jun 15 00:05:08 2019}{ }
\begin{thm}
\label{th:xQy}Consider $\dq x\in\dq{\mathcal{S}}$ such that $\dq x=\quat r+\dual\left(1/2\right)\quat p\quat r$,
with $\quat r\in\mathbb{S}^{3}$ and $\quat p\in\mathbb{H}_{p}$.
Thus, 
\[
\vector_{8}\dot{\dq x}=\underset{\mymatrix Q_{8}\left(\dq x\right)}{\underbrace{\begin{bmatrix}\quat Q\left(\quat r\right) & \mymatrix 0_{4\times3}\\
\frac{1}{2}\hamiquat +{\quat p}\quat Q\left(\quat r\right) & \hamiquat -{\quat r}\quat Q_{p}
\end{bmatrix}}}\vector_{6}\dot{\dq y},
\]
where 
\begin{align*}
\mymatrix Q\left(\quat r\right) & =\frac{\partial\vector_{4}\quat r}{\partial\vector_{3}\quat y}, & \mymatrix Q_{p} & =\begin{bmatrix}\mymatrix 0_{1\times3}\\
\mymatrix I_{3}
\end{bmatrix}, & \dq y & =\log\dq x.
\end{align*}
Furthermore, $\mymatrix Q_{8}\left(\dq x\right)\in\mathbb{R}^{8\times6}$
has full column rank; therefore, 
\begin{align*}
\mymatrix Q_{8}\left(\dq x\right)^{+}\mymatrix Q_{8}\left(\dq x\right) & =\mymatrix I
\end{align*}
and $\vector_{8}\dot{\dq x}=\myvec 0$ if and only if $\vector_{6}\dot{\dq y}=\myvec 0$.\lyxdeleted{Bruno Vilhena Adorno}{Sat Jun 15 00:05:08 2019}{ }
\end{thm}

\begin{proof}
Since $\dq x=\quat r+\dual\left(1/2\right)\quat p\quat r$ then 
\begin{align*}
\dot{\dq x} & =\dot{\quat r}+\dual\left(1/2\right)\left(\dot{\quat p}\quat r+\quat p\dot{\quat r}\right),
\end{align*}
hence 
\[
\vector_{8}\dot{\dq x}=\begin{bmatrix}\mymatrix I_{4} & \mymatrix 0_{4\times4}\\
\frac{1}{2}\hamiquat +{\quat p} & \hamiquat -{\quat r}
\end{bmatrix}\begin{bmatrix}\vector_{4}\dot{\quat r}\\
\frac{1}{2}\vector_{4}\dot{\quat p}
\end{bmatrix}.
\]
Using the fact that $\vector_{4}\dot{\quat r}=\myvec Q\left(\quat r\right)\vector_{3}\dot{\quat y}$
(see Theorem~\ref{th:rQy}) and $\log\dq x=\quat y+\dual\left(1/2\right)\quat p$,
with $\quat y=\log\quat r$, we obtain 
\begin{align*}
\vector_{8}\dot{\dq x} & =\begin{bmatrix}\mymatrix I_{4} & \mymatrix 0_{4\times4}\\
\frac{1}{2}\hamiquat +{\quat p} & \hamiquat -{\quat r}
\end{bmatrix}\begin{bmatrix}\myvec Q\left(\quat r\right)\vector_{3}\dot{\quat y}\\
\frac{1}{2}\myvec Q_{p}\vector_{3}\dot{\quat p}
\end{bmatrix}\\
 & =\underset{\mymatrix A}{\underbrace{\begin{bmatrix}\mymatrix I_{4} & \mymatrix 0_{4\times4}\\
\frac{1}{2}\hamiquat +{\quat p} & \hamiquat -{\quat r}
\end{bmatrix}}}\underset{\mymatrix B}{\underbrace{\begin{bmatrix}\mymatrix Q\left(\quat r\right) & \mymatrix 0_{4\times3}\\
\mymatrix 0_{4\times3} & \myvec Q_{P}
\end{bmatrix}}}\begin{bmatrix}\vector_{3}\dot{\quat y}\\
\frac{1}{2}\vector_{3}\dot{\quat p}
\end{bmatrix}\\
 & =\begin{bmatrix}\mymatrix Q\left(\quat r\right) & \mymatrix 0_{4\times3}\\
\frac{1}{2}\hamiquat +{\quat p}\mymatrix Q\left(\quat r\right) & \hamiquat -{\quat r}\myvec Q_{p}
\end{bmatrix}\vector_{6}\dot{\dq y}.
\end{align*}

In order to show that $\vector_{8}\dot{\dq x}=\myvec 0$ if and only
if $\vector_{6}\dot{\dq y}=\myvec 0$, it suffices to show that $\mymatrix Q_{8}\triangleq\mymatrix Q_{8}\left(\dq x\right)$
is full column rank (which implies that $\det\left(\mymatrix Q_{8}{}^{T}\mymatrix Q_{8}\right)\neq0$),
because in this case the left pseudoinverse exists and is defined
by $\mymatrix Q_{8}{}^{+}\triangleq\left(\mymatrix Q_{8}{}^{T}\mymatrix Q_{8}\right)^{-1}\mymatrix Q_{8}{}^{T}$.
Hence, the solution $\vector_{6}\dot{\dq y}=\mymatrix Q_{8}{}^{+}\vector_{8}\dot{\dq x}$
is unique (see Proposition~\ref{prop:existence_of_left_pseudo_inverse_and_unicity}
in \ref{sec:AppA}) and thus $\vector_{8}\dot{\dq x}=\myvec 0$ if
and only if $\vector_{6}\dot{\dq y}=\myvec 0$.

Since $\mymatrix A\in\mathbb{R}^{8\times8}$ and $\mymatrix B\in\mathbb{R}^{8\times6}$
we have from Corollary~2.5.10 of \citep{Bernstein2009} that 
\begin{equation}
\mathrm{rank}\mymatrix A+\mathrm{rank}\mymatrix B-8\leq\mathrm{rank}\mymatrix A\mymatrix B\leq\min\left\{ \mathrm{rank}\mymatrix A,\mathrm{rank}\mymatrix B\right\} .\label{eq:inequality rank}
\end{equation}
From Proposition~\ref{thm:Q is full column rank}, $\mymatrix Q\left(\quat r\right)$
is full column rank. Furthermore, as $\mymatrix Q_{P}$ is also full
column rank, $\mathrm{rank}\,\mymatrix B=6.$ Matrix $\mymatrix A$
is invertible (see Proposition~\ref{thm:A is invertible}), thus
$\mathrm{rank}\,\mymatrix A=8$, hence 
\[
8+6-8\leq\mathrm{rank\,}\mymatrix Q_{8}\left(\dq x\right)\leq\min\left\{ 8,6\right\} \implies\mathrm{rank\,}\mymatrix Q_{8}\left(\dq x\right)=6.
\]
As $\mymatrix Q_{8}\left(\dq x\right)$ is full column rank, the left
pseudoinverse $\mymatrix Q_{8}\left(\dq x\right)^{+}$ exists and,
from Proposition~\ref{prop:existence_of_left_pseudo_inverse_and_unicity},
we conclude that $\vector_{8}\dot{\dq x}=\myvec 0\iff\vector_{6}\dot{\dq y}=\myvec 0$.
\end{proof}

\section{Consensus Protocols\label{sec:consensus}}

In this section we design consensus protocols based on dual quaternions.
Since the group $\spinr$ of unit dual quaternions belongs to a non-Euclidean,
non-additive manifold, we cannot directly use the traditional consensus
protocols, which are mostly based on averaging the variables of interest.
This is due to the fact that directly averaging unit dual quaternions
does not produce meaningful values, as it generally does not yield
a unit dual quaternion.

A workaround to this problem is to choose an output for the system
that is not required to be a unit dual quaternion and thus can be
averaged without losing its group properties. To do that, we first
define the problem of output consensus on pure dual quaternions (i.e.,
elements of $\mathcal{H}_{p}$) and design a corresponding consensus
protocol. The advantage of such approach is that $\mathcal{H}_{p}$
is a six-dimensional Euclidean manifold, and thus the output consensus
protocol on $\mathcal{H}_{p}$ can be based only on linear operations.
Next, we extend the definition to take into account the problem of
pose consensus, where consensus must be achieved on elements of $\dq{\mathcal{S}}$,
and then we design a corresponding consensus protocol using the differential
logarithmic mapping presented in Section~\ref{differential_logarithmic_mapping}.

\subsection{Dual Quaternion Consensus}

Consider a multi-agent system with $n$ agents, in which each agent
has an output state given by the dual quaternion $\dq y_{i}\in\mathcal{H}_{p}$,
for $i=1,\ldots,n$. The topology of the information exchange in the
network is described by a directed graph, where the nodes represent
the agents and the edges the information flow, which can be unidirectional
or bidirectional, as described in Section~\ref{subsec:Algebraic-Graph-Theory}.
The output consensus problem is to make the multi-agent system reach
an agreement on the output variable of interest considering only the
information provided by neighbor agents. For that, we have the following
definition.\lyxdeleted{Bruno Vilhena Adorno}{Sat Jun 15 00:05:22 2019}{ }
\begin{defn}
\label{def:consensus} The multi-agent system with output variables
$\dq y_{i}(t)\in\mathcal{H}_{p},\,\forall i$, is said to asymptotically
achieve output consensus on the dual quaternion variable of interest
if and only if 
\begin{equation}
\lim_{t\rightarrow\infty}\left(\dq y_{i}(t)-\dq y_{j}(t)\right)=0,\,\forall i,j=1,\ldots,n.
\end{equation}
\end{defn}

Given the definition of output consensus, the following theorem shows
a consensus protocol that enables the multi-agent system to achieve
output consensus.\lyxdeleted{Bruno Vilhena Adorno}{Sat Jun 15 00:05:22 2019}{ }
\begin{thm}
\label{thm:consensus} The multi-agent system composed of $n$ agents
with system dynamics given by 
\begin{equation}
\dq u_{i}\triangleq\dot{\dq y}_{i},\label{eq:dynamic}
\end{equation}
for all $i=1,\ldots,n$, using the consensus protocol given by 
\begin{equation}
\dq u_{i}=-\sum_{j=1}^{n}a_{ij}\left(\dq y_{i}-\dq y_{j}\right),\label{eq:protocol}
\end{equation}
where $a_{ij}$ are the elements of the adjacency matrix \eqref{ruleaij}
of a directed graph $\mathcal{G}$ describing the network topology,
achieves output consensus according to Definition~\ref{def:consensus}
if and only if the network topology described by $\mathcal{G}$ has
a directed spanning tree.\lyxdeleted{Bruno Vilhena Adorno}{Sat Jun 15 00:05:22 2019}{ }
\end{thm}

\begin{proof}
The consensus problem in the dual quaternion variables $\dq y_{i}=\dq y_{j}$,
$\forall i,j$ can be transformed into a stability problem with an
extension of the tree-type transformation shown in \citep{Sun2009}.
Thus, for a multi-agent system with $n$ agents, we define $n-1$
error variables given by 
\begin{flalign}
\dq z_{i} & =\dq y_{1}-\dq y_{(i+1)},\qquad i=1,\ldots,n-1.\label{eq:arvore}
\end{flalign}
The remainder of the proof is given by the proof of stability of these
error variables by stacking $\dq z_{i}$ into a vector $\myvec z\in\mathcal{H}_{p}^{n-1}$,
where $\myvec z=[\dq z_{1}~\dq z_{2}~\ldots~\dq z_{(n-1)}]^{T}$,
since output consensus is asymptotically achieved if and only if $\myvec z$
goes to zero \citep{Sun2009}. Therefore, 
\begin{gather}
\myvec z=\underset{\mymatrix U}{\underbrace{\begin{bmatrix}1 & -1 & 0 & \cdots & 0\\
1 & 0 & -1 &  & 0\\
\vdots &  &  & \ddots & \vdots\\
1 & 0 & 0 & \cdots & -1
\end{bmatrix}}}\underset{\myvec y}{\underbrace{\left[\begin{array}{c}
\dq y_{1}\\
\dq y_{2}\\
\vdots\\
\dq y_{n}
\end{array}\right]}},\label{eq:zUx}
\end{gather}
where $\mymatrix U\in\mathbb{Z}^{\left(n-1\right)\times n}$ and $\myvec y\in\mathcal{H}_{p}^{n}$.
Considering \eqref{eq:zUx}, the inverse transformation is given by
\begin{gather}
\myvec y=\underset{\mymatrix 1_{n}}{\underbrace{\left[\begin{array}{c}
1\\
1\\
\vdots\\
1
\end{array}\right]}}\dq y_{1}+\underset{\mymatrix W}{\underbrace{\begin{bmatrix}0 & 0 & \cdots & 0\\
-1 & 0 & \cdots & 0\\
0 & -1 & \cdots & 0\\
\vdots & \vdots & \ddots & \vdots\\
0 & 0 & \cdots & -1
\end{bmatrix}}}\myvec z,\label{eq:xWz}
\end{gather}
thus $\myvec y=\myvec 1_{n}\dq y_{1}+\mymatrix W\myvec z$, where
$\mymatrix W\in\mathbb{Z}^{n\times\left(n-1\right)}$.

The closed-loop dynamics considering \eqref{eq:protocol} and \eqref{eq:dynamic}
gives 
\begin{align}
\dot{\dq y}_{i} & =-\sum_{j=1}^{n}a_{ij}\left(\dq y_{i}-\dq y_{j}\right)\label{eq:closedloop}\\
 & =-\Delta_{ii}\dq y_{i}+\sum_{j=1}^{n}a_{ij}\dq y_{j}\nonumber \\
 & =-\Delta_{ii}\dq y_{i}+\myvec a_{i}\myvec y,\nonumber 
\end{align}
where $\Delta_{ii}=\sum_{j=1}^{n}a_{ij}$ and $\myvec a_{i}\in\mathbb{R}^{1\times n}$
corresponds to the $i$-th row of the adjacency matrix (i.e., $\mymatrix A=\begin{bmatrix}\myvec a_{1}^{T} & \cdots & \mymatrix a_{n}^{T}\end{bmatrix}^{T}$).
Considering the whole multi-agent system, we obtain 
\begin{align}
\dot{\myvec y}=\begin{bmatrix}\dot{\dq y}_{1}\\
\vdots\\
\dot{\dq y}_{n}
\end{bmatrix} & =\begin{bmatrix}-\Delta_{11}\dq y_{1}+\myvec a_{1}\myvec y\\
\vdots\\
-\Delta_{nn}\dq y_{n}+\myvec a_{n}\myvec y
\end{bmatrix}\nonumber \\
 & =-\mymatrix{\Delta}\myvec y+\mymatrix A\myvec y\nonumber \\
 & =-\mymatrix L\myvec y,\label{eq:whole-dynamic}
\end{align}
where $\mymatrix{\Delta}$ and $\mymatrix L$ are the degree matrix
and Laplacian matrix, respectively (see Section~\ref{subsec:Algebraic-Graph-Theory}).

Taking the time-derivative of \eqref{eq:zUx}, and then considering
\eqref{eq:xWz} and \eqref{eq:whole-dynamic}, we have 
\begin{align*}
\dot{\myvec z} & =-\mymatrix U\mymatrix L\myvec y=-\mymatrix U\mymatrix L(\myvec 1_{n}\dq y_{1}+\mymatrix W\myvec z).
\end{align*}
Since $\mymatrix L\myvec 1_{n}=\myvec 0_{n}$ from \eqref{eq:laplacian},
it follows that 
\begin{equation}
\dot{\myvec z}=-\mymatrix U\mymatrix L\mymatrix W\myvec z.\label{eq:proofconsensus}
\end{equation}
The equilibrium point $\myvec z=\myvec 0_{n-1}$ in \eqref{eq:proofconsensus}
is asymptotically stable if and only if all the eigenvalues of $\mymatrix U\mymatrix L\mymatrix W$
have positive real parts. As shown in \citep{Savino2015}, this happens
if and only if $\mathcal{G}$ has a directed spanning tree. This concludes
the proof.\lyxdeleted{Bruno Vilhena Adorno}{Sat Jun 15 00:05:22 2019}{ }
\end{proof}
Therefore, Theorem~\ref{thm:consensus} tells us that a dynamical
system that can be written in the form of \eqref{eq:closedloop} achieves
consensus depending only on the network topology.

\subsection{Pose Consensus}

Since the dynamical system written in the form of \eqref{eq:closedloop}
relies on linear operations, which can be regarded as the most traditional
consensus algorithm, the result in Theorem~\ref{thm:consensus} can
only correctly perform averaging in Euclidean spaces \citep{jorstad2010distributed}.
For the case of rigid bodies, consensus protocols based on averaging
cannot be directly applied to elements of $\dq{\mathcal{S}}$ (that
is, to unit dual quaternions) because the group of rigid motions $\text{Spin(3)}\ltimes\mathbb{R}^{3}$
is a non-Euclidean manifold. Therefore, directly averaging unit dual
quaternions does not produce meaningful values, as it generally does
not yield a unit dual quaternion.

A workaround to this problem is to choose an output for the system
that is not required to be a unit dual quaternion and thus can be
averaged without losing its group properties, i.e. the logarithm $\dq y_{i}=\log\dq x_{i}$.
We now extend Definition~\ref{def:consensus} to the problem of pose
consensus in the set $\dq{\mathcal{S}}$ of unit dual quaternions.\lyxdeleted{Bruno Vilhena Adorno}{Sat Jun 15 00:05:22 2019}{ }
\begin{lem}
\label{lem:log} The multi-agent system with output variables $\dq y_{i}=\log\dq x_{i},\,\forall i$,
asymptotically achieves pose consensus in $\dq x_{i}\in\dq{\mathcal{S}}$
if consensus on $\dq y_{i}\in\mathcal{H}_{p}$ is asymptotically achieved.\lyxdeleted{Bruno Vilhena Adorno}{Sat Jun 15 00:05:22 2019}{ }
\end{lem}

\begin{proof}
Since $\dq x_{i}=\exp\left(\log\dq x_{i}\right)$, where $\exp\,:\,\mathcal{H}_{p}\rightarrow\dq{\mathcal{S}}$
\citep{Adorno2011}, then Definition~\ref{def:consensus} says that
\begin{equation}
\lim_{t\rightarrow\infty}\dq y_{i}(t)=\lim_{t\rightarrow\infty}\dq y_{j}(t),\,\forall i,j=1,\ldots,n,
\end{equation}
which implies 
\begin{align*}
\lim_{t\rightarrow\infty}\exp\dq y_{i}(t) & =\lim_{t\rightarrow\infty}\exp\dq y_{j}(t),\\
\implies\lim_{t\rightarrow\infty}\dq x_{i} & =\lim_{t\rightarrow\infty}\dq x_{j},\,\forall i,j=1,\ldots,n.
\end{align*}
\end{proof}
Lemma~\ref{lem:log} tells us that driving the agents to consensus
on the output variable $\dq y_{i}(t)$ implies consensus on the pose.
However, in general the kinematics is not given in the form of $\dot{\dq y}_{i}(t)=\myvec u_{i}(t)$.
Therefore, to show consensus on the pose we first write the problem
in a closed-loop that is known to achieve consensus, as in \eqref{eq:closedloop},
and then use the relation between $\dot{\dq y}_{i}$ and $\dot{\dq x}_{i}$
given in Theorem~\ref{th:xQy} to find the corresponding consensus
protocol according to the agent's kinematics to enable consensus on
the pose according to Lemma~\ref{lem:log}.

The next theorem summarizes the application of dual quaternion pose
consensus to multi-agent rigid-bodies.
\begin{thm}
\label{th:poseconsensus} Consider a group of $n$ agents described
as rigid-bodies with pose given by $\dq x_{i}$ as in \eqref{eq:rigidmotion}.
Let the system dynamics for each agent be given as 
\begin{equation}
\vector_{8}\dq u_{\dq x,i}\triangleq\vector_{8}\dot{\dq x}_{i},~i=1,\ldots,n,\label{eq:dqkinematics}
\end{equation}
with output 
\begin{equation}
\dq y_{i}=\log\dq x_{i},~i=1,\ldots,n.\label{eq:output}
\end{equation}
Under consensus protocol 
\begin{equation}
\vector_{8}\dq u_{\dq x,i}=-\mymatrix Q_{8}(\dq x_{i})\sum_{j=1}^{n}a_{ij}\vector_{6}\left(\dq y_{i}-\dq y_{j}\right),\label{eq:protocolxu}
\end{equation}
where $\mymatrix Q_{8}(\dq x_{i})\in\mathbb{R}^{8\times6}$ is given
in Theorem~\ref{th:xQy}, the multi-agent system asymptotically achieves
consensus in the dual quaternion output $\dq y_{i}\in\mathcal{H}_{p}$,
which implies consensus in the pose according to Lemma \ref{lem:log},
if and only if the network topology described by $\mathcal{G}$ has
a directed spanning tree.\lyxdeleted{Bruno Vilhena Adorno}{Sat Jun 15 00:05:22 2019}{ }
\end{thm}

\begin{proof}
From Theorem~\ref{thm:consensus}, a multi-agent system described
in the form of \eqref{eq:closedloop} is able to achieve output consensus
on $\dq y_{i}$ if and only if the graph $\mathcal{G}$ has a directed
spanning tree. Applying the $\vector_{6}$ operator in \eqref{eq:closedloop},
we obtain the equivalent equation 
\begin{equation}
\vector_{6}\dot{\dq y}_{i}=-\sum_{j=1}^{n}a_{ij}\vector_{6}\left(\dq y_{i}-\dq y_{j}\right).\label{dyy}
\end{equation}
From Theorem~\ref{th:xQy}, the relationship between $\dot{\dq x}_{i}$
and $\dot{\dq y}_{i}$ is given by 
\begin{equation}
\vector_{8}\dot{\dq x}_{i}=\vector_{8}\dq u_{\dq x,i}=\mymatrix Q_{8}(\dq x_{i})\vector_{6}\dot{\dq y}_{i}.\label{xuqy}
\end{equation}
Choosing $\vector_{8}\dq u_{\dq x,i}$ as \eqref{eq:protocolxu} yields
\begin{align}
-\mymatrix Q_{8}(\dq x_{i})\sum_{j=1}^{n}a_{ij}\vector_{6}\left(\dq y_{i}-\dq y_{j}\right) & =\mymatrix Q_{8}(\dq x_{i})\vector_{6}\dot{\dq y}_{i}.\label{eq:protocolxu_closedloop}
\end{align}
By Theorem~\ref{th:xQy}, $\mymatrix Q_{8}\left(\dq x_{i}\right)^{+}\mymatrix Q_{8}\left(\dq x_{i}\right)=\mymatrix I$,
therefore \eqref{eq:protocolxu_closedloop} implies \eqref{dyy},
which in turn implies output consensus according to Theorem~\ref{thm:consensus},
thus allowing the system to achieve consensus on the pose according
to Lemma~\ref{lem:log}.\lyxdeleted{Bruno Vilhena Adorno}{Sat Jun 15 00:05:22 2019}{ }
\end{proof}
\begin{cor}
Consider the dynamics of each agent expressed by 
\begin{equation}
\dot{\dq x}_{i}=\frac{1}{2}\dq{\xi}_{i}\dq x_{i},~i=1,\ldots,n,\label{eq:dqkinxi}
\end{equation}
where $\dq x_{i}$ is given in \eqref{eq:rigidmotion} and $\dq{\xi}_{i}$
is the corresponding twist given by \eqref{eq:twist2}. If the input
control actions are given as 
\begin{equation}
\vector_{8}\dq u_{\dq{\xi},i}\triangleq\vector_{8}\dq{\xi}_{i},~i=1,\ldots,n,
\end{equation}
consensus on the pose can be achieved by using protocol 
\begin{equation}
\vector_{8}\dq u_{\dq{\xi},i}=-2\hami -_{8}(\dq x_{i}^{*})\mymatrix Q_{8}(\dq x_{i})\sum_{j=1}^{n}a_{ij}\vector_{6}\left(\dq y_{i}-\dq y_{j}\right)\label{eq:protocolxiu}
\end{equation}
\end{cor}

\begin{proof}
Applying the $\vector_{8}$ operator in \eqref{eq:dqkinxi} and using
\eqref{eq:protocolxiu} yields 
\begin{align}
\vector_{8}\dot{\dq x}_{i} & =\frac{1}{2}\hamidq -{\dq x_{i}}\vector_{8}\dq u_{\dq{\xi},i}\\
 & =-\frac{1}{2}\hamidq -{\dq x_{i}}2\hami -_{8}(\dq x_{i}^{*})\mymatrix Q_{8}(\dq x_{i})\sum_{j=1}^{n}a_{ij}\vector_{6}\left(\dq y_{i}-\dq y_{j}\right).\label{eq:protocolxiu2}
\end{align}
Since $\vector_{8}\dot{\dq x}_{i}=\mymatrix Q_{8}(\dq x_{i})\vector_{6}\dot{\dq y}_{i}$,
and $\hamidq -{\dq x_{i}}\hamidq -{\dq x_{i}^{*}}=\mymatrix I$, $\forall\dq x_{i}\in\dq{\mathcal{S}}$,
and by Theorem~\ref{th:xQy} $\mymatrix Q_{8}\left(\dq x_{i}\right)^{+}\mymatrix Q_{8}\left(\dq x_{i}\right)=\mymatrix I$,
then \eqref{eq:protocolxiu2} implies \eqref{dyy}, which in turn
implies output consensus according to Theorem~\ref{thm:consensus},
thus allowing the system to achieve consensus on the pose according
to Lemma~\ref{lem:log}.\lyxdeleted{Bruno Vilhena Adorno}{Sat Jun 15 00:05:22 2019}{ }
\end{proof}
In the next section we write the formation control problem as a consensus
problem and present distributed control laws based on \eqref{eq:protocolxu}
and \eqref{eq:protocolxiu}. Furthermore, we consider the application
of the formation control to mobile manipulators. To that end, the
robot kinematics is explicitly taken into account.

\section{Consensus-Based Formation Control}

\label{sec:formation}

In a formation control problem, the goal is to make a group of agents
achieve desired relative poses in relation to neighbor agents and
keep this formation anywhere in space. Figure~\ref{fig:formation}
illustrates the case of a system composed of four agents in a two-dimensional
space, for better visualization, and formulates the problem in terms
of unit dual quaternions representing the poses.

The agents in the desired formation are shown in Figure~\ref{fig:dformation},
with the coordinate frame $(x$-axis,$y$-axis$)$ representing the
inertial reference frame, $(x_{c},y_{c})$ represents the center of
formation relative to the inertial frame, and $(x_{i},y_{i})$ represents
the local coordinate frame of the $i$-th agent. Each agent's desired
relative pose to the center of formation is represented by the rigid
motion given by the dual quaternion $\dq{\delta}_{i}\in\dq{\mathcal{S}}$.
The dual quaternion representing the relation from the inertial frame
to the center of formation, i.e. the pose of group formation, is represented
by $\dq x_{c}\in\dq{\mathcal{S}}$. This framework for defining the
relation has parallels with Cluster Space Control in \citep{Mas2017}
where the relative poses of the agents are defined by means of relative
transformations given by dual quaternions.

The pose of each agent is expressed by $\dq x_{i}\in\dq{\mathcal{S}}$,
and the desired relation $\dq{\delta}_{i}$ to the center of formation
is locally known (i.e., known by the $i$-th agent) and constant.
Thus, each agent has its local opinion regarding the center of formation,
which is considered as the agent's state and given by $\dq x_{c,i}=\dq x_{i}\dq{\delta}_{i}^{*}$,
as shown in Figure~\ref{fig:cformation}. A consensus-based approach
is used in order to enable all the agents to reach an agreement on
a common center of formation.

The information shared with neighboring agents is given by an output
given as the logarithmic mapping of the agent's state, i.e. 
\begin{equation}
\dq y_{c,i}=\log\dq x_{c,i}=\log(\dq x_{i}\dq{\delta}_{i}^{{*}}).\label{eq:x0ci}
\end{equation}

\begin{figure}
\centering \subfloat[Desired formation.\label{fig:dformation}]{
\global\long\def\svgwidth{0.45\columnwidth}%
 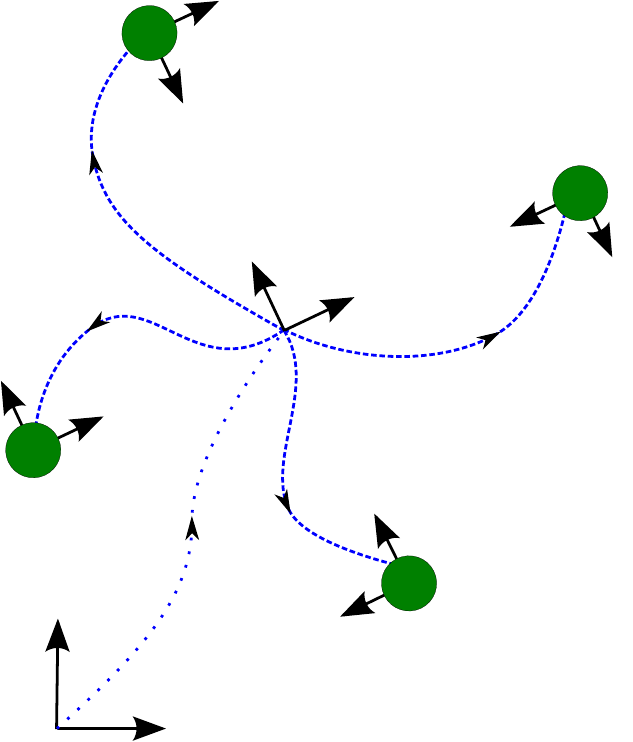\lyxdeleted{Bruno Vilhena Adorno}{Sat Jun 15 00:05:22 2019}{ }

}\subfloat[Consensus approach. All agents try to reach an agreement on the center
of formation $\protect\dq x_{c}$; i.e., $\protect\dq x_{c,i}=\protect\dq x_{c,j}$,
$\forall i,j$.\label{fig:cformation}]{
\global\long\def\svgwidth{0.45\columnwidth}%
 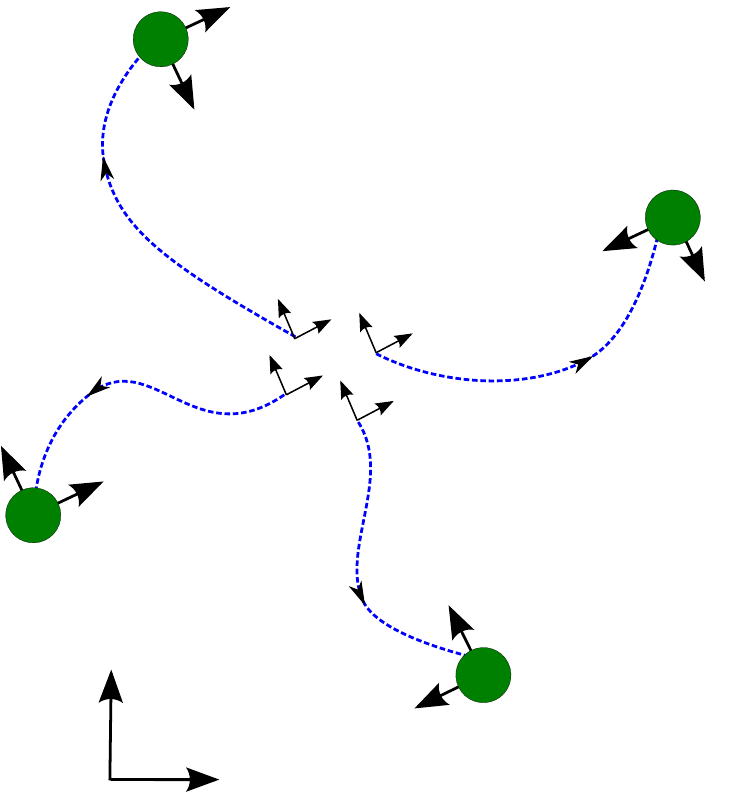 

}

\caption{\label{fig:formation}Each agent has a desired relation $\protect\dq{\delta}_{i}$
with the center of formation $\protect\dq x_{c}$. The information
exchanged is each agent's opinion on this center $\protect\dq x_{c,i}$.}
\end{figure}

Finally, since the desired $\dq{\delta}_{i}\in\mathcal{\dq{\mathcal{S}}}$
is locally defined (i.e., only the $i$-th agent has the information
about its constant $\dq{\delta}_{i}$) and the only variable that
$\dq x_{c,i}$ depends on is the pose $\dq x_{i}$, the formation
control problem can be defined as the problem of reaching output consensus
on the $\dq y_{c,i}$ variables. Therefore, the consensus protocol
that enables the system to achieve formation is presented in the following
theorem.\lyxdeleted{Bruno Vilhena Adorno}{Sat Jun 15 00:05:22 2019}{ }
\begin{thm}
\label{thm:formation} Consider a multi-agent system composed of $n$
agents described as rigid-bodies with pose expressed by $\dq x_{i}$
as given in \eqref{eq:rigidmotion}. Let the dynamics for each agent
be given by 
\begin{equation}
\vector_{8}\dq u_{\dq x,i}\triangleq\vector_{8}\dot{\dq x}_{i},~i=1,\ldots,n,\label{eq:kinematics}
\end{equation}
and each agent's output 
\begin{equation}
\dq y_{c,i}\triangleq\log\left(\dq x_{c,i}\right)=\log(\dq x_{i}\dq{\delta}_{i}^{*}),~i=1,\ldots,n,\label{eq:formation_control_output}
\end{equation}
with $\dq{\delta}_{i}$ being the desired pose in relation to the
center of formation. By means of the consensus protocol given by 
\begin{equation}
\vector_{8}\dq u_{\dq x,i}=-\hami -_{8}(\dq{\delta}_{i})\mymatrix Q_{8}(\dq x_{c,i})\sum_{j=1}^{n}a_{ij}\vector_{6}\left(\dq y_{c,i}-\dq y_{c,j}\right),\label{eq:protocol-form}
\end{equation}
where $a_{ij}$ are the elements of the adjacency matrix of the directed
graph $\mathcal{G}$ describing the network topology, the multi-agent
system asymptotically achieves formation if and only if the graph
$\mathcal{G}$ has a directed spanning tree.\lyxdeleted{Bruno Vilhena Adorno}{Sat Jun 15 00:05:22 2019}{ }
\end{thm}

\begin{proof}
From Theorem~\ref{th:xQy}, 
\begin{equation}
\vector_{8}\dot{\dq x}_{c,i}=\mymatrix Q_{8}(\dq x_{c,i})\vector_{6}\dot{\dq y}_{c,i}.\label{eq40}
\end{equation}
Since $\dq{\delta}_{i}$ is constant, the time-derivative of the agent's
state ${\dq x}_{c,i}=\dq x_{i}\dq{\delta}_{i}^{*}$ yields 
\begin{align}
\dot{\dq x}_{c,i}=\dot{\dq x}_{i}\dq{\delta}_{i}^{*}\implies\dot{\dq x}_{i} & =\dot{\dq x}_{c,i}\dq{\delta}_{i},\label{eq401}
\end{align}
because $\dq{\delta}_{i}^{*}\dq{\delta}_{i}=1$ as $\dq{\delta}_{i}\in\dq{\mathcal{S}}$.
Applying the Hamilton and $\vector_{8}$ operators in \eqref{eq401}
and taking $\vector_{8}\dot{\dq x}_{c,i}$ from \eqref{eq40} results
in 
\begin{equation}
\vector_{8}\dot{\dq x}_{i}=\hami -_{8}(\dq{\delta}_{i})\mymatrix Q_{8}(\dq x_{c,i})\vector_{6}\dot{\dq y}_{c,i}.\label{eq43}
\end{equation}

From Theorem~\ref{thm:consensus}, a system is able to achieve output
consensus on $\dq y_{c,i}\in\mathcal{H}_{p}$ if the closed-loop dynamics
of each agent is given by 
\begin{equation}
\vector_{8}\dot{\dq y}_{c,i}=-\sum_{j=1}^{n}a_{ij}\vector_{6}\left(\dq y_{c,i}-\dq y_{c,j}\right),\label{dyyc}
\end{equation}
and if and only if the graph $\mathcal{G}$ has a directed spanning
tree. Choosing $\vector_{8}\dq u_{\dq x,i}$ as in \eqref{eq:protocol-form},
considering \eqref{eq:kinematics} and \eqref{eq43}, and using the
fact that $\hamidq -{\dq{\delta}_{i}}$ is invertible and that, by
Theorem~\ref{th:xQy} $\mymatrix Q_{8}\left(\dq x_{c,i}\right)^{+}\mymatrix Q_{8}\left(\dq x_{c,i}\right)=\mymatrix I$,
then \eqref{dyyc} is satisfied, and the system achieves output consensus
according to Theorem~\ref{thm:consensus}. As a consequence, by Lemma~\ref{lem:log}
the system achieves pose consensus on the center of formation $\dq x_{c}=\lim_{t\rightarrow\infty}\dq x_{c,i}$,
$\forall i$, and because each $\dq{\delta}_{i}$ is locally known,
the final pose of each agent is given by $\dq x_{i}=\dq x_{c}\dq{\delta}_{i}$,
$\forall i$, which ensures the desired formation. This completes
the proof.\lyxdeleted{Bruno Vilhena Adorno}{Sat Jun 15 00:05:22 2019}{ }
\end{proof}
\begin{cor}
\label{cor:formation_with_twist_input} If the dynamics of each agent
is expressed by \eqref{eq:dqkinxi} and the input control actions
are given by 
\begin{equation}
\vector_{8}\dq u_{\dq{\xi},i}\triangleq\vector_{8}\dq{\xi}_{i},~i=1,\ldots,n,\label{eq:formation_control_twist_input}
\end{equation}
consensus-based formation can be achieved by using the consensus protocol
\begin{equation}
\vector_{8}\dq u_{\dq{\xi},i}=-2\hami -_{8}(\dq x_{i}^{*})\hami -_{8}(\dq{\delta}_{i})\mymatrix Q_{8}(\dq x_{c,i})\sum_{j=1}^{n}a_{ij}\vector_{6}\left(\dq y_{c,i}-\dq y_{c,j}\right),\label{eq:protocolxiuform}
\end{equation}
if and only if the graph $\mathcal{G}$ describing the network topology
has a directed spanning tree.\lyxdeleted{Bruno Vilhena Adorno}{Sat Jun 15 00:05:22 2019}{ }
\end{cor}

\begin{proof}
From \eqref{eq:dqkinxi} and \eqref{eq:formation_control_twist_input}
we obtain 
\begin{align}
\vector_{8}\dot{\dq x}_{i} & =\frac{1}{2}\hamidq -{\dq x_{i}}\vector_{8}\dq u_{\dq{\xi},i}.\label{eq:input_twist}
\end{align}
Replacing \eqref{eq43} and the consensus protocol \eqref{eq:protocolxiuform}
in \eqref{eq:input_twist}, and using the facts that $\hamidq -{\dq x_{i}}\hamidq -{\dq x_{i}^{*}}=\mymatrix I$,
the matrix $\hamidq -{\dq{\delta}_{i}}$ is invertible, and $\mymatrix Q_{8}\left(\dq x_{c,i}\right)^{+}\mymatrix Q_{8}\left(\dq x_{c,i}\right)=\mymatrix I$
by Theorem~\ref{th:xQy}, then \eqref{dyyc} is satisfied, which
ensures the desired formation according to the same argument used
in Theorem~\ref{thm:formation}. This completes the proof.\lyxdeleted{Bruno Vilhena Adorno}{Sat Jun 15 00:05:22 2019}{ }
\end{proof}
\begin{rem}
It can be shown that $\hami -_{8}(\dq x_{i}^{*})\hami -_{8}(\dq{\delta}_{i})=\hami -_{8}(\dq x_{c,i}^{*})$,
which gives an equivalence between \eqref{eq:protocolxiuform} and
\eqref{eq:protocolxiu} when comparing $\dq x_{c,i}$ to $\dq x_{i}$.
\end{rem}

The extension of Theorem~\ref{thm:formation} to time-varying formations
is straightforward as long as we assume that the $i$-th agent knows
it own time-varying desired relation $\dq{\delta}_{i}(t)$ to the
center of formation, as shown in the next corollary.
\begin{cor}
Consider a multi-agent system composed of $n$ agents, described as
rigid-bodies, with dynamics given by \eqref{eq:kinematics} and each
agent's output given by \eqref{eq:formation_control_output}, where
$\dq{\delta}_{i}\triangleq\dq{\delta}_{i}(t)$ is the desired time-varying
pose in relation to the center of formation. By means of the consensus
protocol given by 
\begin{equation}
\vector_{8}\dq u_{\dq x,i}=-\hami -_{8}(\dq{\delta}_{i})\mymatrix Q_{8}(\dq x_{c,i})\sum_{j=1}^{n}a_{ij}\vector_{6}\left(\dq y_{c,i}-\dq y_{c,j}\right)-\vector_{8}\left(\dq x_{i}\dot{\dq{\delta}}_{i}^{*}\dq{\delta}_{i}\right),\label{eq:protocol-time-varying-form}
\end{equation}
where $a_{ij}$ are the elements of the adjacency matrix of the directed
graph $\mathcal{G}$ describing the network topology, the multi-agent
system asymptotically achieves formation if and only if the graph
$\mathcal{G}$ has a directed spanning tree.
\end{cor}

\begin{proof}
Since $\dq x_{c,i}=\dq x_{i}\dq{\delta}_{i}^{*}$ then $\dot{\dq x}_{c,i}=\dot{\dq x}_{i}\dq{\delta}_{i}^{*}+\dq x_{i}\dot{\dq{\delta}}_{i}^{*}$,
therefore $\dot{\dq x}_{c,i}\dq{\delta}_{i}-\dq x_{i}\dot{\dq{\delta}}_{i}^{*}\dq{\delta}_{i}=\dot{\dq x}_{i}$.
Using Theorem~\ref{th:xQy}, we obtain 
\begin{align}
\vector_{8}\dot{\dq x_{i}} & =\hami -_{8}\left(\dq{\delta}_{i}\right)\mymatrix Q_{8}\left(\dq x_{c,i}\right)\vector_{6}\dot{\dq y}_{c,i}-\vector_{8}\left(\dq x_{i}\dot{\dq{\delta}}_{i}^{*}\dq{\delta}_{i}\right).\label{eq:open-loop-formation-dynamics}
\end{align}
Since each agent's dynamics is given \eqref{eq:kinematics}, then
\eqref{eq:open-loop-formation-dynamics} is equal to \eqref{eq:protocol-time-varying-form}.
Using the fact that $\hami -_{8}(\dq{\delta}_{i})$ is invertible
and $\mymatrix Q_{8}\left(\dq x_{c,i}\right)^{+}\mymatrix Q_{8}\left(\dq x_{c,i}\right)=\mymatrix I$
by Theorem~\ref{th:xQy}, the closed-loop dynamics is reduced to
\eqref{eq:closedloop}, which by Theorem~\ref{thm:consensus} ensures
output consensus in the center of formation if and only if the graph
$\mathcal{G}$ has a directed spanning tree. As a consequence, time-varying
formation control is achieved.
\end{proof}

\subsection{Formation Control of Holonomic Mobile Manipulators}

The result presented in Theorem~\ref{thm:formation} can be directly
extended to a multi-agent system composed of multiple mobile manipulators.
In this case, the objective is to achieve desired formations for the
set of end-effectors of mobile manipulators and let each robot generate
its own motion in order to move the end-effector according to the
reference provided by the consensus protocol. The advantage of using
such abstraction is that the consensus protocols are used to determine,
in a decentralized way, how each robot's end-effector should be, regardless
of the topology and dimension of the robots' configuration spaces.
In fact, since the robots use \emph{local} motion controllers, the
result presented in Theorem~\ref{thm:formation} can be applied to
a highly heterogeneous multi-agent system\footnote{For example, the idea presented in this section could be applied to
a system composed of mobile manipulators and aerial manipulators.
However, in this paper we restrict ourselves to holonomic mobile manipulators.}, as long as each agent is capable of following the reference provided
by the consensus protocols.

Each robot is characterized by two main equations (see Section~\ref{sec:wholebody}):
the forward kinematics (FK) and the differential forward kinematics
(DFK). Let $\myvec q_{i}\in\mathbb{R}^{m_{i}}$ be the $m_{i}$-dimensional
vector corresponding to the $i$-th robot's configuration. The corresponding
robot end-effector pose $\dq x_{e,i}\in\dq{\mathcal{S}}$ is given
by 
\begin{equation}
\dq x_{e,i}=\dq f_{i}\left(\myvec q_{i}\right)\label{eq:general_fkm}
\end{equation}
where $\dq f_{i}\,:\,\mathbb{R}^{m_{i}}\rightarrow\dq{\mathcal{S}}$
is the FK of the $i$-th robot. In case of mobile manipulators, this
function is explicitly given by \eqref{eq:fkm derivative}. The DFK
is obtained by taking the time-derivative of \eqref{eq:general_fkm},
which yields 
\begin{align}
\vector_{8}\dot{\dq x}_{e,i} & =\mymatrix J_{w,i}\dot{\myvec q_{i}},\label{eq:general_whole_body_jacobian}
\end{align}
where $\myvec J_{w,i}\in\mathbb{R}^{8\times m_{i}}$ is the robot
(dual quaternion) Jacobian. In case of holonomic mobile manipulators,
this Jacobian is known as whole-body Jacobian (i.e., the Jacobian
that takes into account both the mobile base and manipulator) and
is given explicitly by \eqref{eq:whole-body-jacobian-of-mobile-manipulators}.
Using \eqref{eq:general_whole_body_jacobian}, the following theorem
provides the necessary and sufficient conditions for the formation
control of the end-effectors of a multi-agent system composed of multiple
mobile manipulators.\lyxdeleted{Bruno Vilhena Adorno}{Sat Jun 15 00:05:22 2019}{ }
\begin{thm}
Consider a multi-agent system composed of $n$ holonomic mobile manipulators
whose forward kinematics is given by \eqref{eq:general_fkm} and the
differential forward kinematics is given by \eqref{eq:general_whole_body_jacobian}.
Let the control input for each robot be given by 
\begin{equation}
\dq u_{\myvec q,i}\triangleq\dot{\myvec q}_{i},~i=1,\ldots,n,\label{eq:low-level-kinematics}
\end{equation}
and each agent's output be given by 
\begin{equation}
\dq y_{ce,i}\triangleq\log\left(\dq x_{ce,i}\right)=\log(\dq x_{e,i}\dq{\delta}_{i}^{*}),~i=1,\ldots,n,
\end{equation}
where $\dq x_{ce,i}\triangleq\dq x_{e,i}\dq{\delta}_{i}^{*}$ is the
opinion of the $i$-th agent related to the center of formation, $\dq x_{e,i}$
is the end-effector pose given by \eqref{eq:general_fkm}, and $\dq{\delta}_{i}$
is the desired end-effector pose with respect to the center of formation.

By means of the control input given by 
\begin{equation}
\myvec u_{\myvec q,i}=\mymatrix J_{w,i}^{\dagger}\vector_{8}\dq u_{\dq x,i},\label{eq:low level input}
\end{equation}
where $\mymatrix J_{w,i}^{\dagger}$ is the generalized Moore-Penrose
pseudoinverse of $\mymatrix J_{w,i}$, and the consensus protocol
$\vector_{8}\dq u_{\dq x,i}$ is given by

\begin{equation}
\vector_{8}\dq u_{\dq x,i}=-\hami -_{8}(\dq{\delta}_{i})\mymatrix Q_{8}(\dq x_{ce,i})\sum_{j=1}^{n}a_{ij}\vector_{6}\left(\dq y_{ce,i}-\dq y_{ce,j}\right),\label{eq:protocol-form-1}
\end{equation}
the multi-agent system asymptotically achieves formation if and only
if the graph $\mathcal{G}$ describing the network topology has a
directed spanning tree and $\vector_{8}\dq u_{\dq x,i}$ is in the
range space of $\mymatrix J_{w,i}$.\footnote{The range space of $\mymatrix M\in\mathbb{R}^{m\times n}$ is defined
as $\mathrm{range}\,\mymatrix M\triangleq\left\{ \mymatrix M\myvec v\,:\,\myvec v\in\mathbb{R}^{n}\right\} .$}\lyxdeleted{Bruno Vilhena Adorno}{Sat Jun 15 00:05:22 2019}{ }
\end{thm}

\begin{proof}
First we prove that $\vector_{8}\dq u_{\dq x,i}$ is in the range
space of $\mymatrix J_{w,i}$ if and only if $\vector_{8}\dq u_{x,i}=\mymatrix J_{w,i}\mymatrix J_{w,i}^{\dagger}\vector_{8}\dq u_{x,i}$.
Let $\mymatrix J_{w,i}\in\mathbb{R}^{8\times n}$, if $\vector_{8}\dq u_{\dq x,i}\in\mathrm{range}\,\mymatrix J_{w,i}$
then $\exists\myvec v\in\mathbb{R}^{n}$ such that $\vector_{8}\dq u_{\dq x,i}=\mymatrix J_{w,i}\myvec v$.
Since $\mymatrix J_{w,i}\mymatrix J_{w,i}^{\dagger}\mymatrix J_{w,i}=\mymatrix J_{w,i}$
(see \citep{Bernstein2009}), then $\vector_{8}\dq u_{\dq x,i}=\mymatrix J_{w,i}\myvec v=\mymatrix J_{w,i}\mymatrix J_{w,i}^{\dagger}\mymatrix J_{w,i}\myvec v=\mymatrix J_{w,i}\mymatrix J_{w,i}^{\dagger}\vector_{8}\dq u_{\dq x,i}$.
Thus we conclude that 
\begin{gather}
\vector_{8}\dq u_{\dq x,i}\in\mathrm{range}\,\mymatrix J_{w,i}\implies\vector_{8}\dq u_{x,i}=\mymatrix J_{w,i}\mymatrix J_{w,i}^{\dagger}\vector_{8}\dq u_{x,i}.\label{eq:range_space_if}
\end{gather}
Conversely, if $\mymatrix J_{w,i}\mymatrix J_{w,i}^{\dagger}\vector_{8}\dq u_{x,i}=\vector_{8}\dq u_{x,i}$
then $\exists\myvec v'\triangleq\mymatrix J_{w,i}^{\dagger}\vector_{8}\dq u_{x,i}$
such that $\mymatrix J_{w,i}\myvec v'=\vector_{8}\dq u_{x,i}$, which
implies that $\vector_{8}\dq u_{\dq x,i}\in\mathrm{range}\,\mymatrix J_{w,i}$.
Hence, 
\begin{equation}
\vector_{8}\dq u_{\dq x,i}\in\mathrm{range}\,\mymatrix J_{w,i}\impliedby\vector_{8}\dq u_{x,i}=\mymatrix J_{w,i}\mymatrix J_{w,i}^{\dagger}\vector_{8}\dq u_{x,i}.\label{eq:range_space_only_if}
\end{equation}
From \eqref{eq:range_space_if} and \eqref{eq:range_space_only_if}
we conclude that 
\begin{equation}
\vector_{8}\dq u_{\dq x,i}\in\mathrm{range}\,\mymatrix J_{w,i}\iff\vector_{8}\dq u_{x,i}=\mymatrix J_{w,i}\mymatrix J_{w,i}^{\dagger}\vector_{8}\dq u_{x,i}.\label{eq:range_space_only_if-and-only-if}
\end{equation}

Using \eqref{eq:low-level-kinematics} in \eqref{eq:general_whole_body_jacobian}
yields $\vector_{8}\dot{\dq x}_{e,i}=\mymatrix J_{w,i}\dq u_{\myvec q,i}$.
Considering \eqref{eq:low level input} we obtain 
\begin{align}
\vector_{8}\dot{\dq x}_{e,i} & =\mymatrix J_{w,i}\mymatrix J_{w,i}^{\dagger}\vector_{8}\dq u_{\dq x,i}.\label{eq:single_integrator_end_effector}
\end{align}
Since $\dq x_{e,i}=\dq x_{ce,i}\dq{\delta}_{i}$, with $\dq{\delta}_{i}$
constant, we use Theorem~\ref{th:xQy} to obtain 
\begin{align}
\vector_{8}\dot{\dq x}_{e,i} & =\hamidq -{\dq{\delta}_{i}}\vector_{8}\dot{\dq x}_{ce,i}\nonumber \\
 & =\hamidq -{\dq{\delta}_{i}}\mymatrix Q_{8}\left(\dq x_{ce,i}\right)\vector_{6}\dot{\dq y}_{ce,i}.\label{eq:formation_control_effector_logarithm}
\end{align}
Assuming that \eqref{eq:range_space_only_if-and-only-if} holds, then
\eqref{eq:single_integrator_end_effector} results in $\vector_{8}\dot{\dq x}_{e,i}=\vector_{8}\dq u_{\dq x,i}$.
Therefore, we use the consensus protocol \eqref{eq:protocol-form-1}
together with \eqref{eq:formation_control_effector_logarithm}, and
use the fact that $\hamidq -{\dq{\delta}_{i}}$ is invertible and
$\mymatrix Q_{8}\left(\dq x_{ce,i}\right)^{+}\mymatrix Q_{8}\left(\dq x_{ce,i}\right)=\mymatrix I$,
to obtain 
\begin{equation}
\vector_{6}\dot{\dq y}_{ce,i}=-\sum_{j=1}^{n}a_{ij}\vector_{6}\left(\dq y_{ce,i}-\dq y_{ce,j}\right).\label{eq:effectors-closed-loop-dynamics}
\end{equation}
From Theorem~\ref{thm:consensus}, if the closed-loop dynamics of
each agent is given by \eqref{eq:effectors-closed-loop-dynamics},
the system is able to achieve output consensus on $\dq y_{ce,i}\in\mathcal{H}_{p}$
if and only if the graph $\mathcal{G}$ describing the network topology
has a directed spanning tree.

As a consequence, if the aforementioned conditions are fulfilled (i.e.,
$\vector_{8}\dq u_{\dq x,i}\in\mathrm{range}\,\mymatrix J_{w,i}$
and $\mathcal{G}$ has a directed spanning tree), by Lemma~\ref{lem:log}
the system achieves pose consensus on the center of formation $\dq x_{ce}=\lim_{t\rightarrow\infty}\dq x_{ce,i}$,
$\forall i$, and because each $\dq{\delta}_{i}$ is locally known,
the final pose of each end-effector is given by $\dq x_{e,i}=\dq x_{ce}\dq{\delta}_{i}$,
$\forall i$, which ensures the desired formation. This completes
the proof.\lyxdeleted{Bruno Vilhena Adorno}{Sat Jun 15 00:05:22 2019}{ }
\end{proof}
\begin{rem}
The reference $\vector_{8}\dq u_{\dq x,i}$ generated by the consensus
protocol \eqref{eq:protocol-form-1} is always in the range space
of the Jacobian matrix $\mymatrix J_{w,i}$ as long as the $i$-th
manipulator is not in a singular configuration or has not reached
its joint limits (both in position and velocity).\lyxdeleted{Bruno Vilhena Adorno}{Sat Jun 15 00:05:22 2019}{ }
\end{rem}

\section{Numerical Examples and Experiments}

\label{sec:Examples} This section presents numerical examples and
experiments with real robots to illustrate the applicability of the
consensus-based formation control. First, a simple numerical simulation
is performed by considering five free-flying agents that are supposed
to make a circular formation in an arbitrary location. Another simulation
is then performed by considering 100 free-flying agents in a time-varying
formation scenario to show the scalability of the proposed method.
Finally, we perform an experiment with two mobile manipulators in
a task of decentralized cooperative manipulation.

In both numerical examples and experiments, we used DQ Robotics,\footnote{\url{https://dqrobotics.github.io/}}
a standalone open-source robotics library that provides dual quaternion
algebra and kinematic calculation algorithms in MATLAB, Python, and
C++. The numerical simulations were performed in Matlab whereas C++
was used for the implementation on the real robots.

\subsection{Formation control of free-flying agents}

In this example, all agents must be equally distributed along a circumference
such that the final formation is a circle with radius equal to 0.5~m.
A coordinate system $\frame c\left(o_{c},x_{c},y_{c},z_{c}\right)$
is located at the center of the circle with the $z_{c}$-axis being
normal to the plane containing the circle. Each free-flying agent
is represented by a coordinate system $\frame i\left(o_{i},x_{i},y_{i},z_{i}\right)$
with corresponding unit dual quaternion $\dq x_{i}$. The desired
transformation $\dq{\delta}_{i}$ with respect to the center of formation
for the $i$-th agent is defined such that the agents are equally
distributed in a complete revolution around the $z_{c}$-axis with
the $x_{i}$-axis being tangent to the circumference and $y_{i}$
pointing towards the center. More specifically, given $n$ agents,
the desired transformation $\dq{\delta}_{i}$ of the $i$-th agent
is given by 
\begin{align}
\dq{\delta}_{i} & \triangleq\quat r_{\delta,i}\left(1+\dual\frac{1}{2}\quat p_{\delta,i}\right),\label{eq:deltas}
\end{align}
where 
\begin{align}
\quat r_{\delta,i} & =\cos\left(\frac{\phi_{\delta,i}}{2}\right)+\imk\sin\left(\frac{\phi_{\delta,i}}{2}\right)\label{eq:delta2}
\end{align}
and 
\begin{align}
\phi_{\delta,i} & =\frac{2\pi(i-1)}{n}, & \quat p_{\delta,i} & =-0.5\imj.\label{eq:delta1}
\end{align}

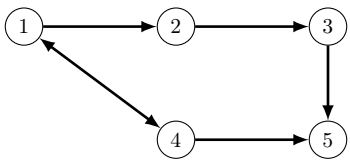
\begin{figure}[th]
\centering \begin{tikzpicture}[scale=1,auto=left,every node/.style=circle]
	\node[circle,draw,scale=.8] (n4) at (2,0)  {4};
	\node[circle,draw,scale=.8] (n5) at (4,0)  {5};
	\node[circle,draw,scale=.8] (n3) at (4,1.5)  {3};
	\node[circle,draw,scale=.8] (n2) at (2,1.5)  {2};
	\node[circle,draw,scale=.8] (n1) at (0,1.5)  {1};
	
	\pgfsetarrows{latex-latex} \foreach \from/\to in {n1/n4}
	\draw[line width=1pt] (\from) -- (\to);
	\pgfsetarrows{-latex} \foreach \from/\to in {n1/n2,n2/n3,n3/n5,n4/n5}
	\draw[line width=1pt] (\from) -- (\to);                 
	\end{tikzpicture} \caption{Network topology.\label{fig:network}}
\end{figure}

For any initial position, the system must achieve formation, as described
by $\dq{\delta}_{i}$ in \eqref{eq:deltas}, \emph{anywhere} in the
space. The network topology, which is depicted in Figure~\ref{fig:network},
is a directed graph with a directed spanning tree, and does not require
to be strongly connected. For simulation, the numerical integration
of $\dq x_{i}$ is carried out, as presented in \citep{Adorno2017},
by the formula 
\begin{equation}
\dq x_{i}(t+\Delta t)=\exp\left(\dfrac{\Delta t}{2}{\displaystyle \dq{\xi}_{i}}\right)\dq x_{i}(t),
\end{equation}
where $\Delta t$ is the time interval of integration, and the exponential
map $\exp(\cdot)$ is given by \eqref{eq:expdg}. Furthermore, the
control input for each agent is calculated by using \eqref{eq:protocolxiuform}.

In the first simulation, five free-flying agents are considered (i.e.,
$n=5$) and the result is shown in Figure~\ref{fig:simulation},
in which the initial poses of the agents are randomly chosen and marked
by the bolder frame $\dq x_{i}(0)$, for $i=1,\ldots,5$, the initial
local opinion regarding the center of formation $\dq x_{c,i}(0)=\dq x_{i}(0)\dq{\delta}_{i}^{*}$
is the thinner frame, the trajectories executed by each agent are
shown by the continuous bolder lines, and the trajectory of the local
center of formation is shown by the thinner dotted line while achieving
consensus on a common center of formation. The final circular formation
is shown at the center of the figure. The state-trajectories for each
coefficient of $\dq y_{c,i}(t)=y_{2c,i}\imi+y_{3c,i}\imj+y_{4c,i}\imk+\dual(y_{6c,i}\imi+y_{7c,i}\imj+y_{8c,i}\imk)$
are shown in Figure~\ref{fig:simulation1y} as the agents achieve
output consensus, which by Corollary~\ref{cor:formation_with_twist_input}
implies that the system achieves formation.

\begin{figure}
\centering{}\includegraphics[viewport=0bp 0bp 800bp 397bp,width=0.5\columnwidth]{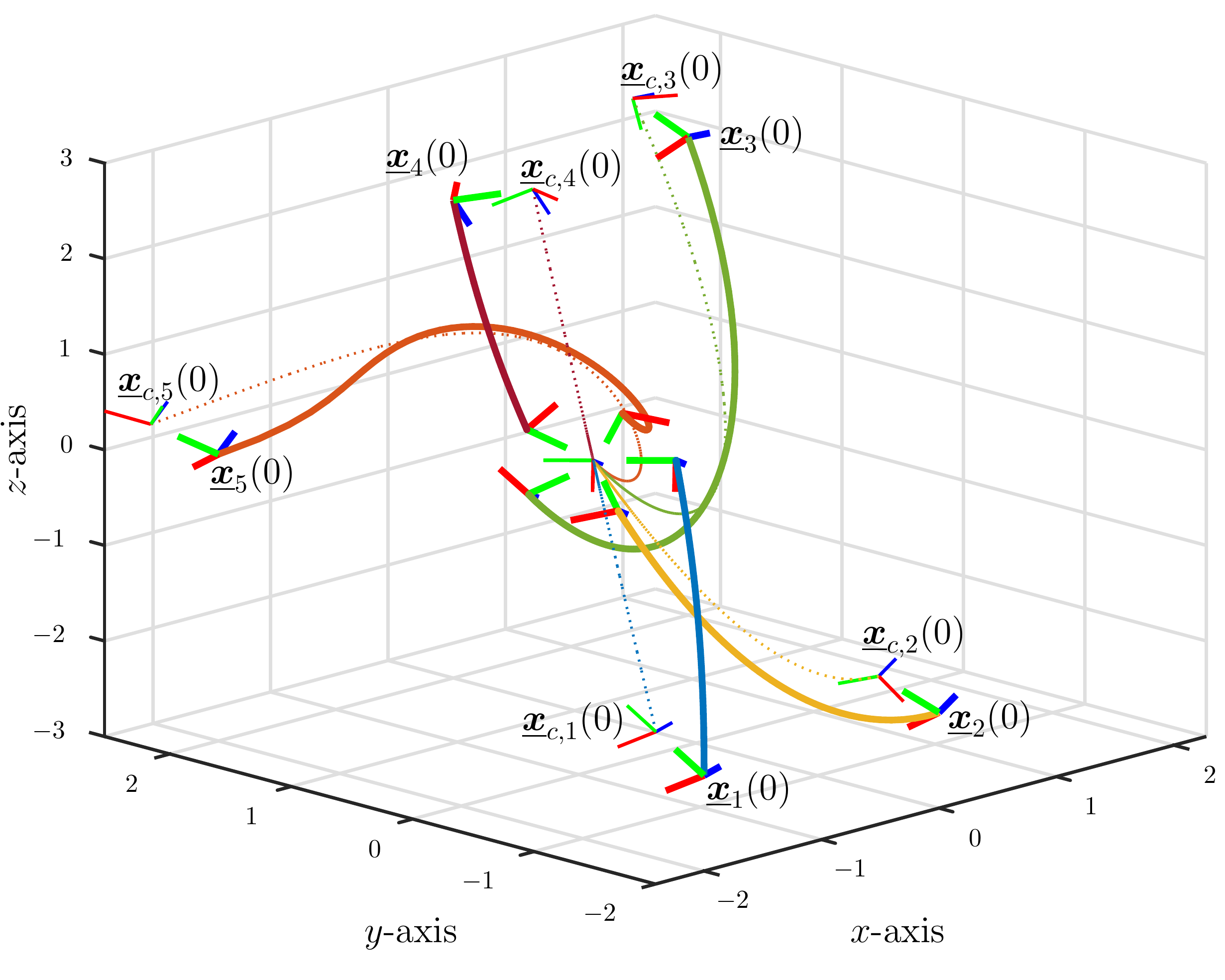}
\caption{\label{fig:simulation}Simulation for five agents in a circular formation.}
\end{figure}

\begin{figure}
\noindent \begin{centering}
\subfloat[$y_{2c,i}(t)$ for all agents.\label{fig:y2}]{\includegraphics[width=0.3\columnwidth]{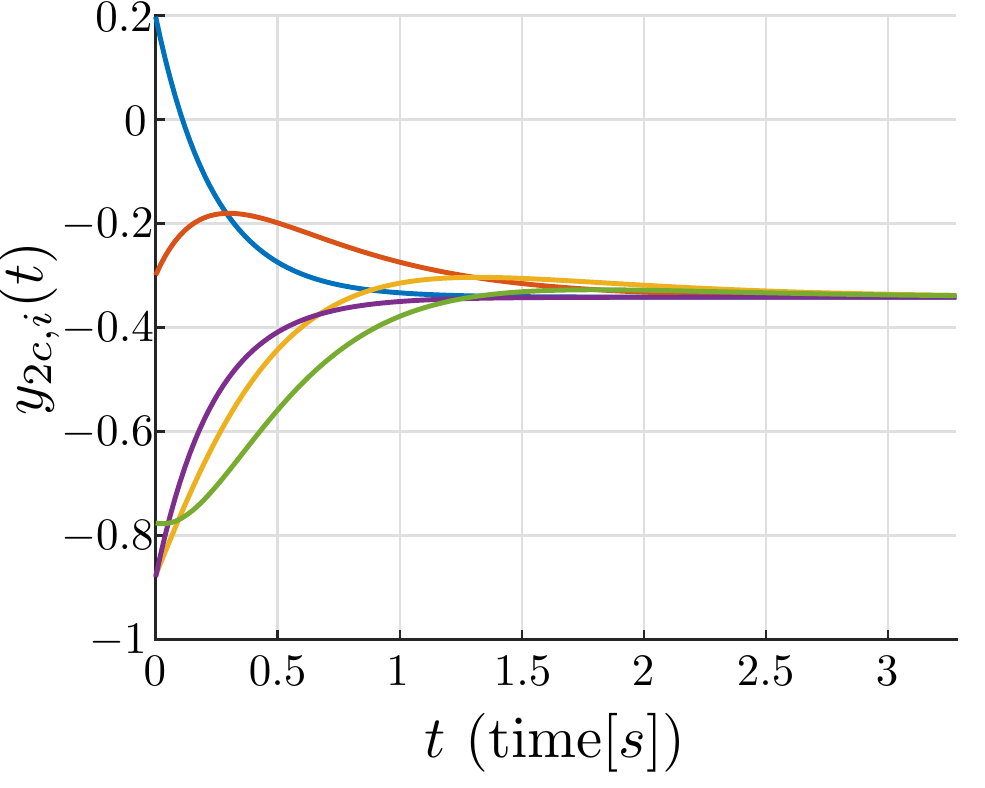}\lyxdeleted{Bruno Vilhena Adorno}{Sat Jun 15 00:05:32 2019}{ }

}\subfloat[$y_{3c,i}(t)$ for all agents.\label{fig:y3}]{\includegraphics[width=0.3\columnwidth]{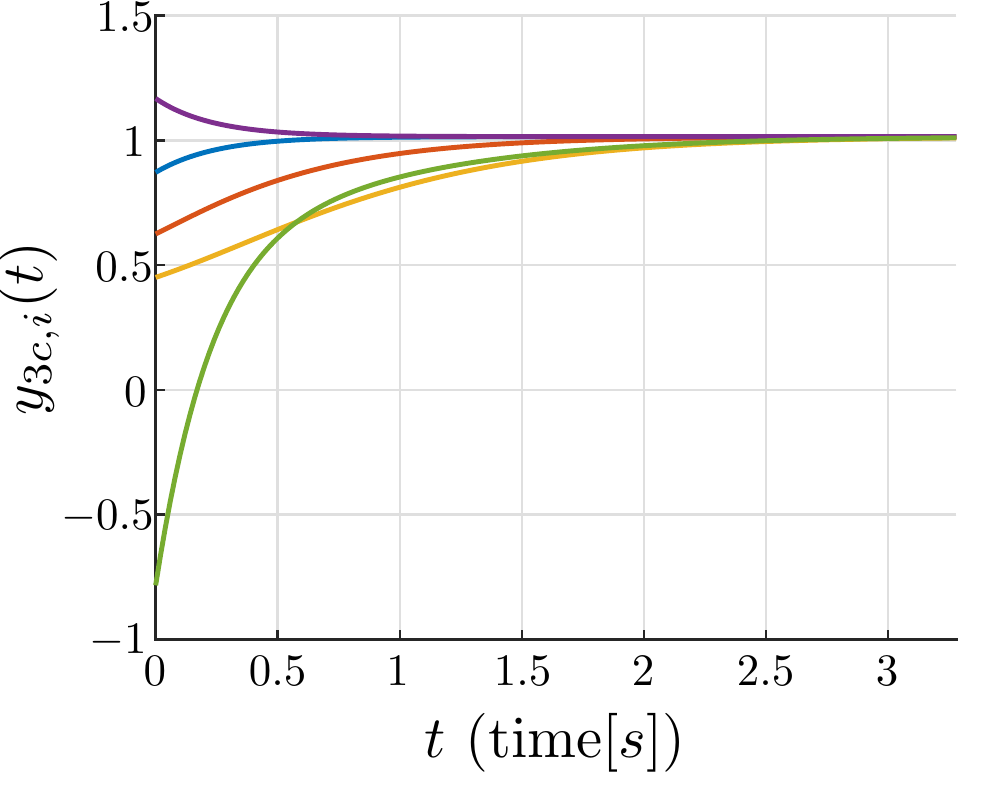}\lyxdeleted{Bruno Vilhena Adorno}{Sat Jun 15 00:05:32 2019}{ }

}\subfloat[$y_{4c,i}(t)$ for all agents.\label{fig:y4}]{\includegraphics[width=0.3\columnwidth]{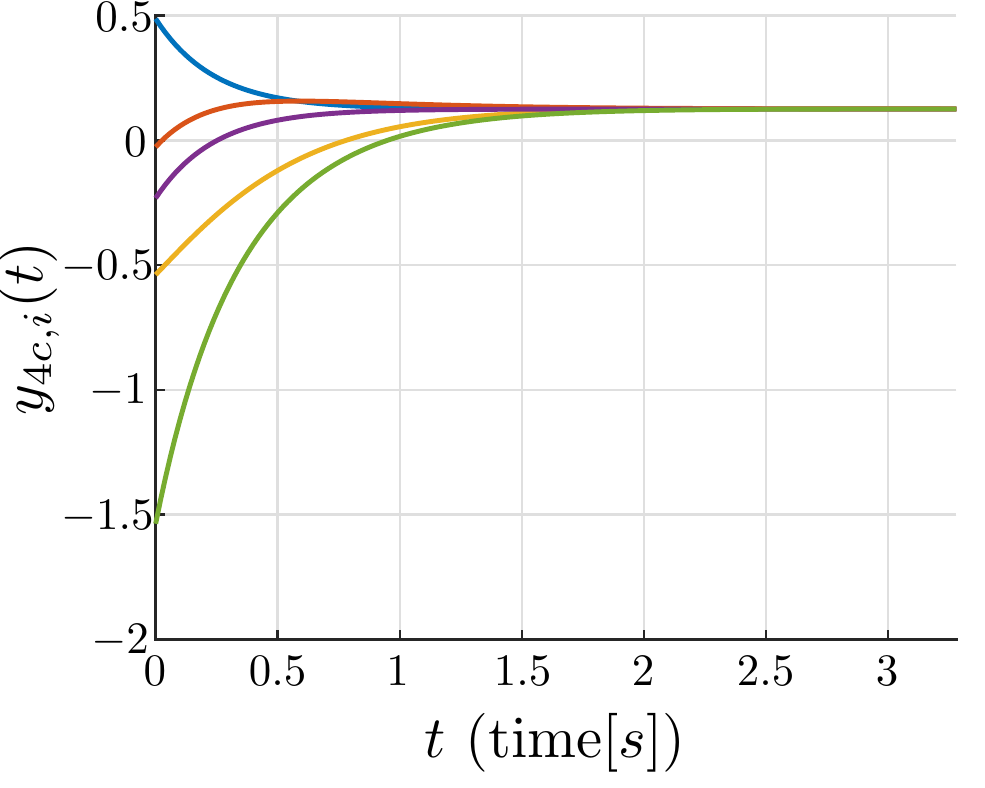}\lyxdeleted{Bruno Vilhena Adorno}{Sat Jun 15 00:05:32 2019}{ }

}
\par\end{centering}
\noindent \centering{}\subfloat[$y_{6c,i}(t)$ for all agents.\label{fig:y6}]{\includegraphics[width=0.3\columnwidth]{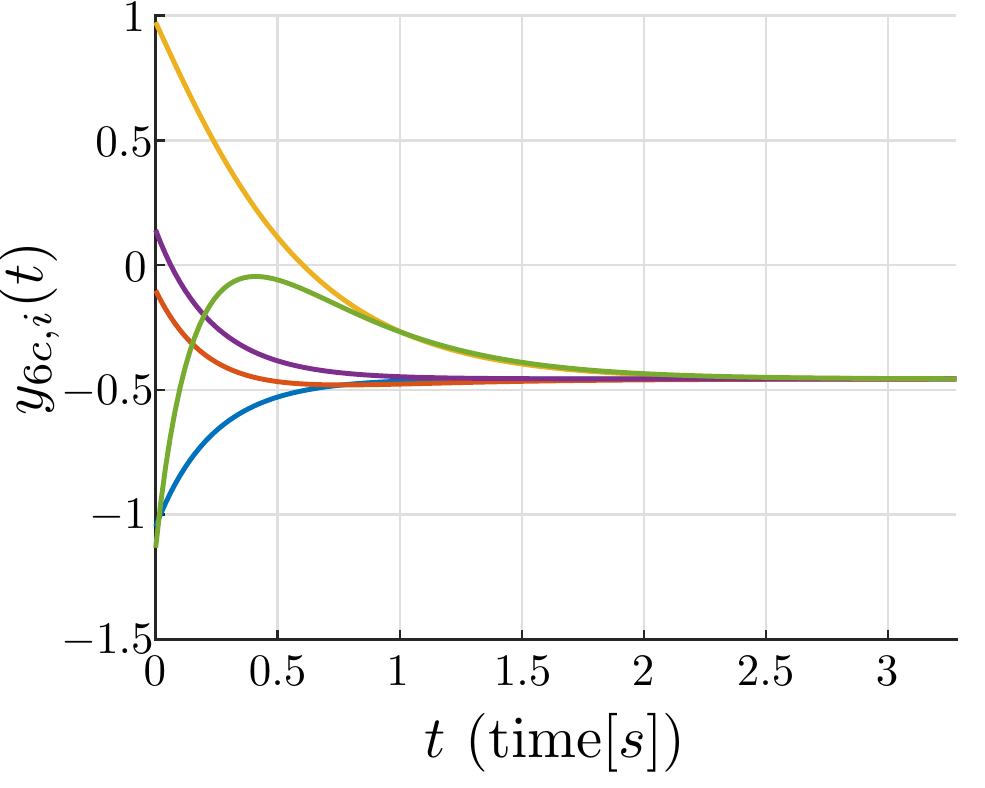}

}\subfloat[$y_{7c,i}(t)$ for all agents.\label{fig:y7}]{\includegraphics[width=0.3\columnwidth]{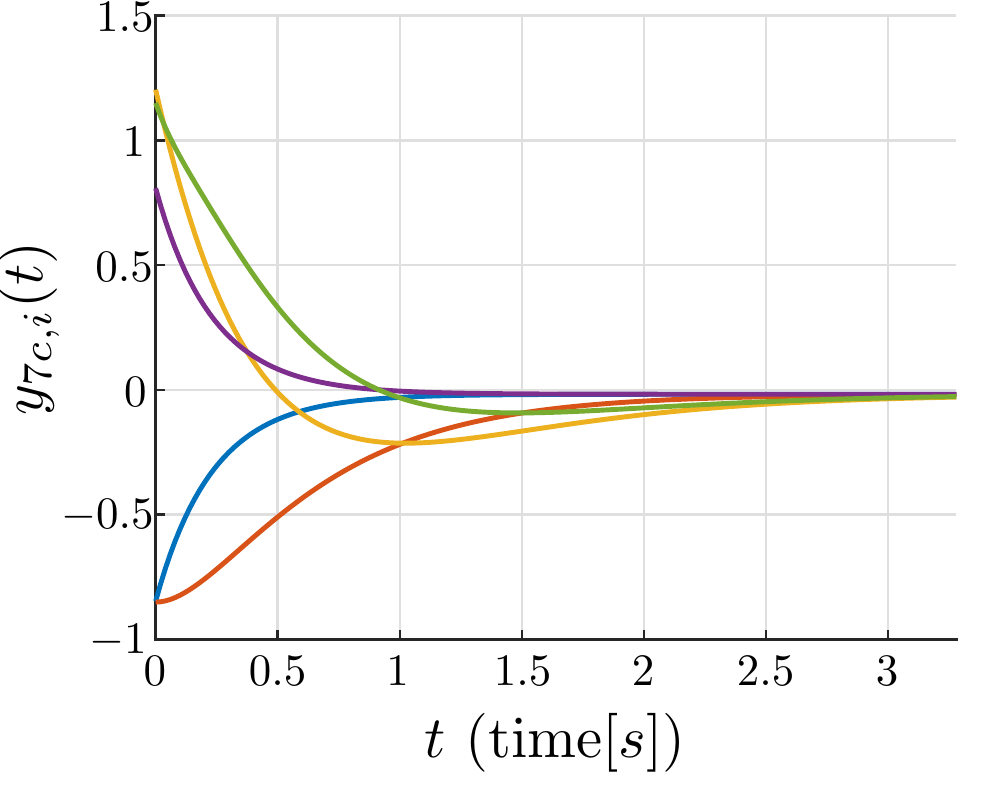}\lyxdeleted{Bruno Vilhena Adorno}{Sat Jun 15 00:05:32 2019}{ }

}\subfloat[$y_{8c,i}(t)$ for all agents.\label{fig:y8}]{\includegraphics[width=0.3\columnwidth]{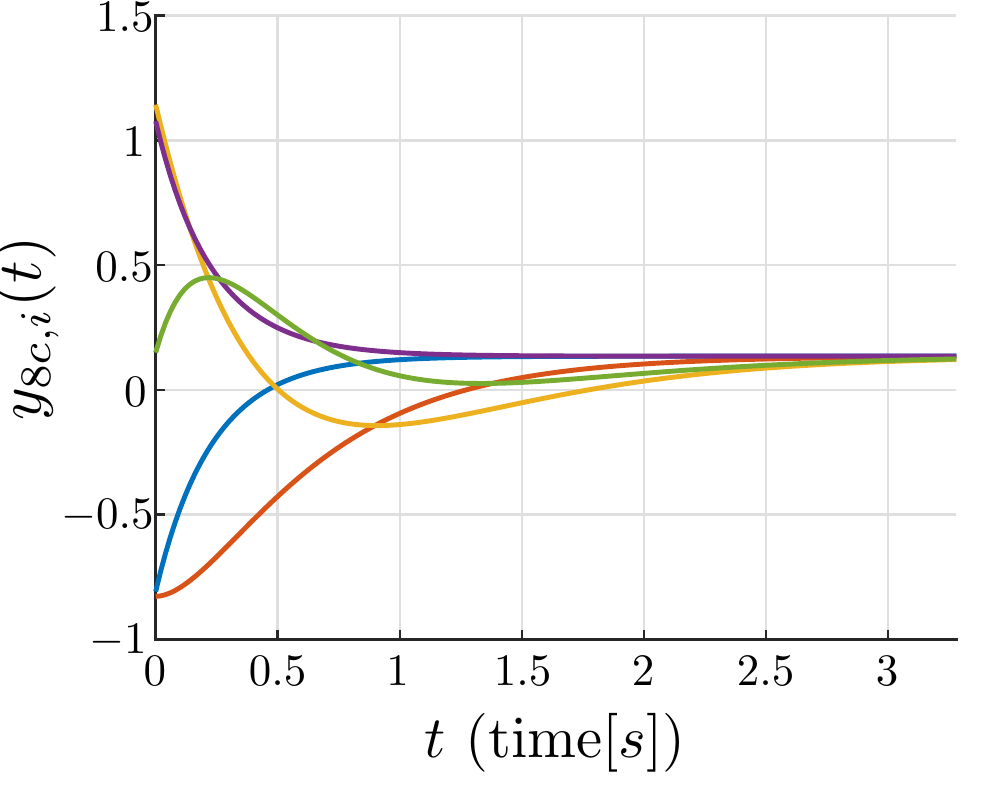}\lyxdeleted{Bruno Vilhena Adorno}{Sat Jun 15 00:05:32 2019}{ }

}\caption{\label{fig:simulation1y} Time-evolution for each coefficient of $\protect\dq y_{c,i}=y_{2c,i}\protect\imi+y_{3c,i}\protect\imj+y_{4c,i}\protect\imk+\protect\dual(y_{6c,i}\protect\imi+y_{7c,i}\protect\imj+y_{8c,i}\protect\imk)$
in the circular formation.}
\end{figure}

Finally, in order to show scalability and validate the time-varying
formation decentralized controller, a second simulation is carried
out with 100 agents. First we generate a random fixed directed network
containing a directed spanning tree, and then we randomly generate
the initial poses $\dq x_{i}\left(0\right)$, $\forall i\in\left\{ 1,\ldots,100\right\} $.
The random fixed directed network containing a directed spanning tree
is obtained according to the following procedure. First we randomly
generate a $100\times100$ matrix and set to zero all elements of
the main diagonal. The resulting matrix is defined as the adjacency
matrix $\mymatrix A$ if the corresponding Laplacian matrix has at
most one zero eigenvalue and all the others have positive real part,
because such matrix corresponds to a topology that contains a directed
spanning tree \citep[Cor. 2.5]{ren2008distributed}. If the corresponding
Laplacian matrix does not contain at most one zero eigenvalue or has
one or more eigenvalues with negative real part, the adjacency matrix
is discarded and the procedure is repeated until an appropriate matrix
is generated.

The goal is to reach a formation given by 
\begin{align}
\dq{\delta}_{i}\left(t\right) & =\quat r_{x,i}\quat r_{z,i}\dq p\left(t\right),\label{eq:time-varying-delta}
\end{align}
where $\quat r_{z,i}=\quat r_{\delta,i}$ as in \eqref{eq:delta2},
$\quat r_{x,i}=\cos\left(\phi_{\delta,i}/2\right)+\imi\sin\left(\phi_{\delta,i}/2\right)$,
with $\phi_{\delta,i}$ given by \eqref{eq:delta1}, $\dq p\left(t\right)=1+\dual0.5\left(-\imi-\imj\right)\left(2+\cos\left(8\pi t\right)\right)$
and $\dot{\dq p}=-\dual0.5\left(-\imi-\imj\right)8\pi\sin\left(8\pi t\right)$,
with $t\in[0,0.25]\unit{s}$.

\begin{figure}[th]
\noindent \begin{centering}
\subfloat[$t=0\unit{ms}$]{\noindent \begin{centering}
\includegraphics[width=0.32\textwidth]{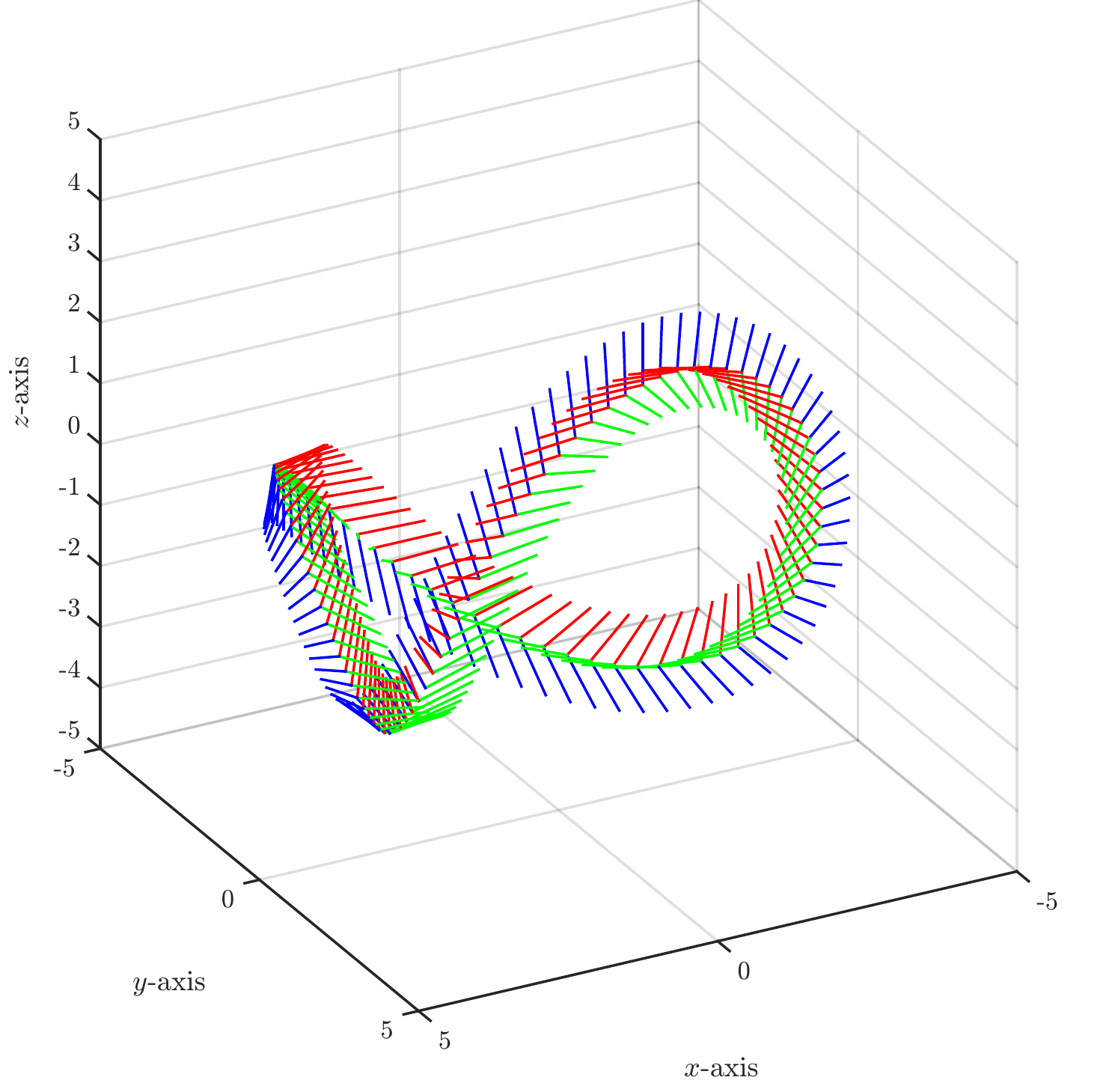}
\par\end{centering}
}\subfloat[$t=30\unit{ms}$]{\noindent \begin{centering}
\includegraphics[width=0.32\textwidth]{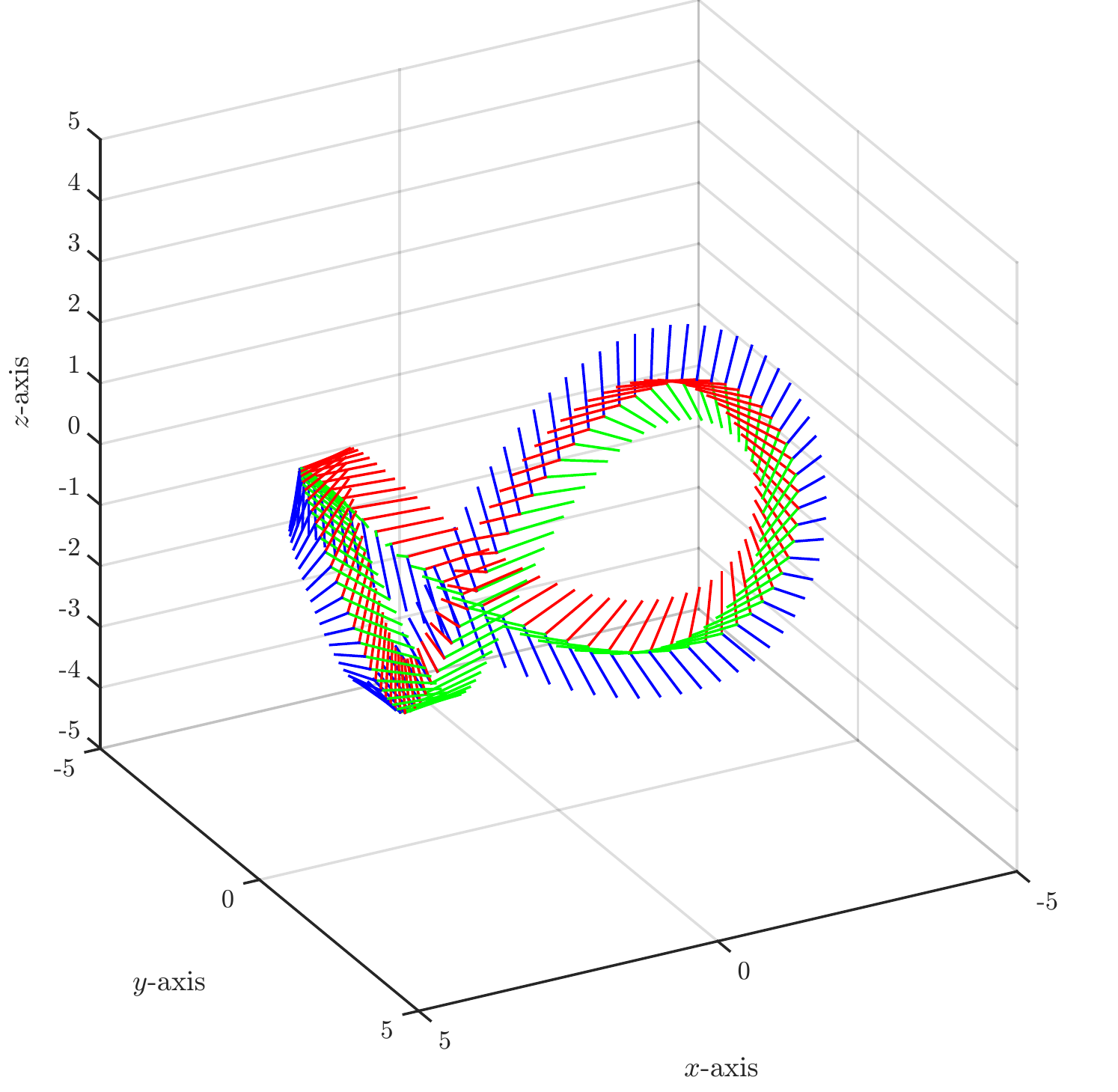}
\par\end{centering}
}\subfloat[$t=75\unit{ms}$]{\noindent \begin{centering}
\includegraphics[width=0.32\textwidth]{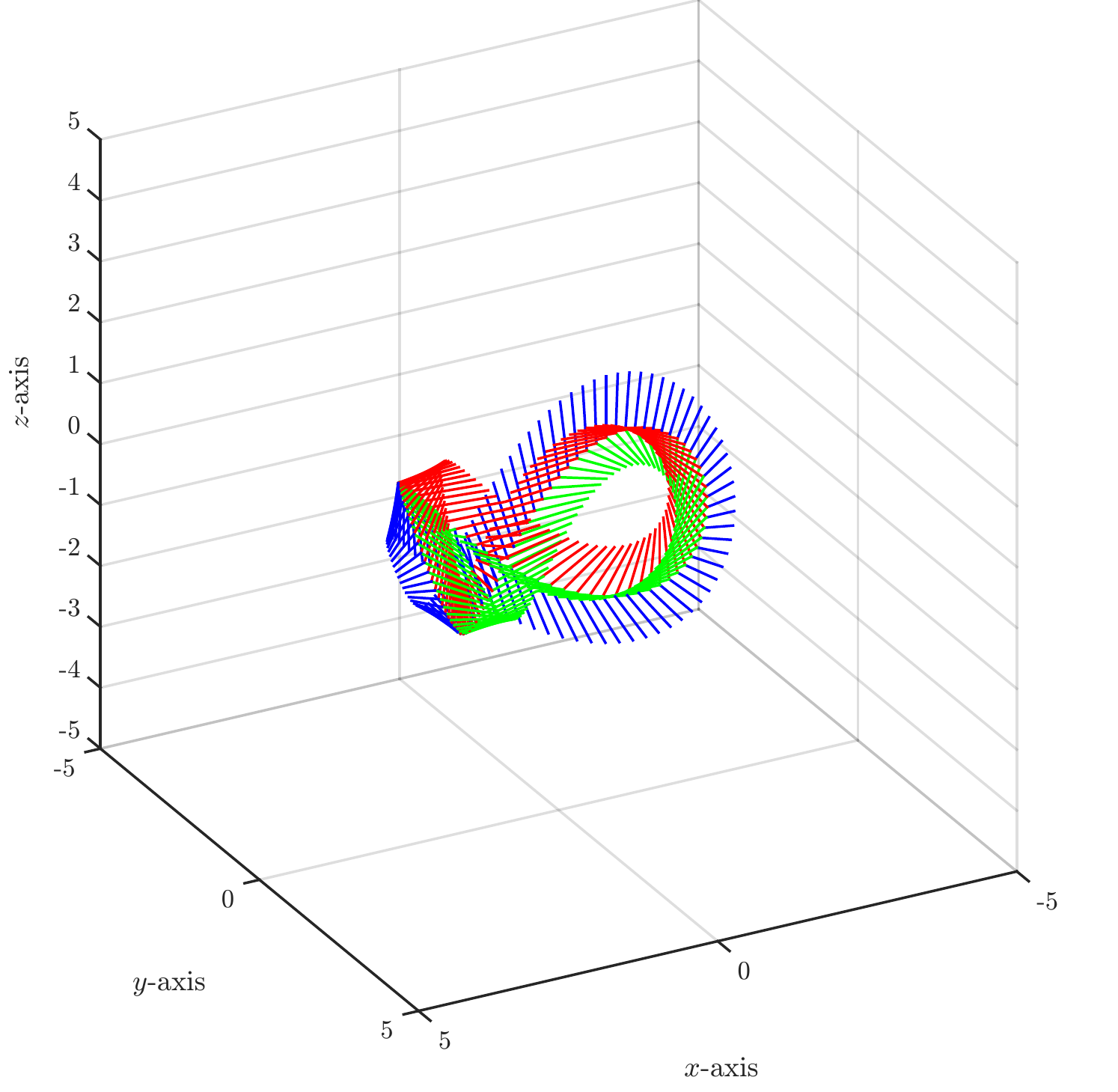}
\par\end{centering}
}
\par\end{centering}
\noindent \begin{centering}
\subfloat[$t=0\unit{ms}$]{\noindent \begin{centering}
\includegraphics[width=0.32\textwidth]{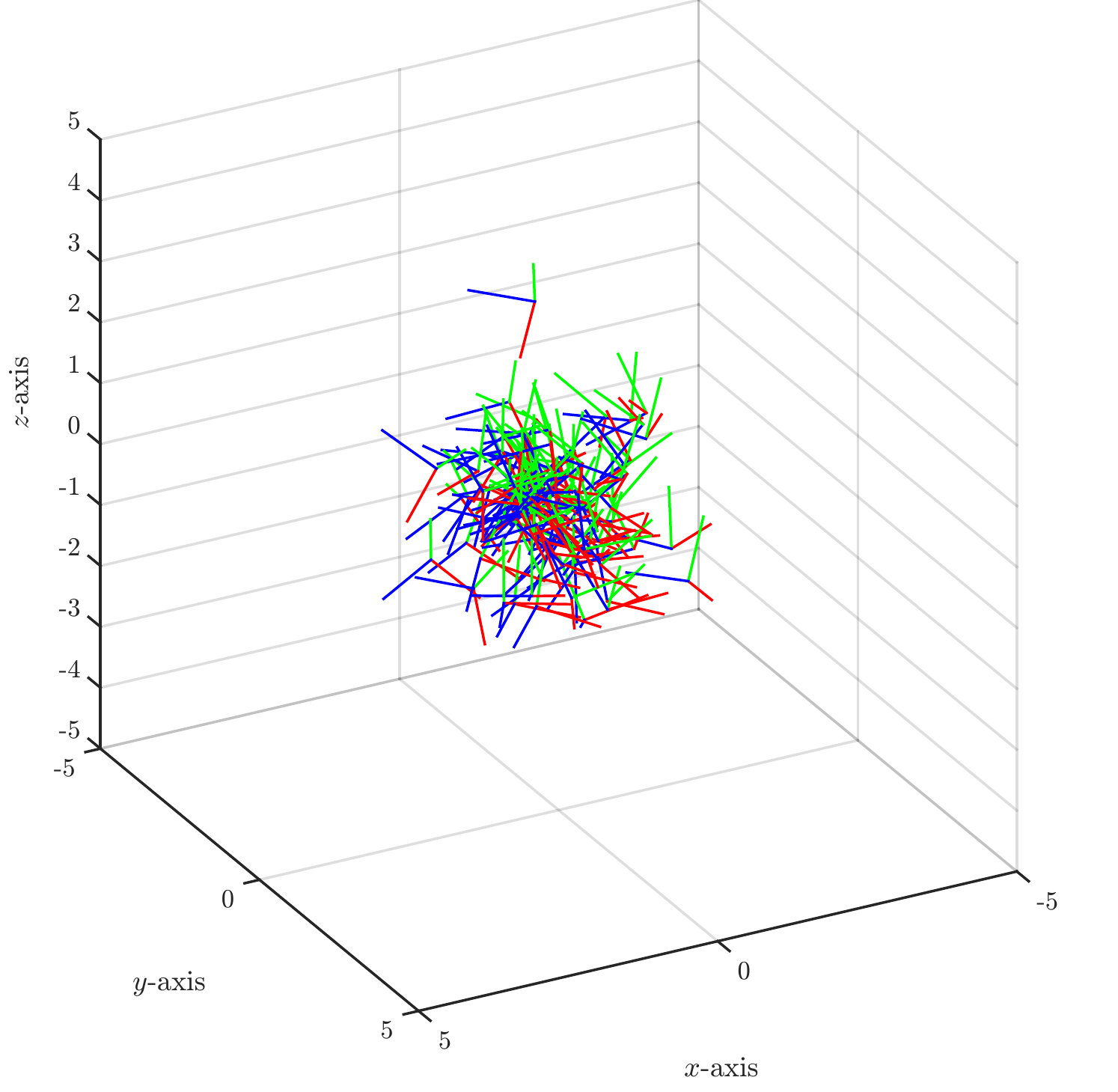}
\par\end{centering}
}\subfloat[$t=30\unit{ms}$]{\noindent \begin{centering}
\includegraphics[width=0.32\textwidth]{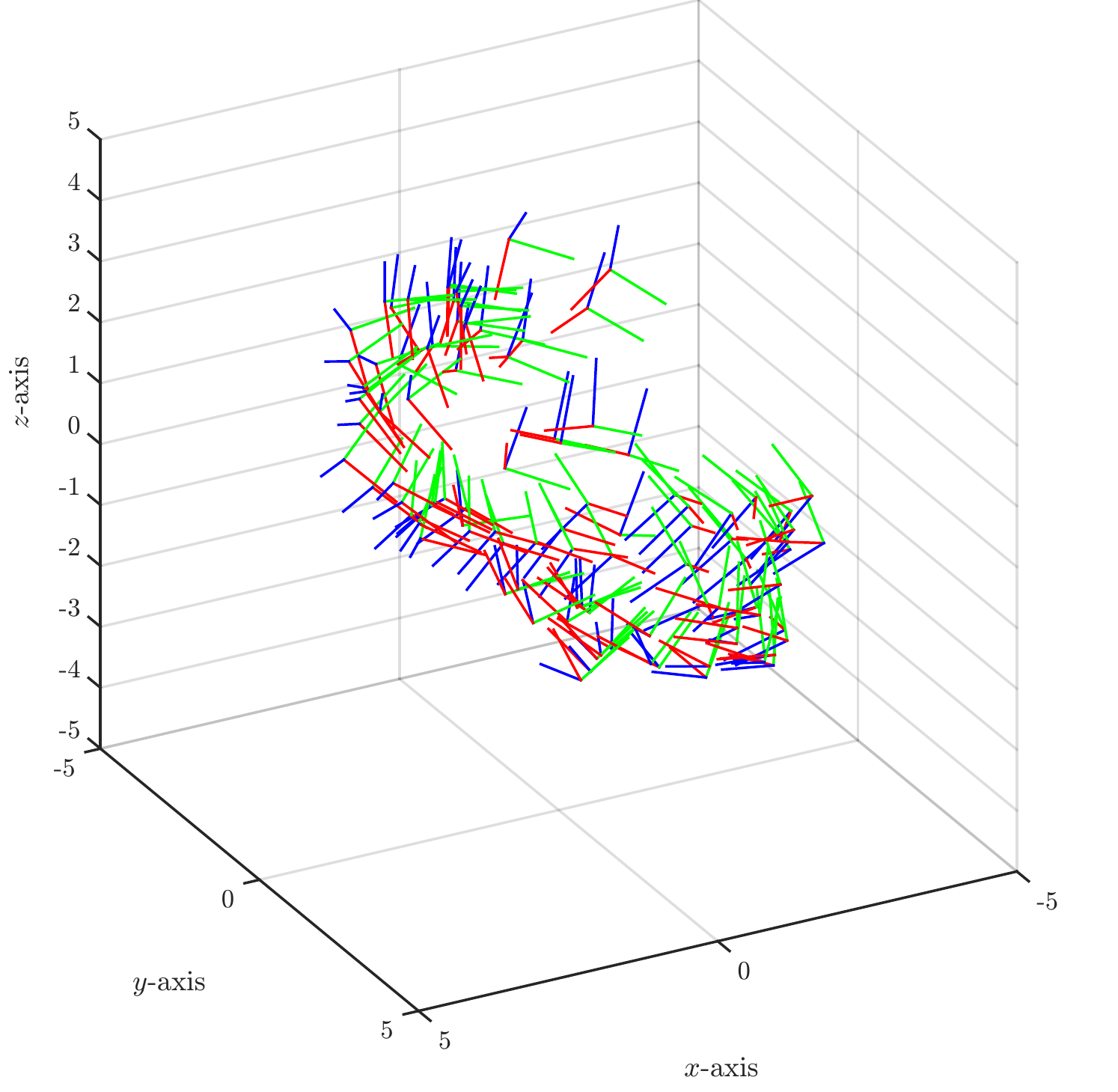}
\par\end{centering}
}\subfloat[$t=75\unit{ms}$]{\noindent \begin{centering}
\includegraphics[width=0.32\textwidth]{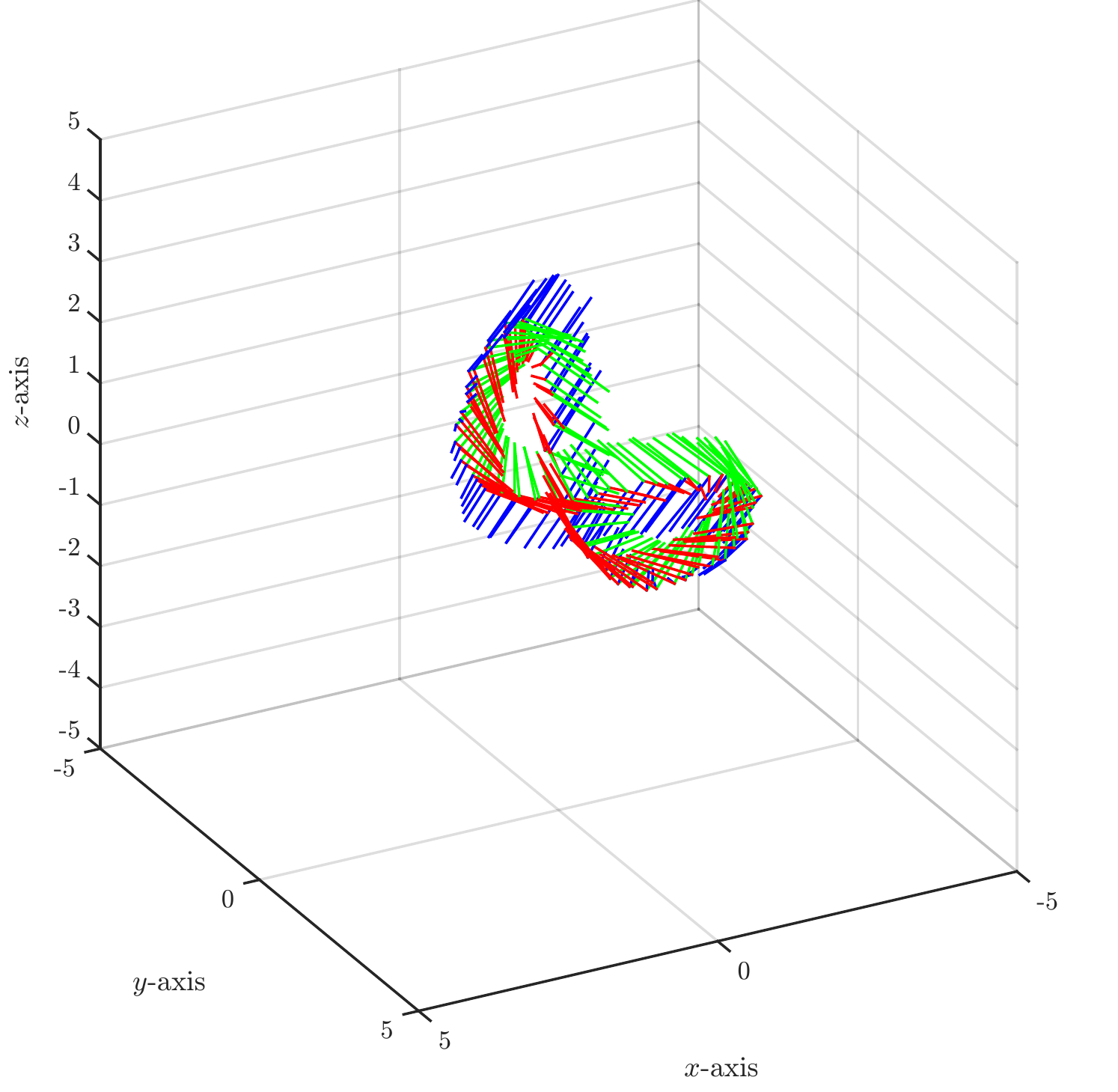}
\par\end{centering}
}
\par\end{centering}
\caption{Simulation for 100 agents in a time-varying formation with $\protect\dq{\delta}_{i}\left(t\right)$
given by \eqref{eq:time-varying-delta}. The \emph{upper} row shows
the desired formation and the \emph{lower} row shows the executed
one. From $0\unit{ms}$ to $75\unit{ms}$, the desired formation is
shrinking. When $t=75\unit{ms}$, the system has almost achieved the
desired formation.\label{fig:Simulation-for-100-shrink}}
\end{figure}

\begin{figure}[th]
\noindent \begin{centering}
\subfloat[$t=150\unit{ms}$]{\noindent \begin{centering}
\includegraphics[width=0.32\textwidth]{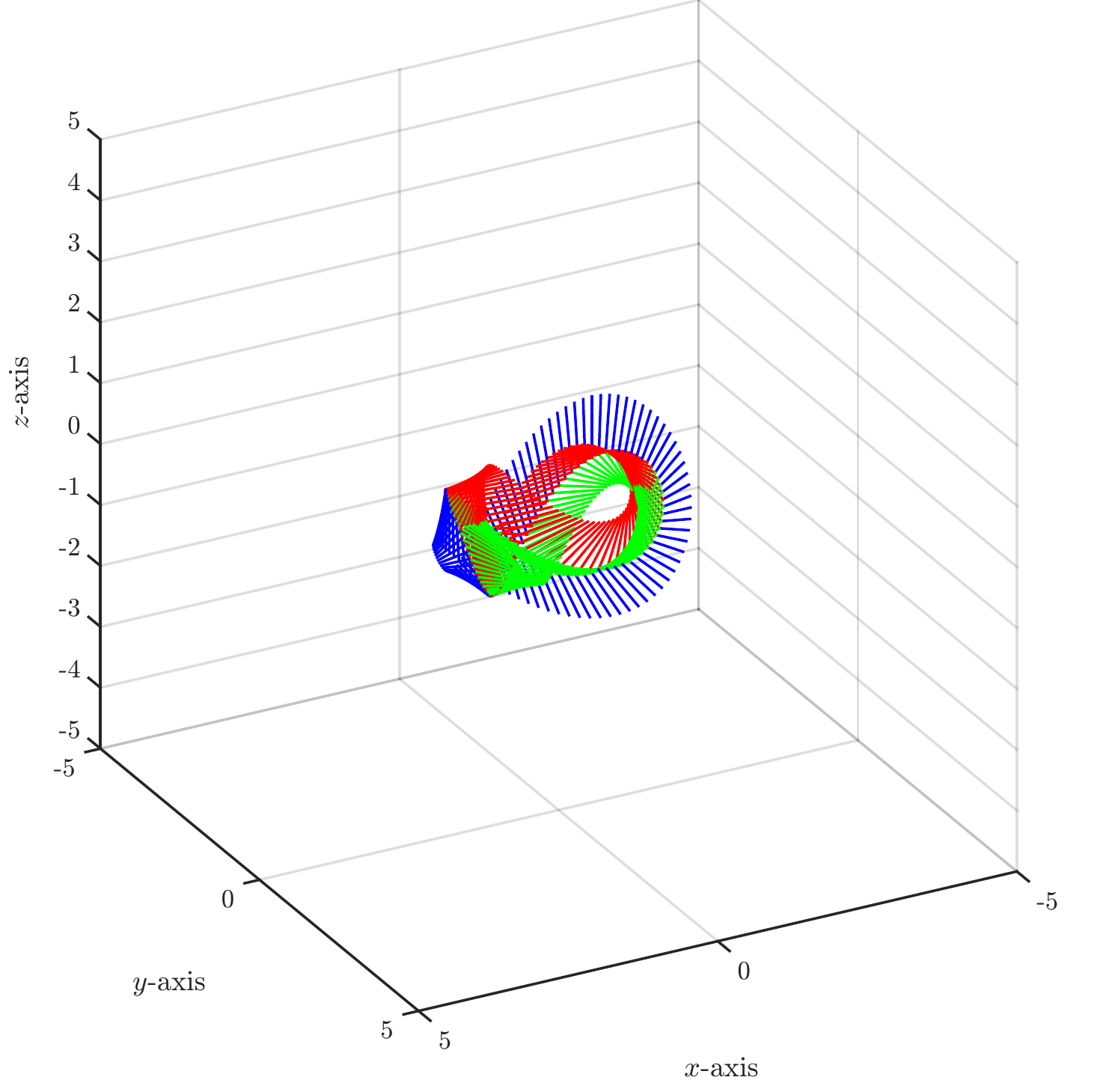}
\par\end{centering}
}\subfloat[$t=200\unit{ms}$]{\noindent \begin{centering}
\includegraphics[width=0.32\textwidth]{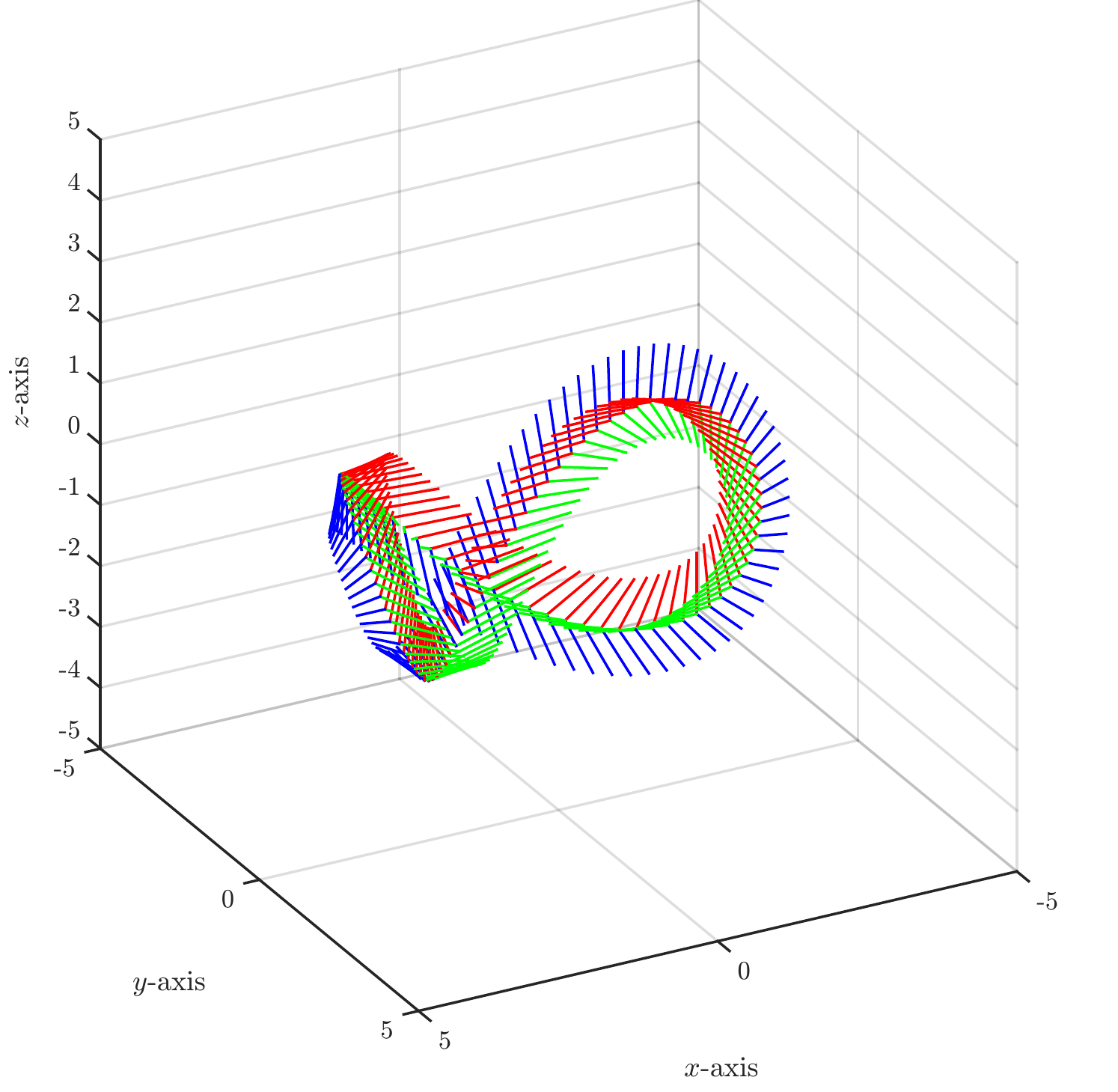}
\par\end{centering}
}\subfloat[$t=250\unit{ms}$]{\noindent \begin{centering}
\includegraphics[width=0.32\textwidth]{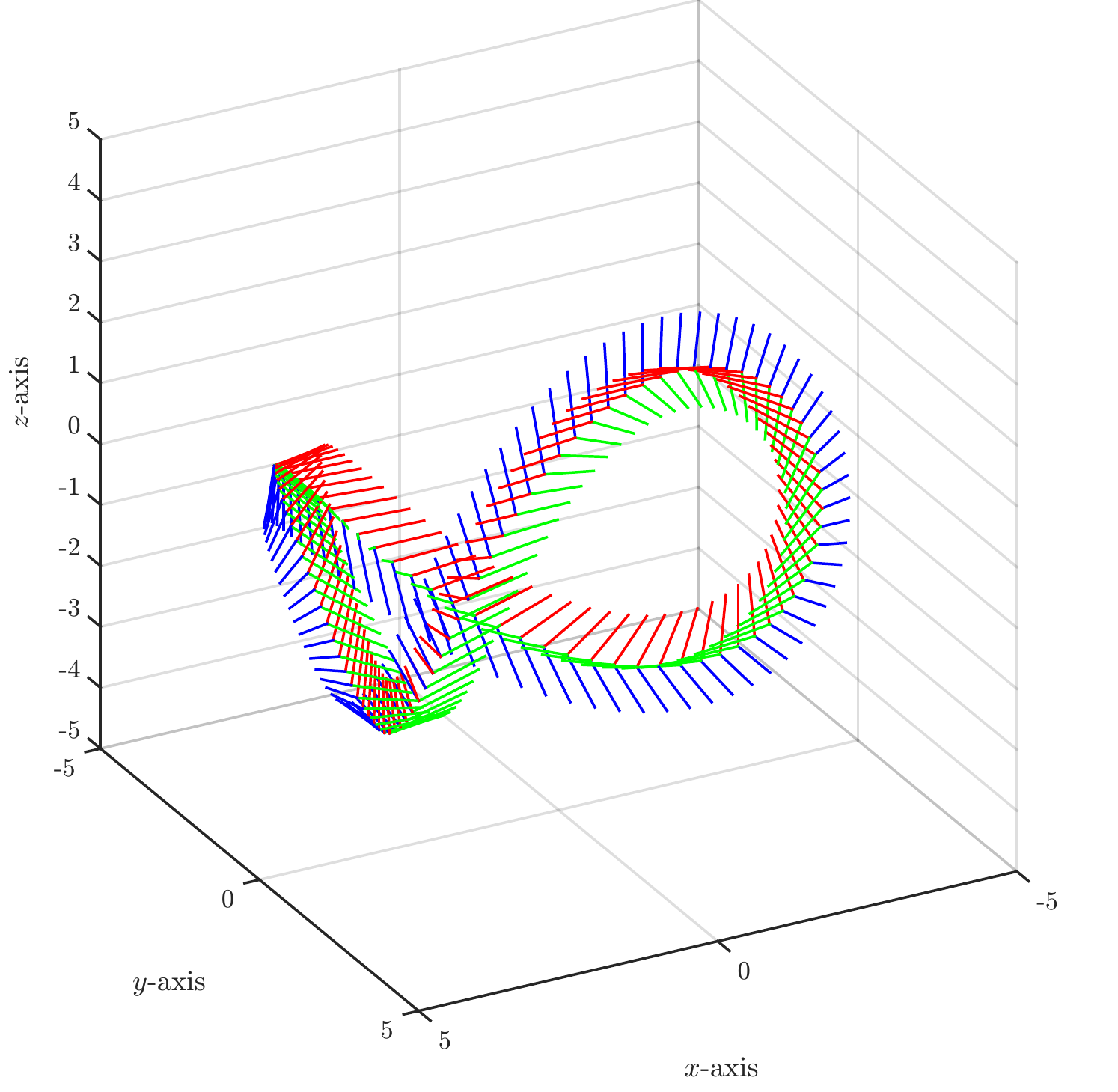}
\par\end{centering}
}
\par\end{centering}
\noindent \begin{centering}
\subfloat[$t=150\unit{ms}$]{\noindent \begin{centering}
\includegraphics[width=0.32\textwidth]{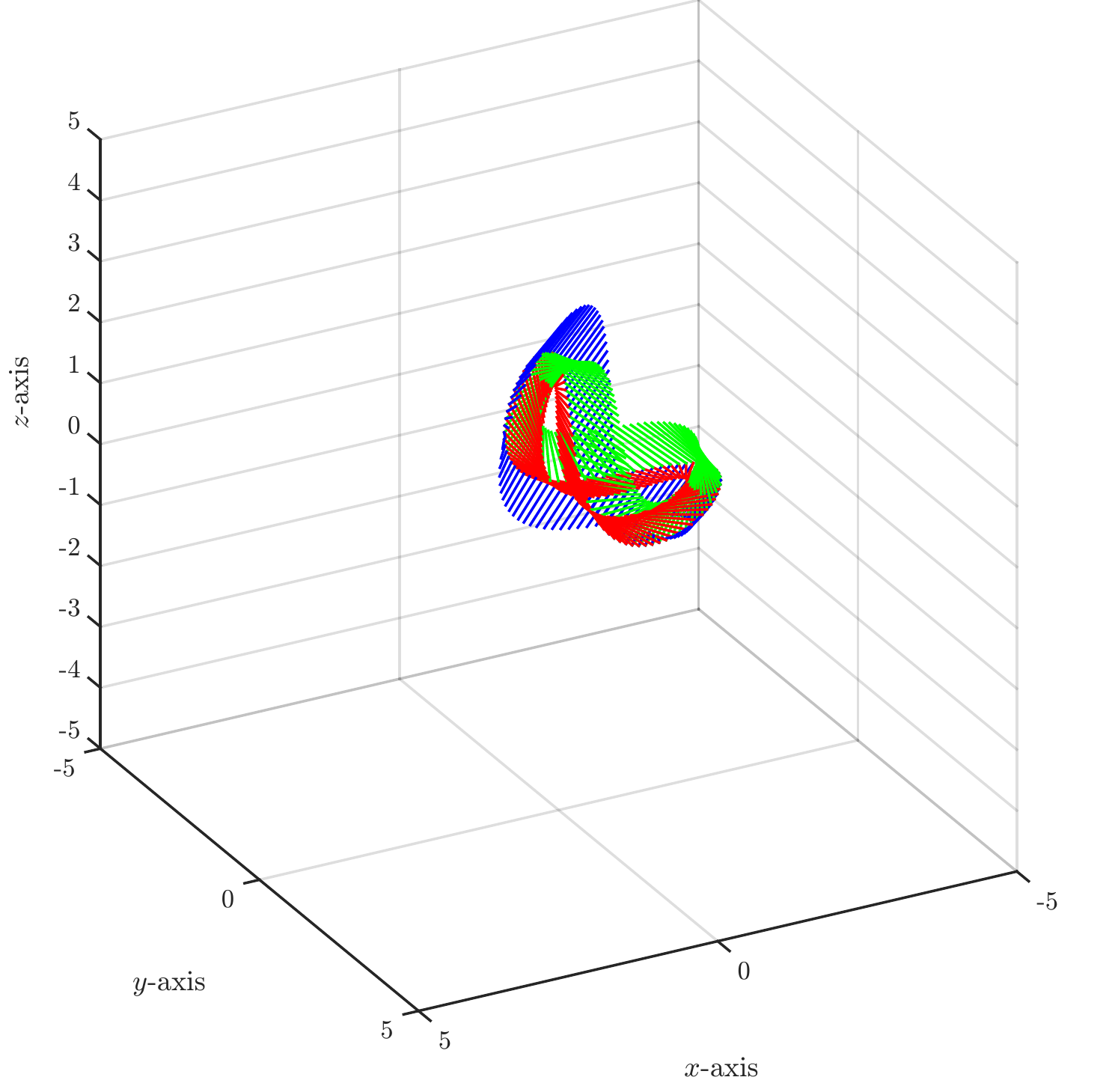}
\par\end{centering}
}\subfloat[$t=200\unit{ms}$]{\noindent \begin{centering}
\includegraphics[width=0.32\textwidth]{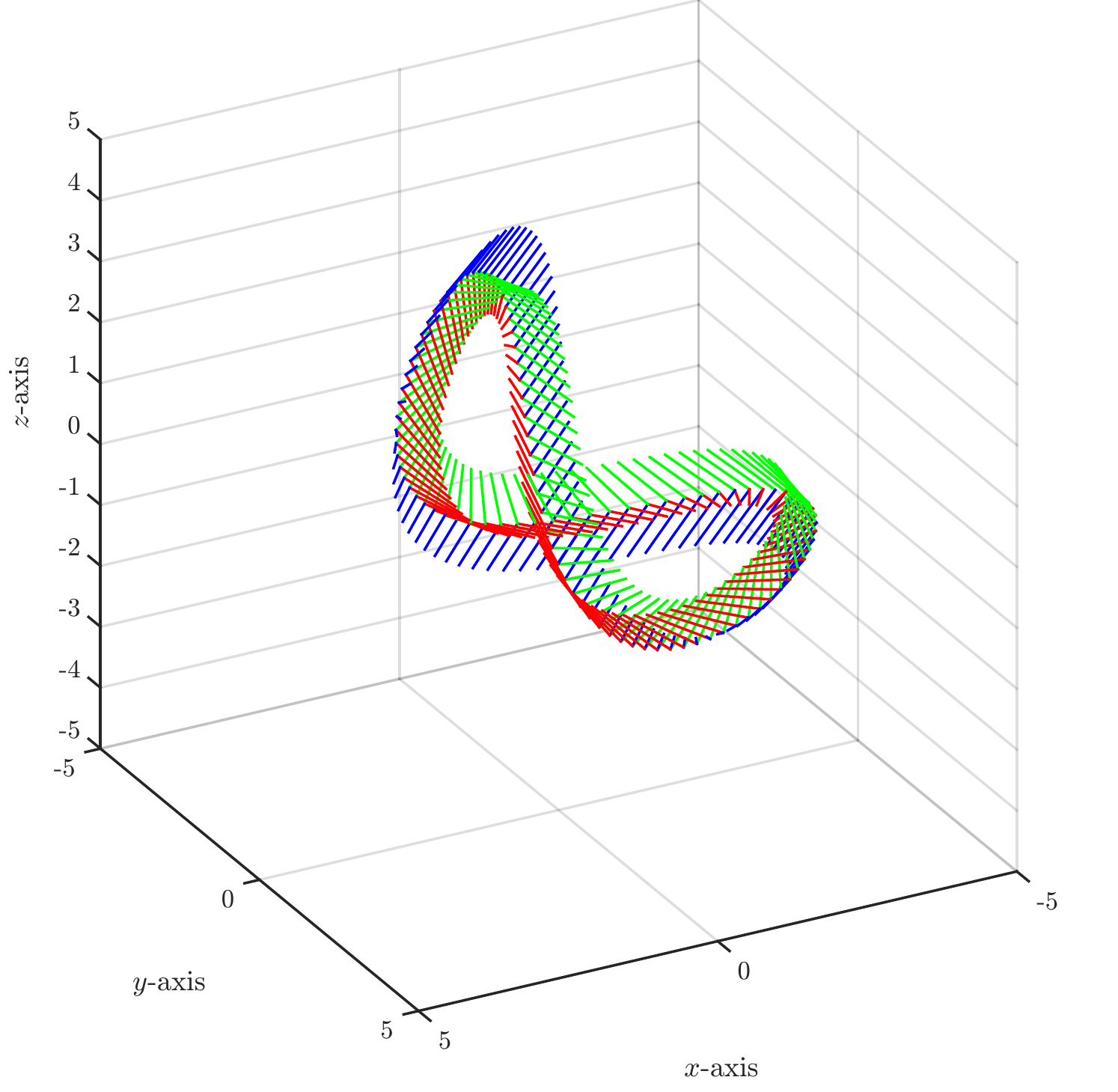}
\par\end{centering}
}\subfloat[$t=250\unit{ms}$]{\noindent \begin{centering}
\includegraphics[width=0.32\textwidth]{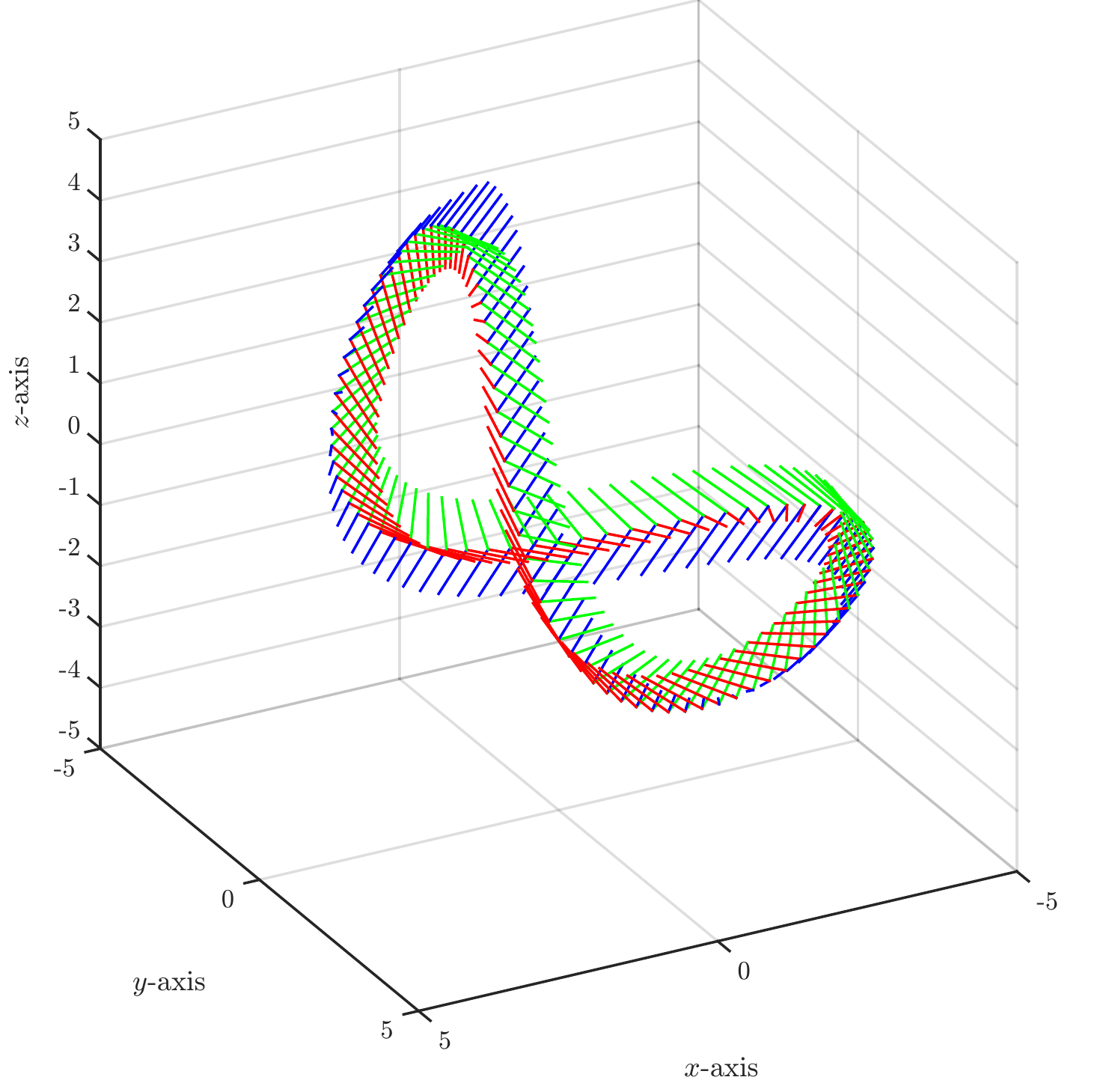}
\par\end{centering}
}
\par\end{centering}
\caption{\label{fig:Simulation-for-100-expand}Simulation for 100 agents in
a time-varying formation with $\protect\dq{\delta}_{i}\left(t\right)$
given by \eqref{eq:time-varying-delta}. The \emph{upper} row shows
the desired formation and the \emph{lower} row shows the executed
one. From $150\unit{ms}$ to $250\unit{ms}$, the desired formation
is expanding. When $t=150\unit{ms}$, the system has already achieved
the desired formation and from this point forward it tracks the time-varying
formation very closely.}
\end{figure}

The simulation is shown in Figures \ref{fig:Simulation-for-100-shrink}
and \ref{fig:Simulation-for-100-expand}. From $0\unit{ms}$ to $75\unit{ms}$,
the desired formation is shrinking, and when $t=75\unit{ms}$, the
system has almost achieved the desired formation. From $150\unit{ms}$
to $250\unit{ms}$, the desired formation is expanding, and when $t=150\unit{ms}$,
the system has already achieved the desired formation. From this point
forward it tracks the time-varying formation very closely. This behavior
can also be seen in Figure~\ref{fig:Time-evolution-time-varying-formation},
which shows the time evolution of each coefficient of the agents'
outputs. It indicates that after $100\unit{ms}$ all agents have agreed
on the desired center of formation, which implies that they track
the time-varying formation without error. Since the agents agree on
a center of formation by means of local information exchange, the
formation can happen anywhere in space, as both Figures \ref{fig:Simulation-for-100-shrink}
and \ref{fig:Simulation-for-100-expand} show.

\begin{figure}[th]
\noindent \begin{centering}
\includegraphics[width=1\textwidth]{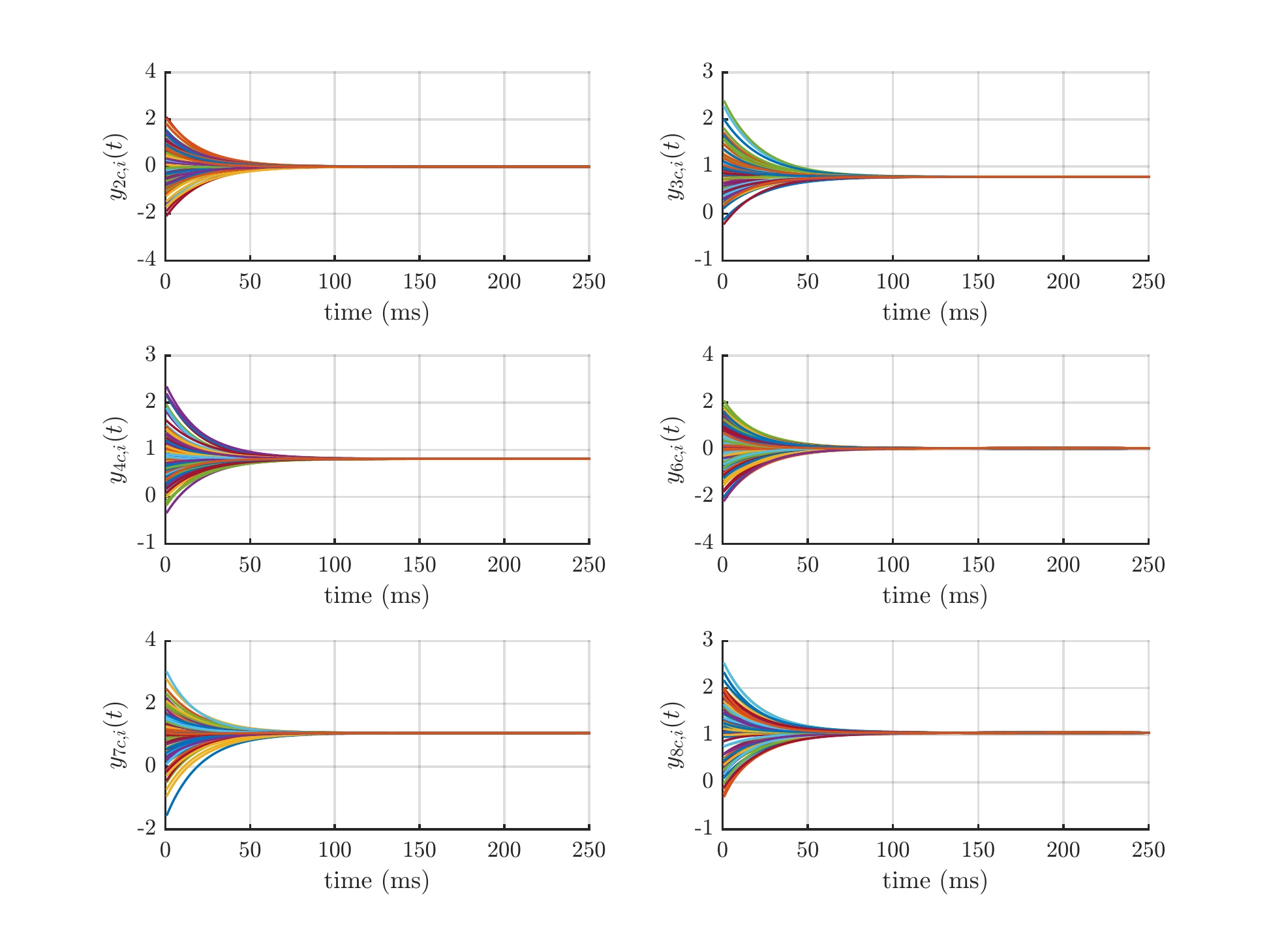}
\par\end{centering}
\caption{Time-evolution for each coefficient of $\protect\dq y_{c,i}=y_{2c,i}\protect\imi+y_{3c,i}\protect\imj+y_{4c,i}\protect\imk+\protect\dual(y_{6c,i}\protect\imi+y_{7c,i}\protect\imj+y_{8c,i}\protect\imk)$
when performing the time-varying formation described by \eqref{eq:time-varying-delta}
and shown in Figures \ref{fig:Simulation-for-100-shrink} and \ref{fig:Simulation-for-100-expand}.\label{fig:Time-evolution-time-varying-formation}}
\end{figure}

\subsection{Experiment with two holonomic mobile manipulators\label{sub:Experiment-with-two}}

An experimental evaluation is important when proposing new methods
that are aimed at being implemented in real multi-robot systems because
several real world phenomena are usually disregarded when developing
the theory or even in numerical simulations. Some important real issues
are actuator saturation, uncertain pose measurements provided by the
real sensors, unmodeled dynamics, sampling and quantization errors
associated with the discrete implementation, packet loss and time
delay related to the real communication infrastructure. Therefore,
in this section we present an experiment with actual robots.\footnote{See accompanying video.}

It is considered the multi-agent system composed of two mobile manipulators
with holonomic base, namely KUKA youBots \citep{kukas}. These robots
are modeled using the whole-body kinematics modeling presented in
\ref{sec:wholebody}. Each robot is equipped with an onboard Mini-ITX
computer, with a processor Intel AtomTM Dual Core D510 ($1$M Cache,
$2\times1.66$ GHz), 2GB single-channel DDR2 667MHz memory, 32GB SSD
drive, and wireless connection by means of a usb-connected Vonets
Wireless Wifi Vap11g card. The experiments were performed at CSAIL,
MIT, in a laboratory equipped with a Vicon motion capture system that
provides, via wireless communication, the local pose for each robot
at 50Hz. The control algorithm was implemented using the Robot Operating
System (ROS) and the C++ API of DQ Robotics. ROS is a meta-operating
system that provides a structured communications layer fundamentally
based on: nodes, which contain the processes performing the computation
of robotics algorithms; messages, which are a strictly typed data
structure used by nodes to communicate with other nodes; and topics,
which are the communication channels used by publisher nodes to send
messages and by subscriber nodes to receive messages \citep{quigley2009ros}.
This framework makes it easier the task of implementing algorithms
in real robotic platforms as it provides a high level hardware abstraction
and a set of libraries, drivers, and tools to help the developer.

We have elaborated a collaborative manipulation scenario in which
the multi-agent system is composed of the two mobile manipulators
and a box to be transported inside the workspace. The formation task
is divided in two subtasks. The first one consists of a pre-grasping
formation, where the robots gather around a box, which is represented
by a static virtual leader, which corresponds to \textbf{Agent 3}
in Figure~\ref{fig:network2}. In the second subtask, the robots
grasp the box and move it around the workspace. In this case, the
agents have to follow a dynamic virtual leader, as they have to move
the box. In both subtasks, the control input for each mobile manipulator
is given by \eqref{eq:low level input}.

\begin{figure}[tbh]
\centering \begin{tikzpicture}[scale=1,auto=left,every node/.style=circle]
	\node[circle,draw,scale=.8] (n3) at (2,2)  {3};
	\node[circle,draw,scale=.8] (n1) at (1,0)  {1};
	\node[circle,draw,scale=.8] (n2) at (3,0)  {2};
	
	\pgfsetarrows{latex-latex} \foreach \from/\to in {n1/n2}
	\draw[line width=1pt] (\from) -- (\to);
	\pgfsetarrows{-latex} \foreach \from/\to in {n3/n1}
	\draw[line width=1pt] (\from) -- (\to);                 
	\end{tikzpicture}

\caption{Network topology for the experiment with two mobile manipulators.
Nodes 1 and 2 represent each robot, respectively, and node 3 represents
the virtual agent (i.e., the box).\label{fig:network2}}
\end{figure}
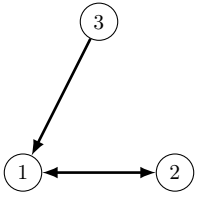

The two robots are able to send information to each other and the
box acts as a third virtual leader agent providing an output reference
related to the desired center of formation. This leader is an agent
that provides information without listening to other agents and without
executing the consensus protocol to update the output reference. The
whole system is modeled by the network topology shown in Figure~\ref{fig:network2},
where node 3 is the virtual agent used to generate the reference for
the desired formation, and nodes 1 and 2 are the mobile manipulators.
By using that topology, \textbf{Agent 3} provides the reference about
the desired center of formation only to \textbf{Agent 1}.

We use a Multi-Master ROS architecture \citep{juan2015multi} to implement
a distributed architecture. This is shown in Figure~\ref{fig:ros-architecture},
where the gray circles refer to the nodes running on each independent
agent and the square white boxes are the shared topics, which are
the communication channels in the ROS architecture. In one fixed computer,
which is responsible for the localization system, the poses of the
agents' bases, namely \texttt{pose\_base\_1} and \texttt{pose\_base\_2},
are provided by the Vicon motion capture system and made available
through ROS topics that any agent on the system can have access. Furthermore,
this same computer is responsible for the role of the virtual \textbf{Agent
3} (the box), providing information about the center of formation,
\texttt{output\_pose\_3}, as well as providing the information for
every agent about their relative pose $\dq{\delta}_{i}$ with respect
to the center of formation, namely \texttt{relative\_pose\_1} and
\texttt{relative\_pose\_2}. Separately, each agent runs its own ROS
master and shares topics with the agents and the fixed computer using
the Multi-Master ROS architecture. Each agent is able to access its
own local information regarding its end-effector pose and also its
formation parameter $\dq{\delta}_{i}$, which is provided by the fixed
computer. Furthemore, the agents exchange data with their neighbors---more
specifically \texttt{output\_pose\_1}, \texttt{output\_pose\_2} and
\texttt{output\_pose\_3}---according to the graph topology shown
in Figures \ref{fig:network2} and \ref{fig:ros-architecture}.

\begin{figure}
\noindent \centering{}\includegraphics[width=0.6\columnwidth]{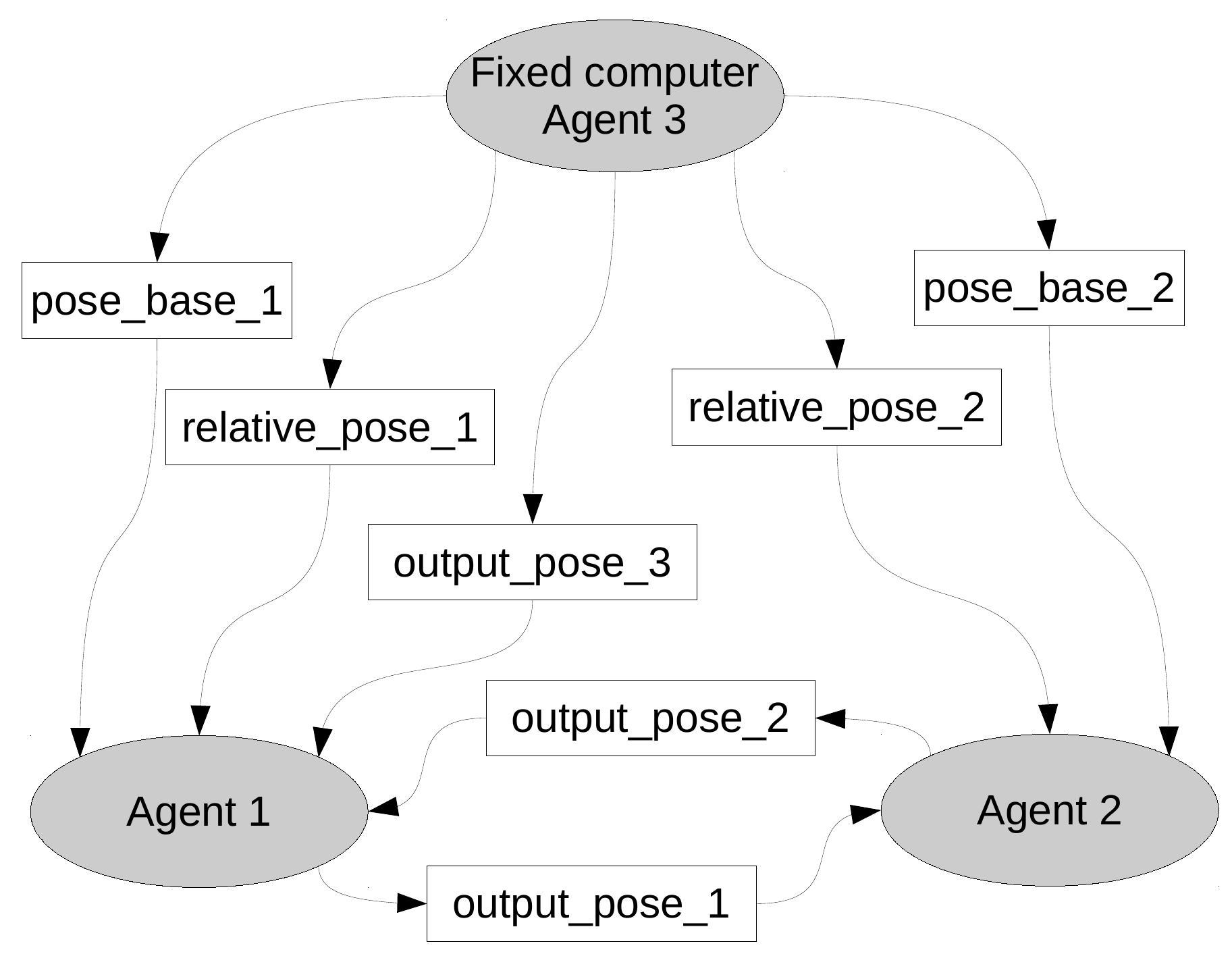}
\caption{\label{fig:ros-architecture} Multi-master ROS architecture with shared
topics.}
\end{figure}

\subsubsection{Pre-grasping formation}

The first goal is to achieve formation around a box, whose location
is informed by the state of agent $3$. For this first task, the relative
pose $\dq{\delta}_{i}$ of each agent (i.e., the pose of each end-effector
with respect to the center of formation) is defined such that the
end-effectors of agents $1$ and $2$ should point to the center of
formation at a distance of 0.30~m in the $x$ axis in opposite directions;
that is, 
\begin{align}
\dq{\delta}_{1}=1-\dual0.15\imi\label{eq:kukad1}
\end{align}
and 
\begin{align}
\dq{\delta}_{2}=\imk\left(1-\dual0.15\imi\right).\label{eq:kukad2}
\end{align}

The initial configuration of the experiment is shown in Figure~\ref{fig:expkuka1s1},
which shows the two KUKA YouBots. Agent $1$ corresponds to the robot
in the left, agent $2$ corresponds to the robot in the right, and
the virtual agent $3$ corresponds to the box. The Laplacian matrix
is thus given by 
\begin{align}
\mymatrix L & =\begin{bmatrix}1 & -0.5 & -0.5\\
-0.5 & 0.5 & 0\\
0 & 0 & 0
\end{bmatrix},
\end{align}
where the weights of all edges were chosen as 0.5 after a process
of trial and error, throughout several executions, in order to achieve
satisfactory convergence rate.

During the execution of the experiment, as shown in Figures \ref{fig:expkuka1s2},
\ref{fig:expkuka1s3}, and finally Figure~\ref{fig:expkuka1s4},
the agents are able to achieve formation around the box with the desired
poses given by $\dq{\delta}_{1}$ and $\dq{\delta}_{2}$, relative
to the center of formation, which is located at the center of the
box.

\begin{figure}[t]
\noindent \begin{centering}
\subfloat[$t=0$s.\label{fig:expkuka1s1}]{\noindent \begin{centering}
\includegraphics[viewport=80bp 85.5bp 770bp 630bp,clip,width=0.4\columnwidth]{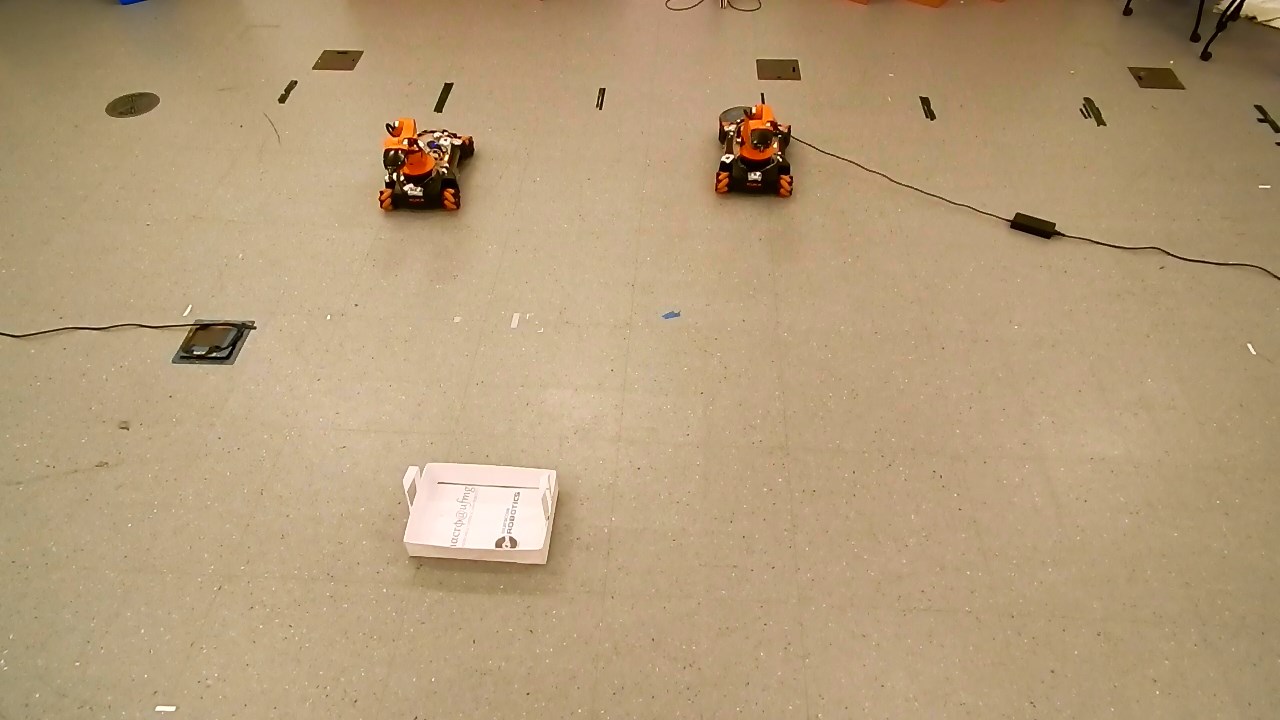}\lyxdeleted{Bruno Vilhena Adorno}{Sat Jun 15 00:05:32 2019}{ }
\par\end{centering}
}\subfloat[$t=1$s.\label{fig:expkuka1s2}]{\noindent \begin{centering}
\includegraphics[viewport=80bp 85.5bp 770bp 630bp,clip,width=0.4\columnwidth]{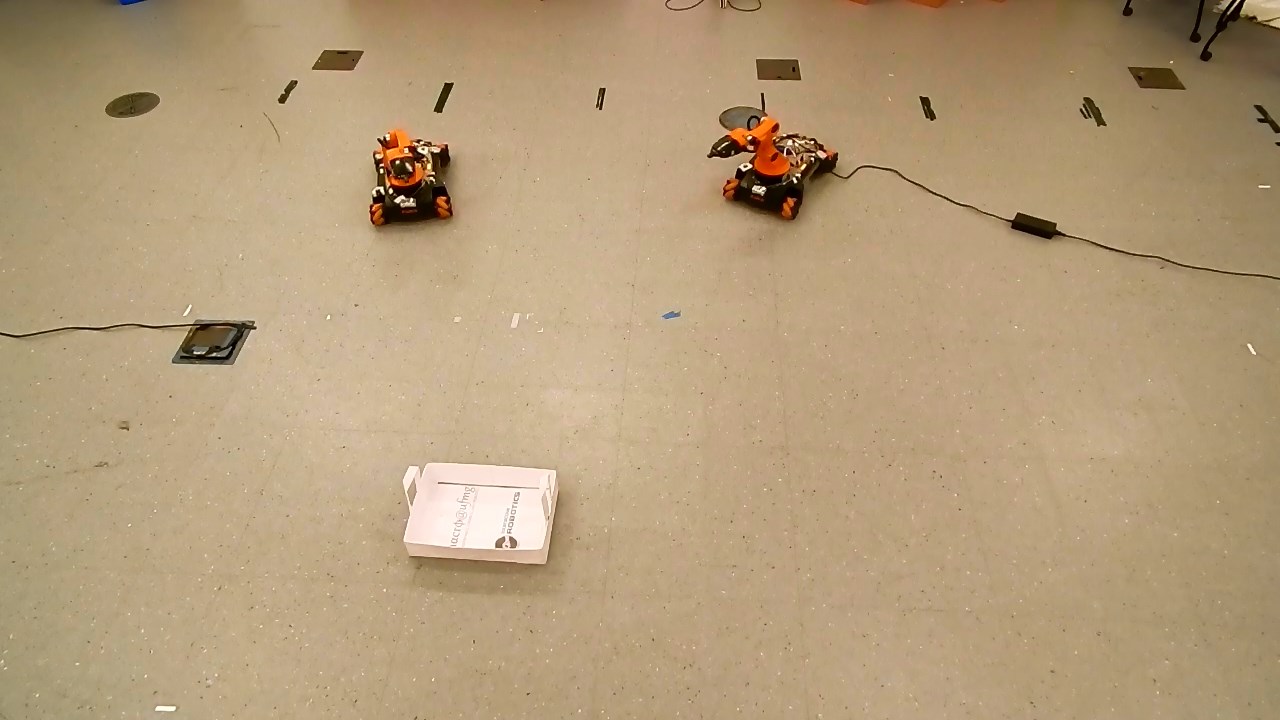}\lyxdeleted{Bruno Vilhena Adorno}{Sat Jun 15 00:05:32 2019}{ }
\par\end{centering}
}
\par\end{centering}
\vspace{-20bp}

\noindent \centering{}\subfloat[$t=3$s.\label{fig:expkuka1s3}]{\noindent \begin{centering}
\includegraphics[viewport=80bp 85.5bp 770bp 630bp,clip,width=0.4\columnwidth]{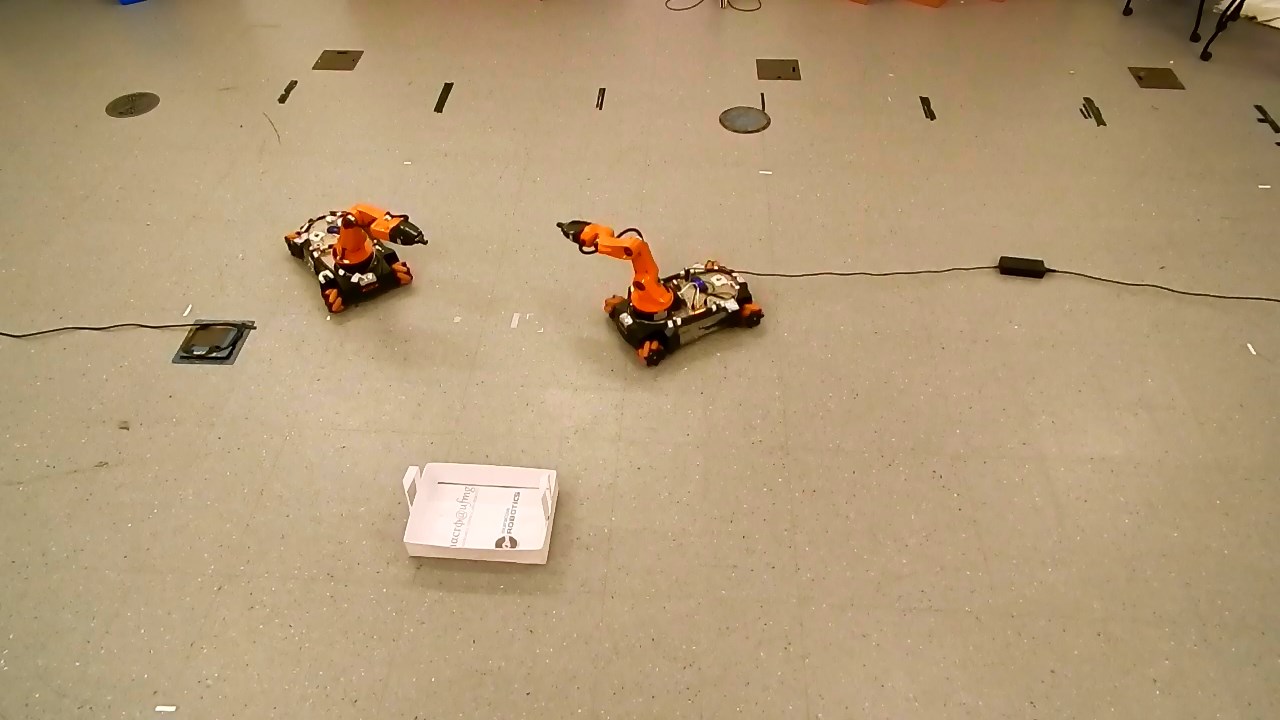}\lyxdeleted{Bruno Vilhena Adorno}{Sat Jun 15 00:05:32 2019}{ }
\par\end{centering}
}\subfloat[$t=14$s.\label{fig:expkuka1s4}]{\noindent \begin{centering}
\includegraphics[viewport=80bp 85.5bp 770bp 630bp,clip,width=0.4\columnwidth]{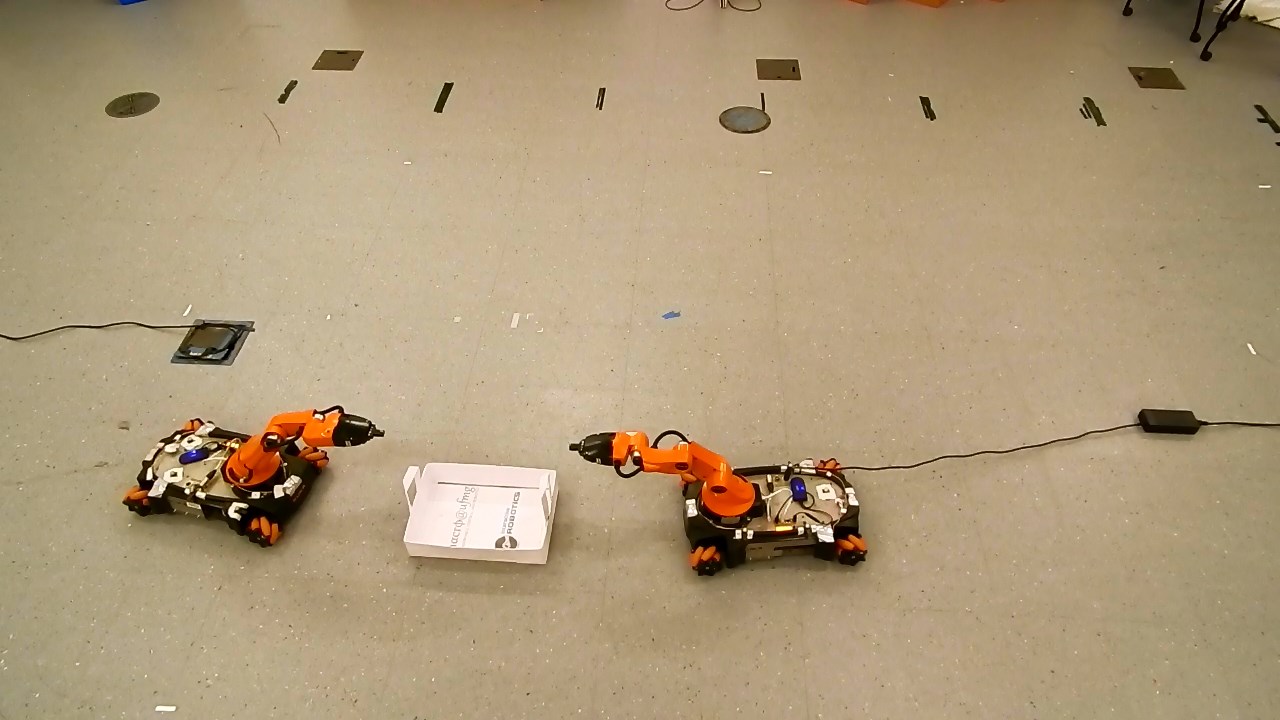}\lyxdeleted{Bruno Vilhena Adorno}{Sat Jun 15 00:05:32 2019}{ }
\par\end{centering}
}\caption{\lyxdeleted{Bruno Vilhena Adorno}{Sat Jun 15 00:05:32 2019}{ }\label{fig:expkuka1s}
Experiment on formation control with two KUKA YouBots. The goal is
to have a final formation where the robots are located around the
box with their end-effectors pointing to the center of the box, opposite
to each other.\lyxdeleted{Bruno Vilhena Adorno}{Sat Jun 15 00:05:32 2019}{ }}
\end{figure}

The state trajectories of the outputs $\dq y_{ce,i}=y_{ce,i,2}\imi+y_{ce,i,3}\imj+y_{ce,i,4}\imk+\dual(y_{ce,i,6}\imi+y_{ce,i,7}\imj+y_{ce,i,8}\imk)$
for each agent are shown in Figure~\ref{fig:expkuka1}. The constant
yellow line represents the leader state (i.e., the box pose), and
the blue and orange lines represent agents $1$ and $2$, respectively.
The continuous lines represent the measurements of the agents outputs,
and the thinner dashed lines represent the solution given by a simulation
carried out with the same initial pose configurations. The states
mainly follow the expected behavior given by the analytical solution,
although noises, delays, and initial conditions on velocities, which
are not explicitly considered in the designed control laws, cause
some deviations from the simulated values, as expected.

\begin{figure}[t]
\noindent \begin{centering}
\subfloat[$y_{ce,i,2}(t)$ for each agent.\label{fig:expkuka1y2}]{\includegraphics[width=0.3\columnwidth]{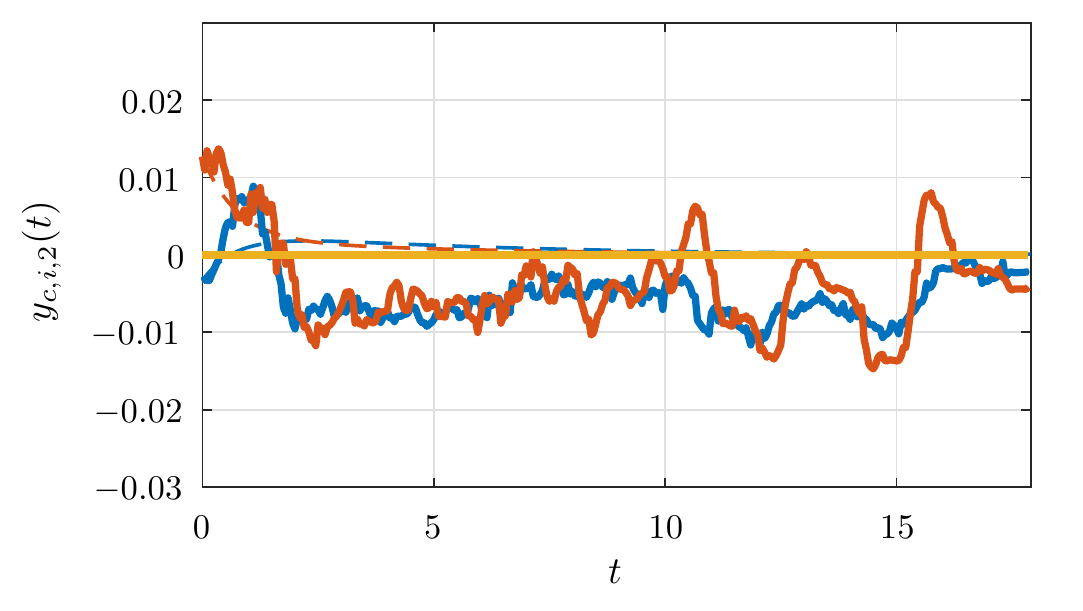}

}\subfloat[$y_{ce,i,3}(t)$ for each agent.\label{fig:expkuka1y3}]{\includegraphics[width=0.3\columnwidth]{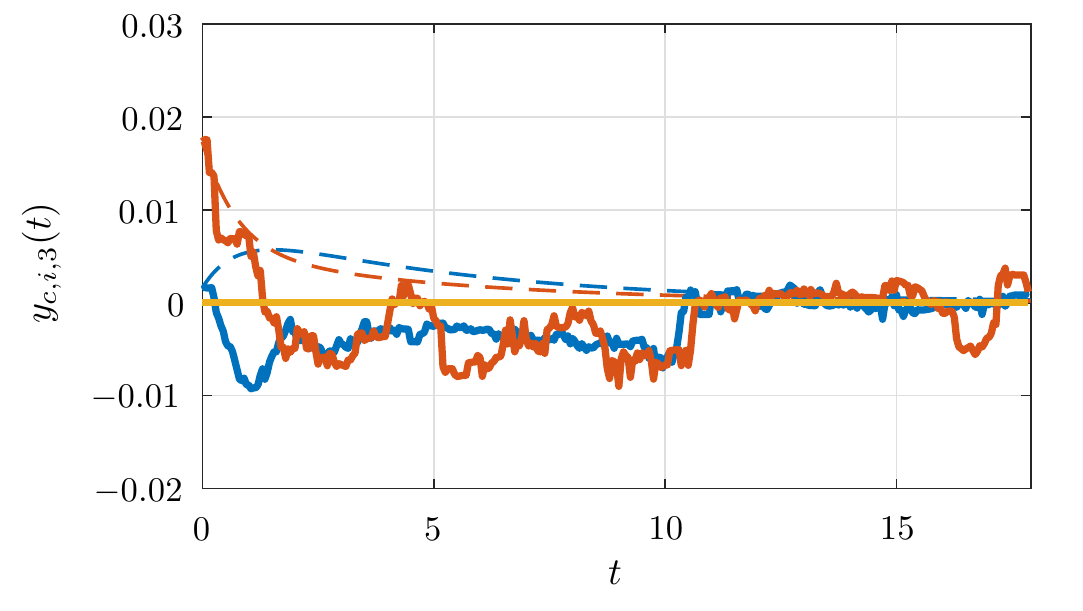}\lyxdeleted{Bruno Vilhena Adorno}{Sat Jun 15 00:05:32 2019}{ }

}\subfloat[$y_{ce,i,4}(t)$ for each agent.\label{fig:expkuka1y4}]{\includegraphics[width=0.3\columnwidth]{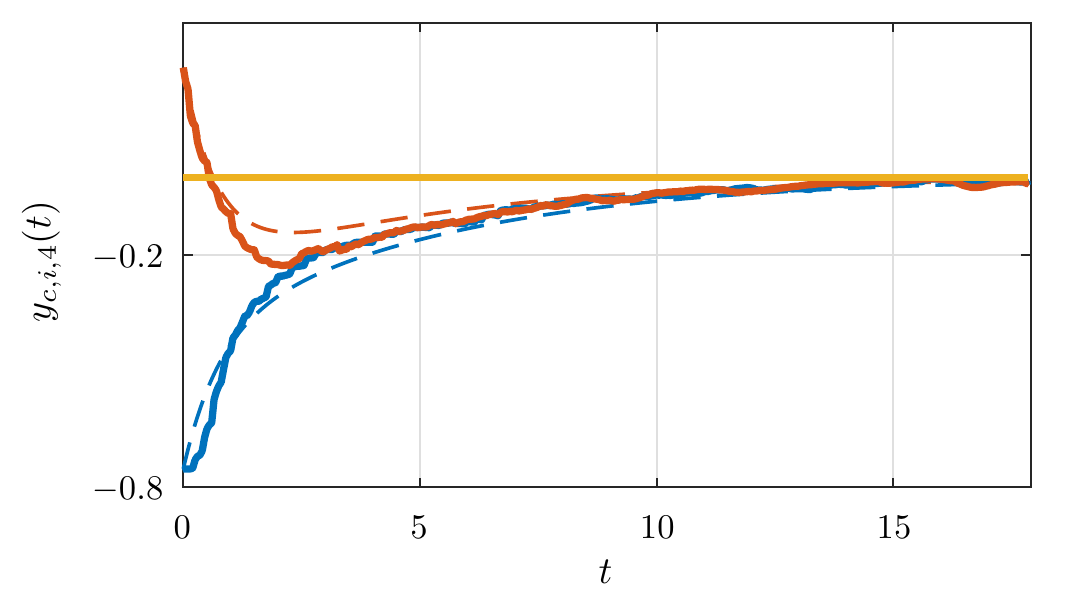}\lyxdeleted{Bruno Vilhena Adorno}{Sat Jun 15 00:05:32 2019}{ }

}
\par\end{centering}
\centering{}\subfloat[$y_{ce,i,6}(t)$ for each agent.\label{fig:expkuka1y6}]{\includegraphics[width=0.3\columnwidth]{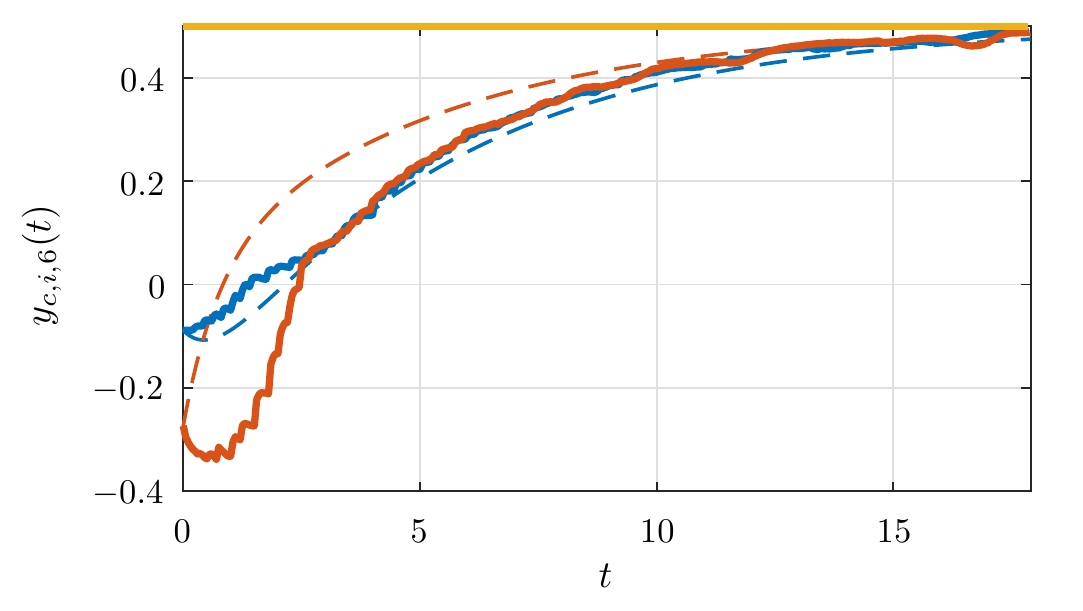}\lyxdeleted{Bruno Vilhena Adorno}{Sat Jun 15 00:05:32 2019}{ }

}\subfloat[$y_{ce,i,7}(t)$ for each agent.\label{fig:expkuka1y7}]{\includegraphics[width=0.3\columnwidth]{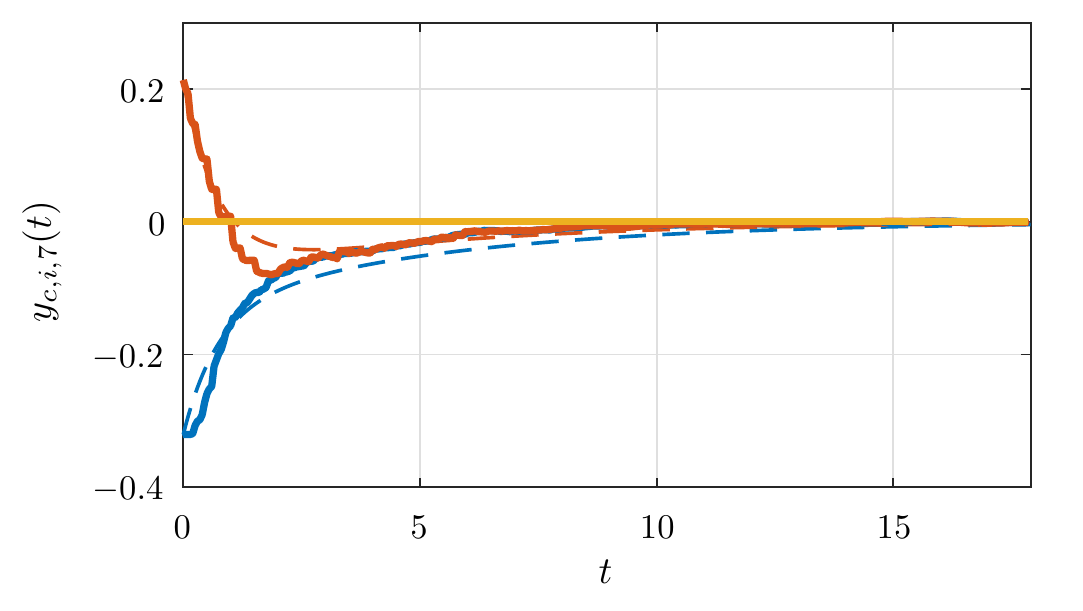}\lyxdeleted{Bruno Vilhena Adorno}{Sat Jun 15 00:05:32 2019}{ }

}\subfloat[$y_{ce,i,8}(t)$ for each agent.\label{fig:expkuka1y8}]{\includegraphics[width=0.3\columnwidth]{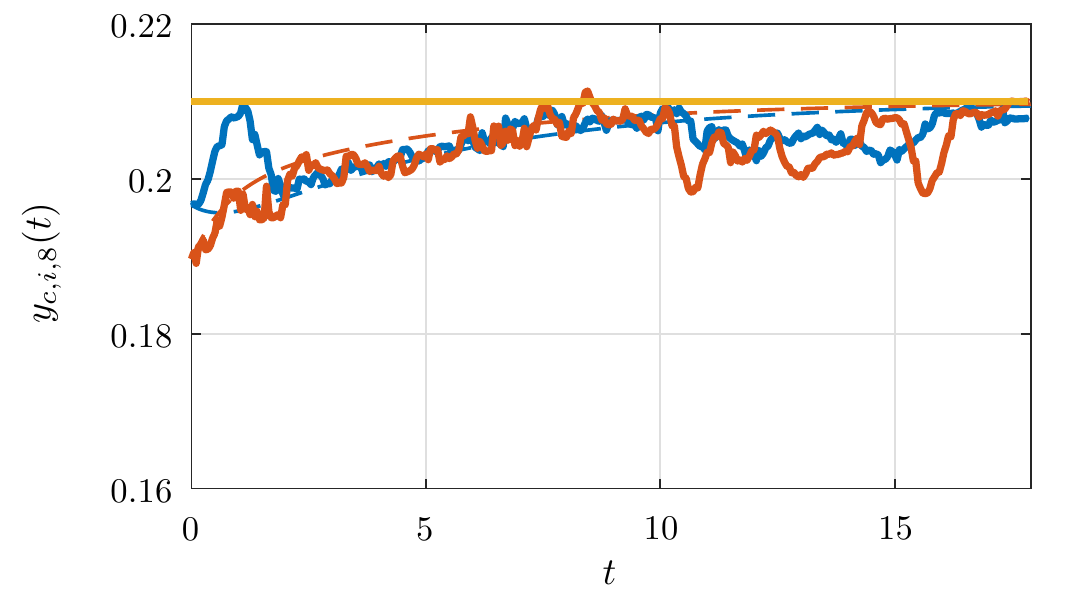}\lyxdeleted{Bruno Vilhena Adorno}{Sat Jun 15 00:05:32 2019}{ }

}\caption{\label{fig:expkuka1} Coefficients of the output $\protect\dq y_{ce,i}$
of each agent in the experiment on formation with two KUKA YouBots.
The dashed curves correspond to the simulated values, whereas the
solid ones correspond to the actual values obtained from the experiments.
The constant curves correspond to the reference provided by the virtual
agent 3.}
\end{figure}

\subsubsection{Cooperative manipulation}

In this second subtask, the goal is to make the robots grasp the box
and then move it around the workspace while maintaining the formation.
To that end, after the robots achieve the formation around the box
in the pre-grasping subtask, as shown in Figure~\ref{fig:expkuka1s4},
the references $\dq{\delta}_{i}$ are changed to a lower position
in the $z$ axis and rotated around the $y$ axis, so that the agents
adjust the grasp (Figure~\ref{fig:expkuka1s5}). By reducing the
distance of each $\dq{\delta}_{i}$ with respect to the center of
formation and returning the reference to a higher position in the
$z$ axis, the agents grasp the box by the flexible straps (Figure~\ref{fig:expkuka1s6}).
Next, the reference corresponding to the box location is changed in
order to drive the agents to a pick up zone, where the box is loaded
(Figure~\ref{fig:expkuka1s7}). After loading the box in the pick-up
zone, the reference is changed again and the agents carry the box
in the direction of a delivery zone, passing through the location
shown in Figure~\ref{fig:expkuka1s8}, then reaching the delivery
zone in Figure~\ref{fig:expkuka1s9}. Once the agents reach the delivery
zone, the value of each $\dq{\delta}_{i}$ is changed in order to
release and deliver the box (Figure~\ref{fig:expkuka1s10}).

\begin{figure}[t]
\centering 
 \subfloat[Adjusting the grasp.\label{fig:expkuka1s5}]{\includegraphics[clip,width=0.4\columnwidth]{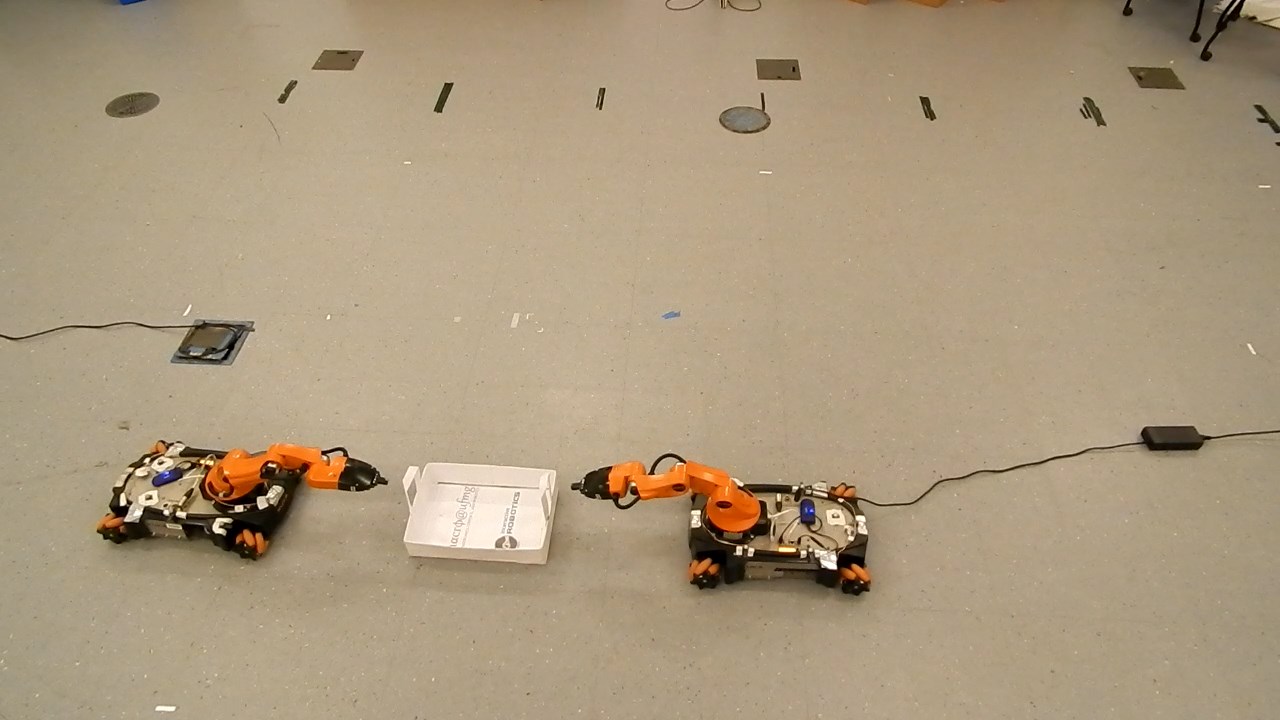}\lyxdeleted{Bruno Vilhena Adorno}{Sat Jun 15 00:05:32 2019}{ }

}\subfloat[Carrying the box.\label{fig:expkuka1s6}]{\includegraphics[clip,width=0.4\columnwidth]{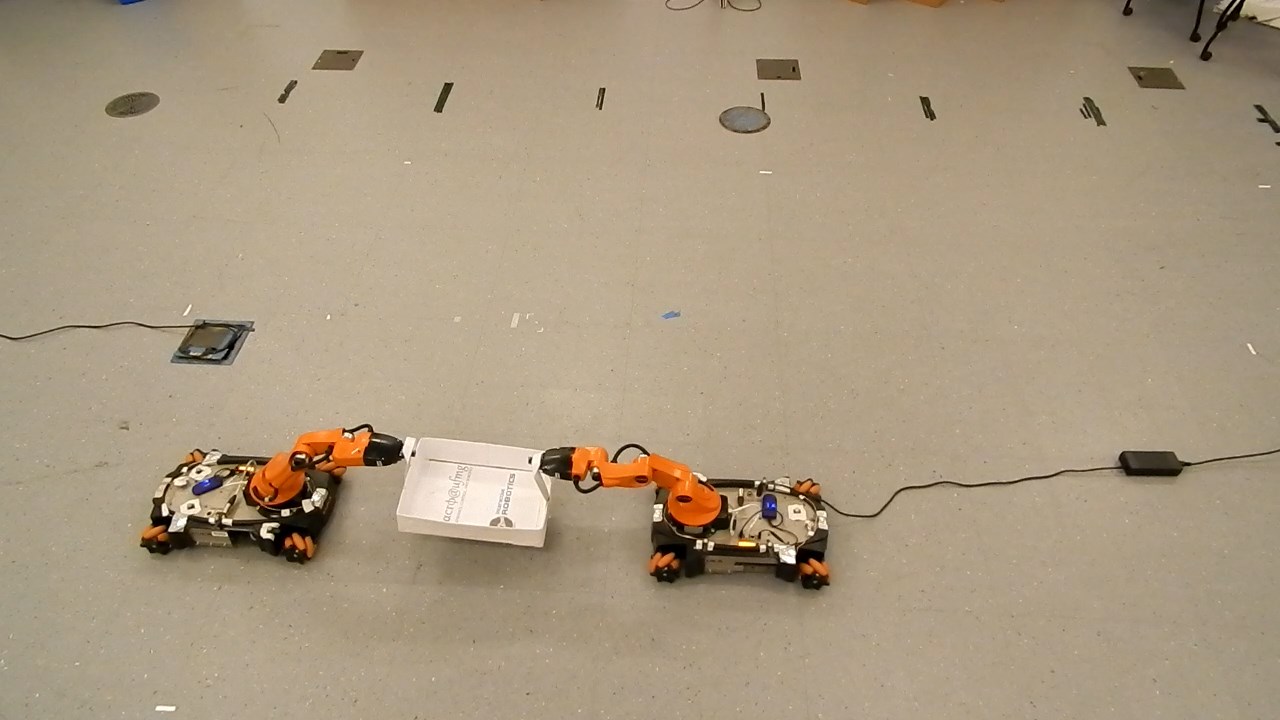}\lyxdeleted{Bruno Vilhena Adorno}{Sat Jun 15 00:05:32 2019}{ }

}\\
 \subfloat[Pick-up zone.\label{fig:expkuka1s7}]{\includegraphics[clip,width=0.4\columnwidth]{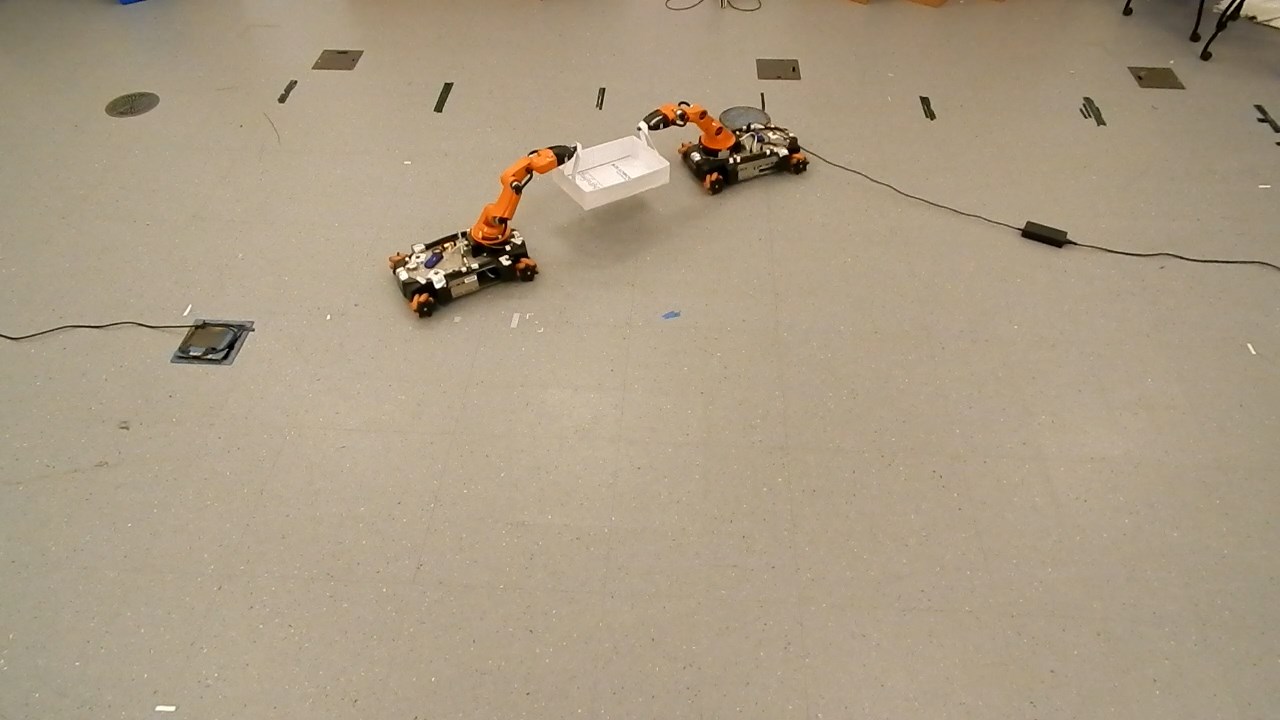}\lyxdeleted{Bruno Vilhena Adorno}{Sat Jun 15 00:05:32 2019}{ }

}\subfloat[Carrying the box to the delivery zone.\label{fig:expkuka1s8}]{\includegraphics[clip,width=0.4\columnwidth]{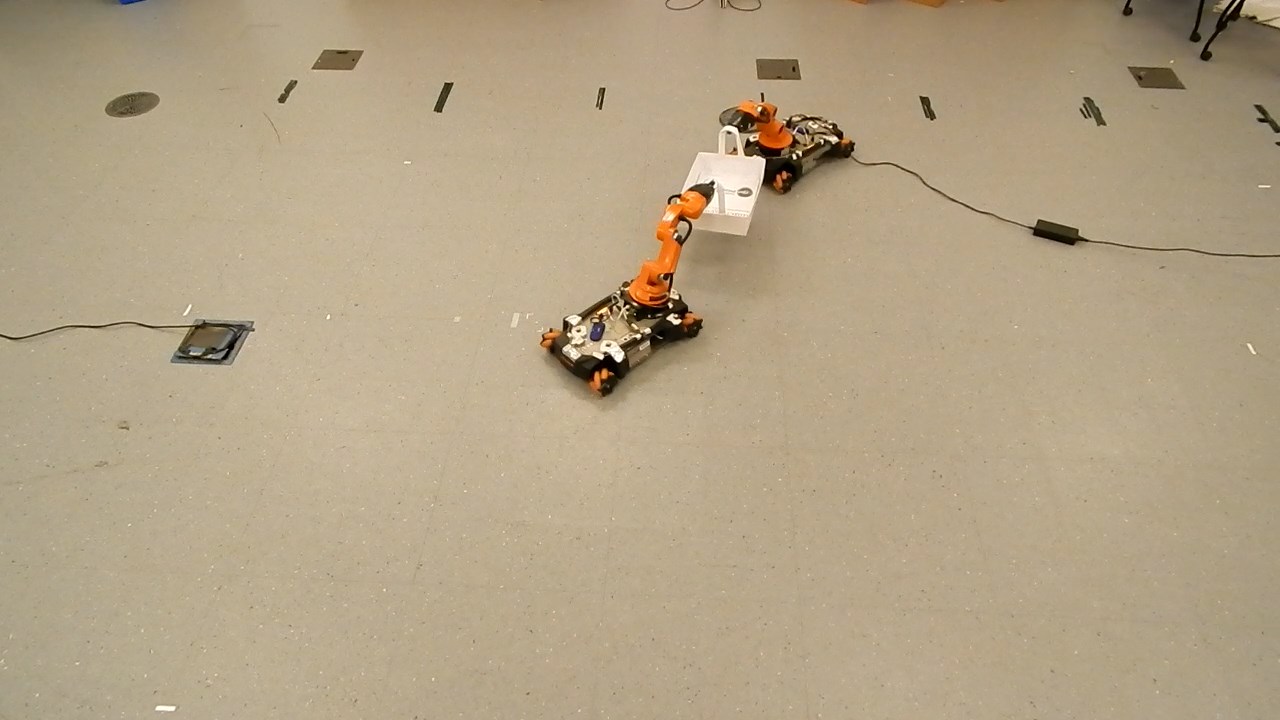}\lyxdeleted{Bruno Vilhena Adorno}{Sat Jun 15 00:05:32 2019}{ }

}\\
 \subfloat[Delivery zone.\label{fig:expkuka1s9}]{\includegraphics[clip,width=0.4\columnwidth]{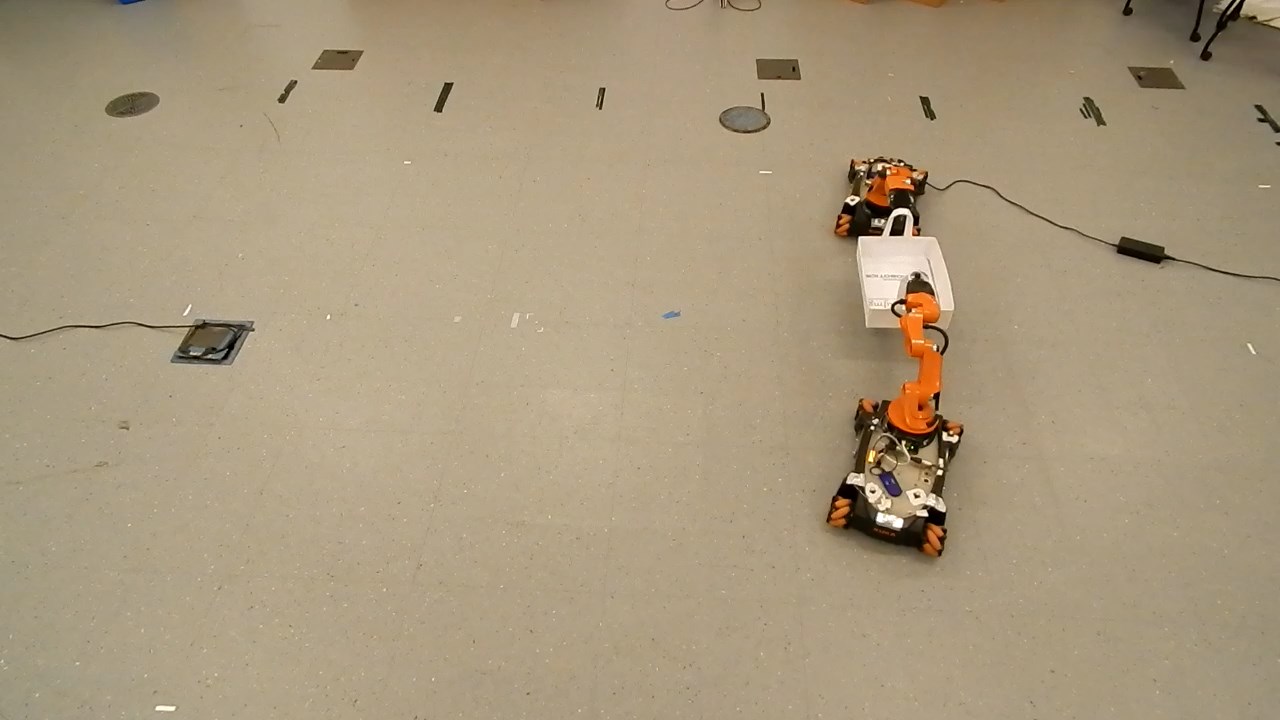}\lyxdeleted{Bruno Vilhena Adorno}{Sat Jun 15 00:05:32 2019}{ }

}\subfloat[Delivering the box.\label{fig:expkuka1s10}]{\includegraphics[clip,width=0.4\columnwidth]{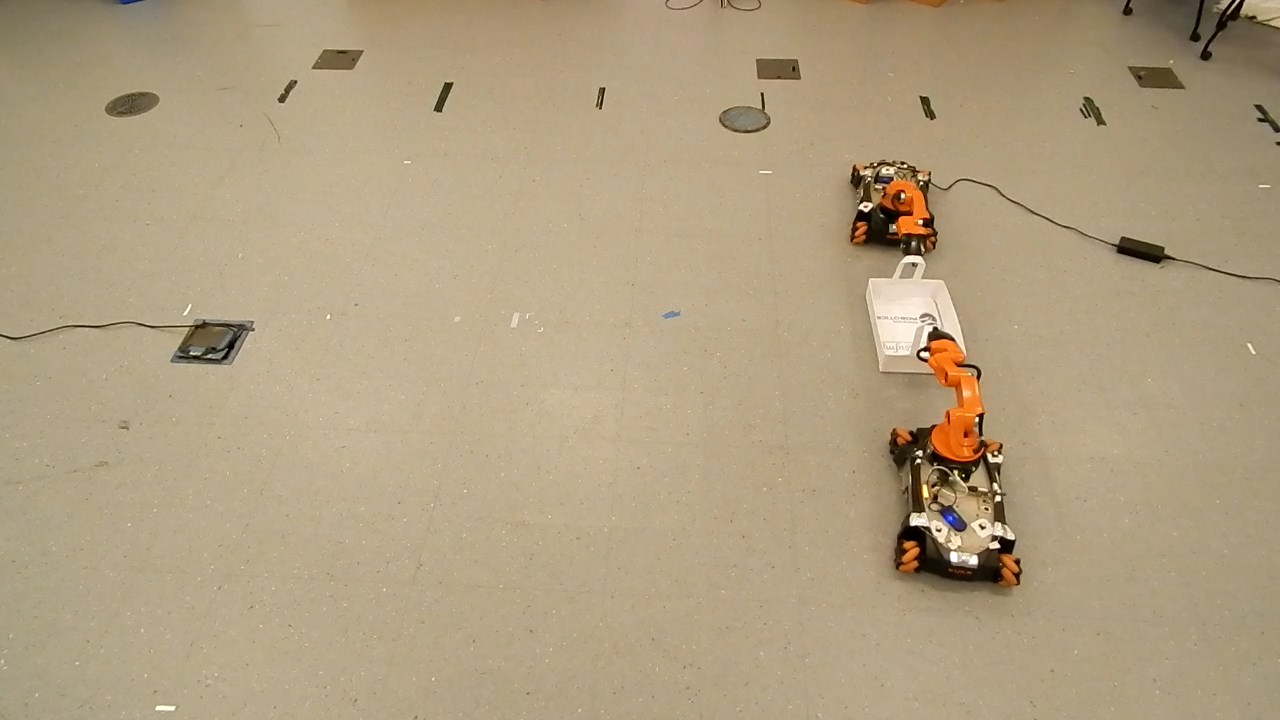}\lyxdeleted{Bruno Vilhena Adorno}{Sat Jun 15 00:05:32 2019}{ }

}\\
 \caption{\lyxdeleted{Bruno Vilhena Adorno}{Sat Jun 15 00:05:32 2019}{ }\label{fig:expkuka1sp2}
Experiment on cooperative manipulation with KUKA YouBots.\lyxdeleted{Bruno Vilhena Adorno}{Sat Jun 15 00:05:32 2019}{ }}
\end{figure}

With the interplay between changing the reference of an object, which
is represented by \textbf{Agent $3$}, and providing different assignments
of $\dq{\delta}_{i}$ for each robot, many different tasks can be
achieved, as depicted in the given example.

\subsubsection{Discussion}

The manipulation task presented in this section can be categorized
as a \emph{leader-following} problem as we have defined the pose of
the box as a single virtual leader. Although we have not explicitly
mentioned the solution of this type of problem in the development
of our theoretical results, the techniques proposed in this work are
general enough to allow its treatment. More specifically, when a single
virtual leader is static, the leader-following problem can also be
defined as a \emph{consensus regulation} problem in which the objective
is to guide the consensus variables of the system to the values of
the leader variables, in contrast to the leaderless consensus problem,
where the variables converge to a set of values that are a function
of the initial values of the agents variables. Therefore, as the first
subtask consists of a pre-grasping formation, where the robots gather
around a box, which is represented by a static virtual leader (\textbf{Agent
3} in Figure~\ref{fig:network2}), the leader-following problem boils
down to a consensus regulation problem with a static leader as the
root of a directed spanning tree, thus satisfying the requirement
of the existence of a directed spanning tree stated in our proofs.
In conclusion, the execution of this subtask can be seen as a real
world verification of the proposed methodology.

On the other hand, in the second subtask, the robots grasp the box
and move it around the workspace. In this case, the agents have to
follow a dynamic virtual leader, as they have to move the box. Although
the design of controllers able to guarantee perfect tracking of a
dynamic leader is out of the scope of this work, by designing a trajectory
in which the virtual leader moves smoothly and slowly enough, the
system has shown to be able to track it with a small error. Indeed,
this demonstrates some robustness of our approach as the independent
dynamic behavior of the leader can be seen as a disturbance to the
system.

\section{Conclusion}

\label{sec:Conclusion}

This paper presented a solution based on dual quaternion algebra to
the general problem of pose consensus for systems composed of multiple
rigid-bodies, and then extended the theory in order to design consensus-based
formation control laws. Since unit dual quaternions belong to a non-Euclidean
manifold, the consensus protocols usually found in the literature
cannot be directly applied to the problem of pose consensus because
those protocols assume an $n$-dimensional Euclidean space. However,
thanks to the isomorphism of pure dual quaternions (i.e., dual quaternions
with real part equal to zero) and $\mathbb{R}^{6}$ under the addition
operation, an \emph{output} consensus protocol was designed and then
we proved that output consensus (i.e., consensus on $\log{\dq x}_{i}$)
implies pose consensus (i.e., consensus on $\dq x_{i}$). This result,
together with the differential logarithm mapping of unit dual quaternions,
allowed the design of pose consensus protocols, which ensures that
the system will achieve consensus as long as the information flow
is described by directed graphs that have a directed spanning tree.

Since dual quaternions are a generalization of quaternions, the corresponding
proofs are much more compact than those obtained when using quaternions.
Usually, proofs are shorter because we do not need to do a separate
analysis for rotation and translation, and, in addition, we usually
exploit the dual quaternion algebra to make those proofs even cleaner
and shorter.

Furthermore, unit dual quaternions capture the intrinsic coupling
between translation and rotation in rigid motions, which has an important
practical consequence: the instantaneous control effort (i.e., the
norm of the control input) of controllers based on dual quaternions
is smaller than the instantaneous control effort of decoupled controllers
that use rotation quaternions and translation vectors separately,
as reported in the literature \citep{Figueredo2018}.

A consensus-based approach for formation control of free-flying rigid-body
teams was also proposed and then applied to the decentralized formation
control of mobile manipulators. In that case, the objective is to
achieve desired formations for the set of end-effectors of mobile
manipulators and let each robot generate its own motion in order to
move the end-effector according to the reference provided by the consensus
protocol. The advantage of using such abstraction is that the consensus
protocols are used to determine, in a decentralized way, how each
robot's end-effector should be, regardless of the topology and dimension
of the robots' configuration spaces. In fact, since the robots use
\emph{local} motion controllers, the consensus-based formation control
can be applied to a highly heterogeneous multi-agent system as long
as each agent is capable of following the reference provided by the
consensus protocols.

Finally, numerical simulations were carried out to illustrate the
applicability and scalability of the proposed method and an experiment
with real mobile manipulators was presented to show the proposed method,
in practice, in a cooperative manipulation scenario.\lyxdeleted{Bruno Vilhena Adorno}{Sat Jun 15 00:05:43 2019}{ }

Although the proposed distributed control laws ensure consensus of
free-flying agents, we have not taken into account the problem of
unwinding. As a result, agents may execute longer trajectories before
the overall system achieves consensus. Future works will be focused
on the unwinding problem in the context of pose consensus protocols,
which can only be solved by using discontinuous or hybrid controllers
\citep{Kussaba2017}, and may also take into account time-delays in
the agents interactions, switching topologies, leader-follower with
multiple dynamic leaders, containment control, and couplings design.

\section*{Acknowledgements}

This work was supported in part by the Coordenação de Aperfeiçoamento
de Pessoal de Nível Superior (CAPES) (Finance Code 88887.136349/2017-00),
Conselho Nacional de Desenvolvimento Científico e Tecnológico (CNPq)
(grant numbers 456826/2013-0, 232985/2014-6, 311063/2017-9, and 303901/2018-7),
Fundação de Amparo à Pesquisa de Minas Gerais (FAPEMIG), and MIT-Brazil
Program -- MISTI.

\appendix

\section{Auxiliary facts and proofs}

\label{sec:AppA}
\begin{fact}
\label{fact:identities with quaternion norm}Given $\quat y=\left(\phi/2\right)\quat n$,
where $\quat n\in\mathbb{S}^{3}\cap\mathbb{H}_{p}$ and $\phi\in\left[0,2\pi\right)$,
\begin{align}
\frac{\cos\norm{\quat y}}{\norm{\quat y}^{m}} & =\frac{\cos\left(\phi/2\right)}{\left(\phi/2\right)^{m}}\tag{i}\label{eq:cos over y square}\\
\frac{\sin\norm{\quat y}}{\norm{\quat y}^{m}} & =\frac{\sin\left(\phi/2\right)}{\left(\phi/2\right)^{m}}.\tag{ii}\label{sin over y to m}
\end{align}
\end{fact}

\begin{proof}
Since $\norm{\quat n}=1$, then $\norm{\quat y}=\left|\phi\right|/2=\phi/2$
because $\phi$ is nonnegative. Thus we obtain \eqref{eq:cos over y square}
and \eqref{sin over y to m}.
\end{proof}
\begin{prop}
\label{prop:existence_of_left_pseudo_inverse_and_unicity}Let $\mymatrix A\in\mathbb{R}^{m\times n}$,
$\myvec x\in\mathbb{R}^{n}$, and $\myvec b\in\mathbb{R}^{m}$ such
that 
\begin{equation}
\mymatrix A\myvec x=\myvec b\label{eq:overdetermined linear system}
\end{equation}
and $m\geq n$. If there exists a left pseudoinverse $\mymatrix A^{+}$
such that $\mymatrix A^{+}\mymatrix A=\mymatrix I$, then the solution
to \eqref{eq:overdetermined linear system} given by $\myvec x=\mymatrix A^{+}\myvec b$
is unique and $\myvec b=\myvec 0$ if and only if $\myvec x=\myvec 0$.
\end{prop}

\begin{proof}
If there exists $\mymatrix A^{+}$ such that $\mymatrix A^{+}\mymatrix A=\mymatrix I$
then $\mymatrix A\mymatrix A^{+}\mymatrix A=\mymatrix A$, thus $\myvec b=\mymatrix A\myvec x=\mymatrix A\mymatrix A^{+}\mymatrix A\myvec x=\mymatrix A\mymatrix A^{+}\myvec b$.
This way, $\myvec x=\mymatrix A^{+}\myvec b$ is clearly a solution
to \eqref{eq:overdetermined linear system} because $\mymatrix A\myvec x=\mymatrix A\mymatrix A^{+}\myvec b=\myvec b$.
Furthermore, suppose that $\myvec x'$ is also a solution to \eqref{eq:overdetermined linear system},
thus $\myvec b=\mymatrix A\myvec x'=\mymatrix A\myvec x$. Since $\mymatrix A^{+}\mymatrix A=\mymatrix I$
then $\mymatrix A^{+}\mymatrix A\myvec x'=\mymatrix A^{+}\mymatrix A\myvec x$
implies $\myvec x'=\myvec x$, hence $\myvec x=\mymatrix A^{+}\myvec b$
is indeed a unique solution.

Lastly, if $\myvec x=\myvec 0$ then $\myvec b=\mymatrix A\myvec x=\mymatrix A\myvec 0=\myvec 0$;
conversely, if $\myvec b=\myvec 0$ then $\myvec x=\mymatrix A^{+}\myvec b=\mymatrix A^{+}\myvec 0=\myvec 0$.
Hence $\myvec b=\myvec 0\iff\myvec x=\myvec 0$.
\end{proof}
\begin{prop}
\label{thm:Q is full column rank}Consider $\quat r\in\mathbb{S}^{3},$
with $\quat r=\cos\left(\phi/2\right)+\quat n\sin\left(\phi/2\right)$
and $\quat n\in\mathbb{S}^{3}\cap\mathbb{H}_{p}$, and $\quat y\in\mathbb{H}_{p}$
such that $\quat y\triangleq\log\quat r$, then 
\[
\mymatrix Q\left(\quat r\right)\triangleq\frac{\partial\vector_{4}\quat r}{\partial\vector_{3}\quat y}
\]
is full column rank for $\phi\in\left[0,2\pi\right)$.
\end{prop}

\begin{proof}
$\mymatrix Q\left(\quat r\right)$ is full column rank if $\det\left(\mymatrix Q\left(\quat r\right)^{T}\mymatrix Q\left(\quat r\right)\right)\neq0$.
Thus, 
\begin{multline*}
\det\left(\mymatrix Q\left(\quat r\right)^{T}\mymatrix Q\left(\quat r\right)\right)=\Theta^{4}\sin^{2}\left(\frac{\phi}{2}\right)\left(n_{x}^{2}+n_{y}^{2}+n_{z}^{2}\right)\\
+\Gamma^{2}\Theta^{4}\left(n_{x}^{4}+n_{y}^{4}+n_{z}^{4}+2n_{x}^{2}n_{y}^{2}+2n_{x}^{2}n_{z}^{2}+2n_{y}^{2}n_{z}^{2}\right)\\
+2\Gamma\Theta^{5}\left(n_{x}^{2}+n_{y}^{2}+n_{z}^{2}\right)+\Theta^{6},
\end{multline*}
where $\Gamma$ and $\Theta$ are defined as in Theorem~\ref{th:rQy}.
Using the fact that $\norm{\quat n}=1$, $\Gamma=r_{1}-\Theta$ and
\begin{align*}
\left(n_{x}^{2}+n_{y}^{2}+n_{z}^{2}\right)^{2} & =n_{x}^{4}+n_{y}^{4}+n_{z}^{4}+2n_{x}^{2}n_{y}^{2}+2n_{x}^{2}n_{z}^{2}+2n_{y}^{2}n_{z}^{2},
\end{align*}
we obtain 
\begin{align*}
\det\left(\mymatrix Q\left(\quat r\right)^{T}\mymatrix Q\left(\quat r\right)\right) & =\Theta^{4}\sin^{2}\left(\frac{\phi}{2}\right)+\Gamma^{2}\Theta^{4}+2\Gamma\Theta^{5}+\Theta^{6}\\
 & =\Theta^{4}\left(\sin^{2}\left(\frac{\phi}{2}\right)+\Gamma^{2}+2\Gamma\Theta+\Theta^{2}\right)\\
 & =\Theta^{4},
\end{align*}
which is different from zero for $\phi\in\left[0,2\pi\right)$.
\end{proof}
\begin{prop}
\label{thm:A is invertible}Given $\quat p\in\mathbb{H}_{p}$ and
$\quat r\in\mathbb{S}^{3}$, the inverse of 
\[
\mymatrix A=\begin{bmatrix}\mymatrix I_{4} & \mymatrix 0_{4\times4}\\
\frac{1}{2}\hamiquat +{\quat p} & \hamiquat -{\quat r}
\end{bmatrix}
\]
is given by 
\[
\mymatrix A^{-1}=\begin{bmatrix}\mymatrix I_{4} & \mymatrix 0_{4\times4}\\
-\frac{1}{2}\hamiquat -{\quat r^{*}}\hamiquat +{\quat p} & \hamiquat -{\quat r^{*}}
\end{bmatrix}.
\]
\end{prop}

\begin{proof}
Since $\hamiquat -{\quat r^{*}}=\hamiquat -{\quat r}^{T}$ and $\hamiquat -{\quat r}\in\mathrm{O}\left(4\right)$
by Propositions \ref{prop:hamilton quaternion conjugate equals hamilton transpose}
and \ref{prop:hami4_O4}, the result $\mymatrix A\mymatrix A^{-1}=\mymatrix A^{-1}\mymatrix A=\mymatrix I_{8}$
follows by direct calculation.
\end{proof}

\section{Facts about Hamilton operators}
\begin{prop}
\label{prop:hamilton quaternion conjugate equals hamilton transpose}Let
$\quat h\in\mathbb{H}$, $\hamiquat -{\quat h^{*}}=\hamiquat -{\quat h}^{T}$
and $\hamiquat +{\quat h^{*}}=\hamiquat +{\quat h}^{T}$.
\end{prop}

\begin{proof}
Since the Hamilton operators $\hami +$ and $\hami -$ are defined
as in \eqref{eq:hamilton-four}, these equalities can be verified
by inspection.
\end{proof}
\begin{prop}
\label{prop:hami4_O4}If $\quat r\in\mathbb{S}^{3}$ then $\hamiquat +{\quat r},\hamiquat -{\quat r}\in\mathrm{O}\left(4\right)$.
\end{prop}

\begin{proof}
Since $\quat r\in\mathbb{S}^{3}$ then $\quat r^{*}\quat r=1$ and
$\quat x=\quat x\quat r^{*}\quat r$, $\forall\quat x\in\mathbb{H}$,
which implies 
\begin{align*}
\vector_{4}\quat x & =\hamiquat -{\quat r}\vector_{4}\left(\quat x\quat r^{*}\right)\\
 & =\hamiquat -{\quat r}\hamiquat -{\quat r^{*}}\vector_{4}\quat x\\
 & =\hamiquat -{\quat r^{*}\quat r}\vector_{4}\quat x,\qquad\forall\vector_{4}\quat x\in\mathbb{R}^{4}.
\end{align*}
Thus $\hamiquat -{\quat r}\hamiquat -{\quat r^{*}}=\hamiquat -{\quat r^{*}\quat r}=\mymatrix I$,
therefore $\hamiquat -{\quat r^{*}}=\hamiquat -{\quat r}^{-1}$. Furthermore,
from Proposition~\ref{prop:hamilton quaternion conjugate equals hamilton transpose}
we have that $\hamiquat -{\quat r^{*}}=\hamiquat -{\quat r}^{T}$,
which implies $\hamiquat -{\quat r}^{-1}=\hamiquat -{\quat r}^{T}$
and hence $\hamiquat -{\quat r}\in\mathrm{O}\left(4\right)$.

From $\quat x=\quat r^{*}\quat r\quat x$, $\forall\quat x\in\mathbb{H}$,
we apply the same reasoning to conclude that $\hamiquat +{\quat r}\in\mathrm{O}\left(4\right)$.
\end{proof}

\section{Whole Body Kinematics of Holonomic Mobile Manipulators\label{sec:wholebody}}

Consider a holonomic mobile base moving in the plane $XY$ and an
inertial reference frame $\mathcal{F}_{0}$ somewhere in the space.
The position of the local reference frame $\mathcal{F}_{b}$ in the
center of the mobile base is given by the coordinates $\left(x,y\right)$,
and the orientation is given by the rotation angle $\phi$ around
axis $Z$. Thus, the generalized coordinates of the base can be written
as $\myvec q_{b}=\begin{bmatrix}x & y & \phi\end{bmatrix}^{T}$ and
its pose, relative to $\mathcal{F}_{0}$, is given by the following
dual quaternion 
\begin{equation}
\dq x_{b}^{0}=\quat r_{b}^{0}+\dual\frac{1}{2}\quat p_{0,b}^{0}\quat r_{b}^{0},\label{eq:quat_x_hol}
\end{equation}
where $\quat r_{b}^{0}=\cos\left(\phi/2\right)+\imk\sin\left(\phi/2\right)$
and $\quat p_{0,b}^{0}=x\imi+y\imj$ \citep{Adorno2011}.

Taking the first time-derivative of \eqref{eq:quat_x_hol} and mapping
into $\mathbb{R}^{8}$ with the $\vector_{8}$ operator, the differential
forward kinematics of the holonomic mobile base is given by 
\begin{equation}
\vector_{8}\dot{\dq x}_{b}^{0}=\mymatrix J_{b}\dot{\myvec q}_{b},\label{eq:dfkmb}
\end{equation}
where $\mymatrix J_{b}$ is the (dual quaternion) Jacobian matrix
(see page 89 in \citep{Adorno2011}).

Next, consider a manipulator on top of the mobile base. Let the reference
frame of the manipulator's base be $\mathcal{F}_{m}$ and $\dq x_{m}^{b}$
be a constant dual quaternion representing the rigid-motion from $\mathcal{F}_{b}$
to $\mathcal{F}_{m}$. For a serial manipulator with $\eta$ revolute
joints, with $\theta_{k}$ being the angle of the $k$-th joint, for
$k=1,\ldots,\eta$, the forward kinematics that relates the frame
$\mathcal{F}_{e}$ of the end-effector to the base of the manipulator
$\mathcal{F}_{m}$ is a function of all joints. More specifically,
the pose of the end-effector with respect to the base of the manipulator
is given by the unit dual quaternion $\dq x_{e}^{m}=\dq f(\myvec q_{m})$,
with $\myvec q_{m}=\begin{bmatrix}\theta_{1} & \cdots & \theta_{\eta}\end{bmatrix}^{T}$
being the vector containing all the joint angles \citep{Adorno2011}.

The differential forward kinematics is given by $\dot{\dq x}_{e}^{m}=\dq f'(\myvec q_{m})$,
where $\dq f'\triangleq d\dq f/dt$. Thus, applying the $\vector_{8}$
operator, the differential forward kinematics of the manipulator is
\begin{equation}
\vector_{8}\dot{\dq x}_{e}^{m}=\mymatrix J_{m}\dot{\myvec q}_{m},\label{eq:dfkmm}
\end{equation}
where $\mymatrix J_{m}=\partial\dq f/\partial\theta_{m}\in\mathbb{R}^{8\times\eta}$
is the analytical Jacobian relating the joints velocities to the derivative
of the unit dual quaternion that represents the end-effector pose.
Notice that both forward kinematics and differential forward kinematics
are obtained directly in the algebra of dual quaternions \citep{Adorno2011}.

Coupling the manipulator to the mobile base, the pose of the end-effector,
related to the inertial coordinate frame $\mathcal{F}_{0}$, is described
by the composition of each subsystem and its time derivative is given
by 
\begin{equation}
\dq x_{e}^{0}=\dq x_{b}^{0}\dq x_{m}^{b}\dq x_{e}^{m}\implies\dot{\dq x}_{e}^{0}=\dot{\dq x}_{b}^{0}\dq x_{m}^{b}\dq x_{e}^{m}+\dq x_{b}^{0}\dq x_{m}^{b}\dot{\dq x}_{e}^{m}.\label{eq:fkm derivative}
\end{equation}

Mapping \eqref{eq:fkm derivative} into $\mathbb{R}^{8}$, and using
\eqref{eq:hami}, \eqref{eq:dfkmb}, and \eqref{eq:dfkmm}, we obtain
\begin{alignat*}{1}
\end{alignat*}
which can be written as 
\begin{align}
\vector_{8}\dot{\dq x}_{e}^{0} & =\hami -_{8}(\dq x_{m}^{b}\dq x_{e}^{m})\mymatrix J_{b}\dot{\myvec q}_{b}+\hami +_{8}(\dq x_{b}^{0}\dq x_{m}^{b})\mymatrix J_{m}\dot{\myvec q}_{m},\\
 & =\mymatrix J_{w}\dot{\myvec q}_{w},\label{eq:dfkm}
\end{align}
where 
\begin{gather}
\mymatrix J_{w}=\begin{bmatrix}\hami -_{8}(\dq x_{m}^{b}\dq x_{e}^{m})\mymatrix J_{b} &  & \hami +_{8}(\dq x_{b}^{0}\dq x_{m}^{b})\mymatrix J_{m}\end{bmatrix}\label{eq:whole-body-jacobian-of-mobile-manipulators}
\end{gather}
and $\dot{\myvec q}_{w}=\begin{bmatrix}\dot{\myvec q}_{b}\\
\dot{\myvec q}_{m}
\end{bmatrix}$.

\section*{References}

\bibliographystyle{IEEEtran}
\bibliography{library}

\end{document}